\documentclass{article}

\usepackage{microtype}
\usepackage{graphicx}
\usepackage{booktabs} 

\usepackage{hyperref}
\usepackage[accepted]{icml2019}

\usepackage{amsmath,amsfonts,amssymb,amsthm,array}
\usepackage{amsmath}
\usepackage{xspace}
\usepackage{bm}
\usepackage{subfig}
\usepackage{mathtools}
\usepackage{multirow}
\usepackage{caption}
\usepackage{chngcntr}

\usepackage{blkarray}

\usepackage{tikz}
\usetikzlibrary{arrows,patterns}
\usetikzlibrary{positioning}
\usetikzlibrary{backgrounds}
\usetikzlibrary{arrows.meta}

\usepackage{amsmath,amsfonts,bm}

\def\eqref#1{equation~\ref{#1}}

\def\1{\bm{1}}

\def\vx{{\bm{x}}}
\def\vy{{\bm{y}}}
\def\vz{{\bm{z}}}

\DeclareMathAlphabet{\mathsfit}{\encodingdefault}{\sfdefault}{m}{sl}
\SetMathAlphabet{\mathsfit}{bold}{\encodingdefault}{\sfdefault}{bx}{n}

\newcommand{\E}{\mathbb{E}}

\newcommand{\R}{\mathbb{R}}

\newcommand{\itr}[2]{{#1}^{(#2)}}
\newcommand{\Nin}[1]{\mathcal{N}^{\text{in}}_{#1}}
\newcommand{\Nout}[1]{\mathcal{N}^{\text{out}}_{#1}}
\newcommand{\defeq}{\coloneqq}
\newcommand{\ones}{\bm{1}}

\newcommand{\Exp}{\mathbb{E}}

\newcommand{\bP}{\mathbf{P}}

\newcommand{\bX}{\mathbf{X}}

\newcommand{\bZ}{\mathbf{Z}}

\newcommand{\cD}{{\cal D}}

\newcommand{\cF}{{\cal F}}

\newcommand\numberthis{\addtocounter{equation}{1}\tag{\theequation}}

\newcommand{\ie}{\textit{i.e.}}

\newcommand{\norm}[1]{\left\Vert #1 \right\Vert}
\newcommand{\abs}[1]{\left\vert #1 \right\vert}
\newcommand{\T}{\top}
\newcommand{\order}[1]{\mathcal{O}\left( #1 \right)}

\newcommand{\xbar}{\overline{\bm{x}}}
\newcommand{\nbar}{\overline{\bm{n}}}
\newcommand{\ybar}{\overline{\bm{y}}}

\newcommand{\ssgrad}[4]{\nabla F_{#1}( {#2}^{(#4)}_{#1}; {#3}^{(#4)}_{#1} )}

\theoremstyle{plain}
\newtheorem{thm}{Theorem}[]
\newtheorem{lem}[thm]{Lemma}

\theoremstyle{remark}

\newcommand{\AllReduce}{\textsc{AllReduce}\xspace}
\newcommand{\PushSum}{\textsc{PushSum}\xspace}

\icmltitlerunning{Stochastic Gradient Push for Distributed Deep Learning}

\begin{document}
\twocolumn[
\icmltitle{Stochastic Gradient Push for Distributed Deep Learning}



\icmlsetsymbol{equal}{*}

\begin{icmlauthorlist}
\icmlauthor{Mahmoud Assran}{fb,mg}
\icmlauthor{Nicolas Loizou}{fb,ed}
\icmlauthor{Nicolas Ballas}{fb}
\icmlauthor{Mike Rabbat}{fb}
\end{icmlauthorlist}

\icmlaffiliation{fb}{Facebook AI Research, Montr\'{e}al, QC, Canada}
\icmlaffiliation{mg}{Department of Electrical and Computer Engineering, McGill University, Montr\'{e}al, QC, Canada}
\icmlaffiliation{ed}{School of Mathematics, University of Edinburgh, Edinburgh, Scotland}

\icmlcorrespondingauthor{Mahmoud Assran}{mahmoud.assran@mail.mcgill.ca}

\icmlkeywords{optimization,distributed,large scale,deep learning}

\vskip 0.3in
]



\printAffiliationsAndNotice{}  

\begin{abstract}
Distributed data-parallel algorithms aim to accelerate the training of deep neural networks by parallelizing the computation of large mini-batch gradient updates across multiple nodes. Approaches that synchronize nodes using exact distributed averaging (e.g., via \AllReduce) are sensitive to stragglers and communication delays. The \PushSum gossip algorithm is robust to these issues, but only performs approximate distributed averaging. This paper studies Stochastic Gradient Push (SGP), which combines \PushSum with stochastic gradient updates. We prove that SGP converges to a stationary point of smooth, non-convex objectives at the same sub-linear rate as SGD, and that all nodes achieve consensus.
We empirically validate the performance of SGP on image classification (ResNet-50, ImageNet) and machine translation (Transformer, WMT'16 En-De) workloads.
Our code will be made publicly available.
\end{abstract}

\section{Introduction}
\label{sec:intro}

Deep Neural Networks (DNNs) are the state-of-the art machine learning approach in many application areas, including computer vision~\citep{he2016deep} and natural language processing~\citep{vaswani2017attention}. Stochastic Gradient Descent (SGD) is the current workhorse for training neural networks. The algorithm optimizes the network parameters, $\bm{x}$, to minimize a loss function, $f(\cdot)$, through gradient descent, where the loss function's gradients are approximated using a subset of training examples (a mini-batch). 
DNNs often require large amounts of training data and trainable parameters, necessitating non-trivial computational requirements~\citep{wu2016google,mahajan2018exploring}.

Large mini-batch parallel SGD is usually adopted for distributed training of deep networks~\cite{goyal2017accurate,li2014scaling}.
Worker nodes compute local mini-batch gradients of the loss function on different subsets of the data, and then calculate an exact inter-node average gradient using either the \AllReduce communication primitive, in synchronous implementations~\cite{goyal2017accurate,akiba2017extremely}, or using a central parameter server, in asynchronous implementations~\cite{dean2012large}. Using a parameter server to aggregate gradients introduces a potential bottleneck and a central point of failure~\cite{lian2017can}. The \AllReduce primitive computes the exact average gradient at all workers in a decentralized manner, avoiding issues associated with centralized communication and computation.
However, exact averaging algorithms like \AllReduce are not robust in communication-constrained settings, \ie, where the network bandwidth is a significant bottleneck.

This issue motivates the investigation of a decentralized and inexact version of SGD to reduce the communication overhead associated with distributed training.
There have been numerous decentralized optimization algorithms proposed and studied in the control-systems literature that leverage gossip-based approaches for the computation of aggregate information; see the survey of~\citet{Nedic2018network} and references therein.
State-of-the-art gossip-based optimization methods build on the \PushSum algorithm for distributed averaging~\cite{kempe2003gossip, Nedic2018network}.
Rather than computing exact averages (as with \AllReduce), this line of work uses less-coupled message passing and computes approximate averages. The tradeoff is that approximate distributed averaging also injects additional noise in the average gradient estimate.


In this work we study Stochastic Gradient Push (SGP), an algorithm blending parallel SGD and \PushSum.
SGP enables the use of generic communication topologies that may be directed (asymmetric), sparse, and time-varying. 
In contrast, existing gossip-based approaches explored in the context of training DNNs~\cite{lian2017can,Jiang2017collaborative} are constrained to use symmetric communication (\ie, if node $i$ sends to $j$, then $i$ must also receive from $j$ before proceeding) and thus inherently require deadlock-avoidance, and more synchronization, making them slower and more sensitive to stragglers. Moreover, SGP can be seen as a generalization of parallel SGD and these previous approaches.

SGP was first proposed in the control systems literature for minimizing the sum of \textit{strongly-convex} functions~\cite{Nedic2016stochastic}.
We make three main contributions.
1)~We propose and analyze a novel variant of SGP, called Overlap SGP, which overlaps communication and computation to hide communication overhead. 
2)~We provide novel theoretical guarantees, proving that SGP (and Overlap SGP) converges to a stationary point of smooth \textit{non-convex} functions at an $\mathcal{O}(1/\sqrt{nK})$ rate,
for an appropriately chosen step-size,
where $n$ is the number of nodes and $K$ is the number of iterations. 
3)~We conduct experiments on image classification (ResNet50, ImageNet) and neural machine translation tasks (Transformer, WMT16 En-De), demonstrating that SGP and Overlap SGP can substantially accelerate training of deep neural networks, by reducing communication overhead and mitigating the effects of stragglers. Given a fixed runtime budget, we find that SGP and Overlap SGP can train models that achieve better final train/validation accuracies than \AllReduce SGD in communication-constrained settings.

For example, we train a ResNet-50 on ImageNet using 256 GPUs spread across 32 compute nodes (8 GPUs / node), where communication between nodes is over 10~Gbps Ethernet.
In this setting \AllReduce SGD achieves $76.2\%$ top-1 validation accuracy, while Overlap SGP achieves $76.2\%$ accuracy in only 1/3 of the time, and $77.1\%$ accuracy in 1/2 the time.
Similarly, when training a Transformer network on the WMT'16 En-De translation task, SGP runs approximately 1.5$\times$ faster than \AllReduce when using 8 GPUs and achieves a BLEU score that is 0.6 points higher.
While our theory focuses on analyzing SGP, which combines \PushSum with SGD, our experiments illustrate that \PushSum can similarly be efficiently combined with other optimizers like Nesterov's accelerated gradient method~\cite{nesterov1983accelerated} and Adam~\cite{kingma2014adam}.

\section{Preliminaries}

\textbf{Problem formulation.} We consider the setting where a network of $n$ nodes cooperates to solve the stochastic consensus optimization problem
\begin{equation}
\begin{array}{l l}
\min_{\bm{x}_i \in \R^d, i=1,\dots,n} &\frac{1}{n} \sum_{i=1}^n \Exp_{\xi_i \sim D_i} F_i(\bm{x}_i; \xi_i) \\
\text{subject to } &\bm{x}_i = \bm{x}_j, \forall i, j = 1,\dots,n.
\end{array}
\label{eq:prob}
\end{equation}
Each node has local data following a distribution $D_i$, and the nodes wish to cooperate to find the parameters $\bm{x}$ of a DNN that minimizes the average loss with respect to their data, where $F_i$ is the loss function at node $i$.
Moreover, the goal codified in the constraints is for the nodes to reach agreement (\ie, consensus) on the solution they report. We assume that nodes can locally evaluate stochastic gradients $\nabla F_i(\bm{x}_i; \xi_i)$, $\xi_i \sim D_i$, but they must communicate to access information about the objective functions at other nodes.

\textbf{Distributed averaging.} The problem described above encompasses distributed training based on data parallelism, where the canonical approach is parallel stochastic gradient descent: for an overall mini-batch of size $nb$, each node computes a local stochastic mini-batch gradient using $b$ samples, and then the nodes use the \textsc{AllReduce} communication primitive to compute the average gradient at every node. Let $f_i(\bm{x}_i) = \Exp_{\xi_i \sim D_i} F_i(\bm{x}_i; \xi_i)$ denote the objective at node $i$, and let $f(\bm{x}) = \frac{1}{n} \sum_{i=1}^n f_i(\bm{x})$ denote the overall objective. Since $\nabla f(\bm{x}) = \frac{1}{n} \sum_{i=1}^n \nabla f_i(\bm{x})$, averaging gradients via \textsc{AllReduce} provides an exact stochastic gradient of $f$.

\textbf{Approximate distributed averaging.} In this work we explore the alternative approach of using a gossip algorithm for approximate distributed averaging---specifically, the \PushSum algorithm~\cite{kempe2003gossip}. 
Let $\bm{y}_i^{(0)} \in \R^d$ be a vector at node $i$, and consider the goal of computing the average vector $\frac{1}{n} \sum_{i=1}^n \bm{y}_i^{(0)}$ at all nodes. Concatenate the initial vectors into a matrix $\bm{Y}^{(0)} \in \R^{n \times d}$ with one row per node. Typical gossip iterations have the form $\bm{Y}^{(k+1)} = \bm{P}^{(k)} \bm{Y}^{(k)}$ where $\bm{P}^{(k)} \in \R^{n \times n}$ is referred to as the mixing matrix, and conforms to the underlying communication topology. This corresponds to the update $\bm{y}_i^{(k+1)} = \sum_{j=1}^n p_{i,j}^{(k)} \bm{y}_j^{(k)}$ at node $i$. To implement this update, node $i$ only needs to receive messages from other nodes $j$ for which $p_{i,j}^{(k)} \ne 0$, so sparser matrices $\bm{P}^{(k)}$ correspond to less communications.

Drawing inspiration from the theory of Markov chains~\citep{Seneta1981}, the mixing matrices $\bm{P}^{(k)}$ are designed to be column stochastic (each column must sum to 1). Then, under mild conditions (ensuring that information from every node eventually reaches all other nodes) one can show that $\lim_{K \rightarrow \infty} \prod_{k=0}^{K} \bm{P}^{(k)} = \bm{\pi} \ones^\T$, where $\bm{\pi}$ is the ergodic limit of the chain and $\ones$ is a vector with all entries equal to $1$. Consequently, the gossip iterations converge to a limit $\bm{Y}^{(\infty)} = \bm{\pi} (\ones^\T \bm{Y}^{(0)})$; \ie, the value at node $i$ converges to $\bm{y}_i^{(\infty)} = \pi_i \sum_{j=1}^n \bm{y}_j^{(0)}$. 

When the matrices $\bm{P}^{(k)}$ are symmetric, it is straightforward to design the algorithm so that $\pi_i = 1/n$ for all $i$ by making $\bm{P}^{(k)}$ doubly-stochastic (each row and each column must sum to 1). However, symmetric $\bm{P}^{(k)}$ has strong practical ramifications, such as requiring care in the implementation to avoid deadlocks. The \PushSum algorithm only requires that $\bm{P}^{(k)}$ be column-stochastic, and not necessarily symmetric (so node $i$ may send to node $j$, but not necessarily vice versa). However, when the matrices $\bm{P}^{(k)}$ are asymmetric, it is very difficult, and often not possible, to design the algorithm so that $\pi_i = 1/n$. Instead, one additional scalar parameter $w_i^{(k)}$ is maintained at each node. The parameter is initialized to $w_i^{(0)} = 1$ for all $i$, and updated using the same linear iteration, $\bm{w}^{(k+1)} = \bm{P}^{(k)} \bm{w}^{(k)}$. Consequently, the parameter converges to $\bm{w}^{(\infty)} = \bm{\pi} (\ones^\T \bm{w}^{(0)})$, or $w_i^{(\infty)} = \pi_i n$ at node $i$. Thus each node can recover the average of the initial vectors by computing the \emph{de-biased} ratio $\bm{y}_i^{(\infty)} / w_i^{(\infty)}$. In practice, we stop after a finite number of gossip iterations $K$ and compute $\bm{y}_i^{(K)} / w_i^{(K)}$. 


\PushSum provides a mechanism to approximately synchronize parameters across a network of nodes. In the next section we describe the Stochastic Gradient Push algorithm, where \PushSum is used to approximately synchronize (average) parameters across nodes running stochastic gradient descent locally. The same approach can easily be modified to obtain decentralized versions of other popular optimizers such as SGD with momentum or Adam, as illustrated in the experimental section.

\section{Stochastic Gradient Push}
\label{sec:sgp}
\begin{algorithm}[t]
    \small
	\caption{Stochastic Gradient Push (SGP) \label{SGPalg}}
  	\begin{algorithmic}[1]
  	    \REQUIRE{Initialize $\gamma > 0$, $\itr{\bm{x}}{0}_i = \itr{\bm{z}}{0}_i \in \R^d$ and $\itr{w}{0}_i=1$ for all nodes $i \in \{1, 2, \ldots, n \}$}\\
    	\FOR{$k=0,1,2,\cdots, K$, at node $i$,}
        	\STATE {Sample new mini-batch $\itr{\xi}{k}_i \sim {\cD_i}$ from local distribution}
            \STATE {Compute mini-batch gradient at $\itr{\bm{z}}{k}_i$: $\nabla \bm{F}_i(\itr{\bm{z}}{k}_i; \itr{\xi}{k}_i)$}
            \STATE {$\itr{\bm{x}}{k + \frac{1}{2}}_i = \itr{\bm{x}}{k}_i -\gamma \nabla \bm{F}_i(\itr{\bm{z}}{k}_i; \itr{\xi}{k}_i)$}
            \STATE {Send $\big(p_{j,i}^{(k)} \bm{x}_i^{(k+\frac{1}{2})}, p_{j,i}^{(k)} w_i^{(k)}\big)$ to out-neighbors;\par receive $\big(p_{i,j}^{(k)} \bm{x}_j^{(k + \frac{1}{2})}, p_{i,j}^{(k)} w_j^{(k)}\big)$ from in-neighbors}
            \STATE {$\itr{\bm{x}}{k + 1}_i = \sum_{j} \itr{p}{k}_{i,j} \itr{\bm{x}}{k + \frac{1}{2}}_j$}
            \STATE {$\itr{w}{k + 1}_i = \sum_{j} \itr{p}{k}_{i,j} \itr{w}{k}_j$}
            \STATE {$\itr{\bm{z}}{k + 1}_i = \itr{\bm{x}}{k + 1}_i / \itr{w}{k + 1}_i$}
 		\ENDFOR
	\end{algorithmic}
\end{algorithm}

The \emph{Stochastic Gradient Push} (SGP) method \cite{Nedic2016stochastic} for solving \eqref{eq:prob} is obtained by interleaving one local stochastic gradient descent update at each node with one iteration of \PushSum. Pseudocode is shown in Alg.~\ref{SGPalg}. Each node maintains three variables: the model parameters $\bm{x}_i^{(k)}$ at node $i$, the scalar \PushSum weight $w_i^{(k)}$, and the de-biased parameters $\bm{z}_i^{(k)} = \bm{x}_i^{(k)} / w_i^{(k)}$. The vector $\itr{\bm{x}}{0}_i$ can be initialized arbitrarily.  At each iteration, every node performs a local SGD step (lines 2--4) followed by one step of \PushSum for approximate distributed averaging (lines 5--8). Note that the gradients are evaluated at the de-biased parameters $\bm{z}_i^{(k)}$ in line~3, and they are then used to update $\bm{x}_i^{(k)}$, in line~4. All communication takes place in line~5, and each message contains two parts, the \PushSum numerator $\bm{x}_i^{(k)}$ and \PushSum weight $w_i^{(k)}$. 


The non-zero entries in the mixing matrix $\bm{P}^{(k)}$ define the communication topology at each iteration $k$.  SGP can leverage various communication topologies including sparse, asymmetric or time-varying networks. We are mainly interested in the case where the mixing matrices $\bm{P}^{(k)}$ are sparse in order to have low communication overhead. However, we point out that when the nodes' initial values are identical, $\bm{x}_i^{(0)} = \bm{x}_j^{(0)}$ for all $i,j \in [n]$, and every entry of $\bm{P}^{(k)}$ is equal to $1/n$, then SGP is mathematically equivalent to parallel SGD using \textsc{AllReduce}. 

Individual nodes do not need to  know the entire mixing matrix at each time step. Each node $i$ must only know/choose its outgoing mixing weights, which correspond to the $i^{\text{th}}$ column of $\bm{P}^{(k)}$. Each node can consequently choose its mixing weights independently of the other nodes in the network. In Appendix~\ref{sec:implementation_details} we describe one way to design a sequence of mixing matrices that satisfies the requirements of our theory (described in the next section) and for which each node sends and receives exactly one message at every iteration; all appendices are in the Supplementary Material.

\textbf{Overlapping communication and computation.}
Although SGP does not use network-wide collective communication primitives like \AllReduce, the implementation of Alg.~\ref{SGPalg} requires using blocking sends and receives; \ie, nodes do not proceed to line~6 until they have received messages from all in-neighbors at that iteration. 
To hide the communication overhead, we can overlap gradient computation with communication.
For a given $\tau \geq 0$, nodes send messages to their out-neighbours every $\tau$ iterations (non-blocking), and can receive incoming messages at any time in-between communication intervals.
If a node hasn't received messages from its in-neighbors after $\tau$ iterations, then it blocks and waits to receive the messages.

Specifically, the communication in line 5 in Alg.~\ref{SGPalg} is made non-blocking, and each node may perform $\tau$ gradient update steps while it occurs.
This may result in the gossip updates in lines 6 and 7 incorporating outdated messages, $(p_{j,i}^{(k^\prime)} \bm{x}_i^{(k^\prime +\frac{1}{2})}, p_{j,i}^{(k^\prime )} w_i^{(k^\prime )})$, where $k - k^\prime \le \tau$.
However, as long as the delay $k - k'$ remains bounded, SGP is still guaranteed to converge (see Theorem~2 below).
We refer to this method, with delay bound $\tau$ as $\tau$-overlap SGP ($\tau$-OSGP). SGP is equivalent to $\tau$-OSGP with $\tau=0$.
In practice we find that taking $\tau$ to be 1 or 2 is sufficient to hide effectively all of the communication overhead. Pseudocode for $\tau$-OSGP is provided in Appendix~\ref{appendix:osgp}.




\section{Theoretical guarantees.}
SGP was first proposed and analyzed in \citep{Nedic2016stochastic} assuming the local objectives $f_i(\bm{x})$ are strongly convex. Here we provide convergence results in the more general setting of smooth, non-convex objectives with arbitrary, but bounded, message delays. We make the following four assumptions:
\begin{enumerate}
    \item ($L$-smooth) There exists a constant $L > 0$ such that $\| \nabla f_i(\bm{x}) - \nabla f_i(\bm{y})\| \le L \|\bm{x} - \bm{y}\|$.
    \item (Bounded variance) There exist finite positive constants $\sigma^2$ and $\zeta^2$ such that
    \begin{align*}
    \Exp_{\xi \sim D_i} \|\nabla F_i(\bm{x}; \xi) -\nabla f_i(\bm{x})\|^2 &\leq \sigma^2 \quad \forall i, \forall \bm{x}, \text{ and } \\
    \label{boundedVariance2}
    \frac{1}{n} \sum_{i=1}^n \| \nabla f_i(\bm{x}) - \nabla f(\bm{x}) \|^2 &\leq \zeta^2  \quad \forall \bm{x}.
    \end{align*}
    Thus $\sigma^2$ bounds the variance of stochastic gradients at each node, and $\zeta^2$ quantifies the similarity of data distributions at different nodes.

    \item (Bounded delay) There exists a finite constant $\tau \in \mathbb{Z}_{+}$, such that the delay, if overlapping communication and computation, satisfies $k^\prime - k \leq \tau$.
    
    \item (Mixing connectivity) To each mixing matrix $\bm{P}^{(k)}$ we can associate a graph with vertex set $\{1,\dots,n\}$ and edge set $E^{(k)} = \{(i,j) \colon p_{i,j}^{(k)} > 0\}$; \ie, with edges $(i,j)$ from $j$ to $i$ if $i$ receives a message from $j$ at iteration $k$. By convention, we take each node to be an in-neighbor of itself (each node in the graph has a self-loop), and we assume that there exists finite, positive integers, $B$ and $\Delta$, such that the graph with edge set $\bigcup_{k = lB}^{(l+1)B - 1} E^{(k)}$ is strongly connected and has diameter at most $\Delta$ for every $l \ge 0$.
    \footnote{For the purpose of analysis, we model delays by augmenting the mixing-matrices $\bm{P}^{(k)}$, and the corresponding graph topologies, with virtual nodes and edges that store the state of information that was transmitted, but not yet received. We omit this description from the main text to simplify the discussion, and relegate this discussion to the supplementary material \textit{Modeling message delays} in Appendix~\ref{sec:proofs}.}
\end{enumerate}

\smallskip
Let $\xbar^{(k)} = \frac{1}{n}\sum_{i=1}^n \bm{x}_i^{(k)}$. \citet{lian2017can} define that a decentralized algorithm for solving \eqref{eq:prob} converges if, for any $\epsilon > 0$, it eventually satisfies
\begin{equation}
\frac{1}{K} \sum_{k=1}^K \Exp{} \| \nabla f(\xbar^{(k)}) \|^2 \le \epsilon. \label{eq:convergence-criterion1}
\end{equation}
We show that SGP converges in this sense.



\begin{thm} \label{thm:avg-convergence}
Suppose that Assumptions~1--4 hold, and run SGP for $K$ iterations with step-size $\gamma = \sqrt{n / K}$. Let $f^* = \inf_{\bm{x}} f(\bm{x})$ and assume that $f^* > -\infty$.
There exist constants $C$ and $q \in [0,1)$, which depend on the diameter of the network $\Delta$, the upper bound on the delays $\tau$, and the sequence of mixing matrices $\bm{P}^{(k)}$, such that when the total number of iterations satisfies
\begin{align*}
	K \geq \max \Bigg\{& \frac{nL^4C^4 60^2}{(1-q)^4}, \frac{L^4C^4 P_1^2 n}{(1-q)^4 ( f( \itr{\xbar}{0} ) - f^*+ \frac{L\sigma^2}{2})^2}, \\
	&\frac{L^2 C^2 n P_2}{(1-q)^2 ( f( \itr{\xbar}{0} ) - f^*+ \frac{L\sigma^2}{2})}, n \Bigg\}
	\numberthis \label{BoundK}
\end{align*}
where $P_1= 4(\sigma^2+3\zeta^2)n+\frac{\sum^n_{i=1} \norm{ \itr{\bm{x}_i}{0} }^2}{n}$ and $P_2=\sigma^2+3\zeta^2L^2C^2 +2\frac{\sum^n_{i=1} \norm{ \itr{\bm{x}_i}{0} }^2}{n}$,
then
$$
\frac{\sum^{K-1}_{k=0} \E \norm{ \nabla f(\itr{\xbar}{k}) }^2 }{K} \leq \frac{12 ( f( \itr{\xbar}{0} ) - f^*+ \frac{L\sigma^2}{2})}{\sqrt{nK}}.
$$
\end{thm}
The proof is given in Appendix~\ref{sec:proofs}, where we also provide precise expressions for the constants $C$ and $q$.

Theorem~\ref{thm:avg-convergence} shows that the average of the nodes' parameters, $\xbar^{(k)}$, converges, but it does not directly say anything about the parameters at each node. In fact, we can show that:

\begin{thm} 
\label{secondTheorem}
Under the same assumptions as in Theorem~\ref{thm:avg-convergence},
\[
	\frac{1}{nK} \sum_{k=0}^{K-1} \sum_{i=1}^n \Exp{} \norm{\xbar^{(k)} - \bm{z}_i^{(k)}}^2 \le \order{\frac{1}{K} + \frac{1}{K^{3/2}}},
\]
and
\[
    \frac{1}{nK}\sum_{k=0}^{K-1}\sum^n_{i=1} \Exp \norm{\nabla f(\bm{z}_i^k)}^2 \leq O\left(\frac{1}{\sqrt{nK}} + \frac{1}{K} + \frac{1}{K^{3/2}}\right).
\]
\end{thm}

The proof is also given in Appendix~\ref{sec:proofs}. This result shows that as $K$ grows, the de-biased variables $\bm{z}_i^{(k)}$ converge to the node-wise average $\xbar^{(k)}$, and hence the de-biased variables at each node also converge to a stationary point. Note that for fixed $n$ and large $K$, the $1/\sqrt{nK}$ term will dominate the other factors.

\section{Related Work}

A variety of approaches have been proposed to accelerate distributed training of DNNs in the communcation bound setting, including 
quantizing gradients \citep{Alistarh2017qsgd,Wen2017terngrad,jia2018highly} and performing multiple local SGD steps at each node before averaging \citep{McMahan2017federated}.These approaches are complementary to the approach considered in this paper, which advocates for using approximate rather than exact distributed averaging. Quantizing gradients, performing multiple local SGD steps between averaging, and using approximate distributed averaging can all be seen as injecting additional noise (due to approximations) into SGD, leading to a tradeoff between reducing communication (towards training faster) and potentially obtaining worse predictive accuracy (due to approximations). Combining these approaches (quantized, infrequent, and inexact averaging) is an interesting direction for future work.

\citet{Blot2016gossip} report initial experimental results on small-scale experiments with an SGP-like algorithm. \citet{Jin2016how} make a theoretical connection between \PushSum-based methods and Elastic Averaging SGD \citep{Zhang2015elasticsgd}. Relative to those previous works, we provide the first convergence analysis for a \PushSum-based method in the smooth non-convex case. Moreover, our analysis also holds in the presence of bounded message delays. 

\citet{lian2017can} and \citet{Jiang2017collaborative} study synchronous gossip-based versions of SGD. Those methods involve symmetric message passing which inherently involves blocking (if $i$ sends to $j$ at iteration $k$, then $j$ also sends to $i$ before both nodes update). Consequently, they are slower in communication-constrained settings, in comparison to \PushSum-based SGP which may use directed message passing ($i$ can send to $j$ without needing a response). 
The work of~\citet{jakovetic2014fast} studies gossip-based versions of the Nesterov gradient method for smooth strongly-convex functions with deterministic gradients.
The method requires performing two rounds of symmetric message passing per gradient update, and consequently is also slower in communication-constrained settings.

\emph{Decentralized parallel SGD} (D-PSGD) \cite{lian2017can} produces iterates whose node-wise average, $\xbar^{(k)}$, converges in the sense of \eqref{eq:convergence-criterion1}. Our results in Sec.~\ref{sec:sgp} show that SGP converges in the same sense and go beyond to show that the individual values at each node also converge to a stationary point, since the values at each node converge to the network-wide average.
In fact, SGP is a generalization of D-PSGD: when the communication topology is static, undirected, and connected at every iteration, and when nodes use symmetric mixing weights ($p_{j,i}^{(k)} = p_{i,j}^{(k)}$ for all $(i,j)$), then the push-sum weights $w_i^{(k)}$ are always equal to $1$ and SGP is mathematically equivalent to D-PSGD.
We compare SGP with D-PSGD experimentally in Section~\ref{sec:experiments} and find that SGP is consistently faster and the two methods find solutions of comparable accuracy.

\citet{Jin2016how} and \citet{Lian2018asynchronous} study asynchronous gossip-based methods for training DNNs. \citet{Lian2018asynchronous} analyzes an asynchronous version of D-PSGD and proves that its node-wise averages also converge to a stationary point.
In general, these contributions focusing on asynchrony can be seen as orthogonal to the use of a \PushSum based protocol for approximate distributed averaging.
Moreover, we find that synchronous Overlap SGP runs faster than asynchronous state-of-the-art AD-PSGD, and produces models with better training/validation performance.

\section{Experiments}
\label{sec:experiments}
Next we experimentally compare SGP with \AllReduce SGD (AR-SGD), D-PSGD, and asynchronous D-PSGD (AD-PSGD).
We aim to study the relationship between runtime and predictive accuracy as a function of the number of nodes used and the network bandwidth. In low-bandwidth experiments the servers communicate over 10 Gbps Ethernet links (typical in data centers) and in high-bandwidth experiments they communicate over 100 Gbps InfiniBand (typical in high-performance computing clusters). To illustrate the versatility of SGP, we consider two typical workloads: image classification and machine translation. All algorithms are implemented in PyTorch~\citep{paszkepytorch}. 

While our SGP analysis focuses solely on the combination of \PushSum with SGD, we leverage Nesterov momentum or Adam in practice.
Each node sends and receives one message per iteration in our SGP baseline implementation, and the destination and source of these messages changes from one iteration to the next.
We refer the reader to Appendix~\ref{sec:implementation_details} for implementation details, including how we design the sequence of mixing matrices~$\bm{P}^{(k)}$.

\subsection{Image classification}
We train a ResNet-50~\citep{he2016deep} on the ImageNet classification task~\citep{russakovsky2015imagenet}.
\footnote{ImageNet was only used for the non-commercial research purposes of this paper and not for training networks deployed in production or for other commercial purposes.}
Our experiments use 32 NVIDIA DGX-1 servers. Each server has 8 V100 GPUs. To investigate scaling we run experiments using 4, 8, 16, and 32 servers (\ie, 32, 64, 128, and 256 GPUs).
We follow the experimental protocol of~\citet{goyal2017accurate}. Every node uses a mini-batch size of 256, so using more nodes corresponds to larger effective mini-batch size. Unless indicated otherwise, all experiments are run for 90 epochs, the learning rate warms up to $n \times 0.1$ during the first five epochs following~\citet{goyal2017accurate} and is decayed by a factor of 10 at epochs 30, 60, and 80. All methods use Nesterov momentum.
\begin{figure*}[t]
	\centering
	\subfloat[]{\includegraphics[width=0.25\textwidth]{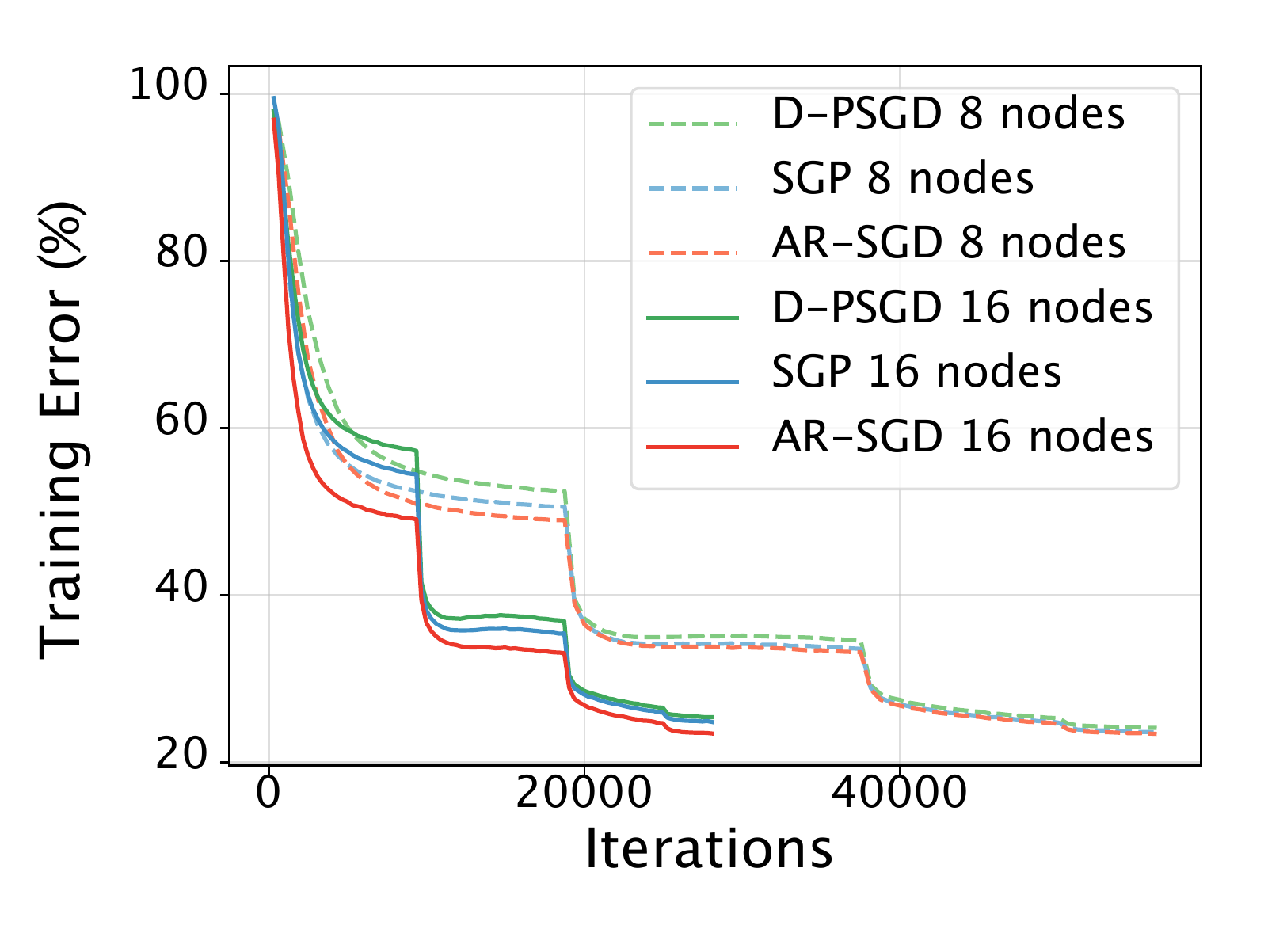}}
	\subfloat[]{\includegraphics[width=0.25\textwidth]{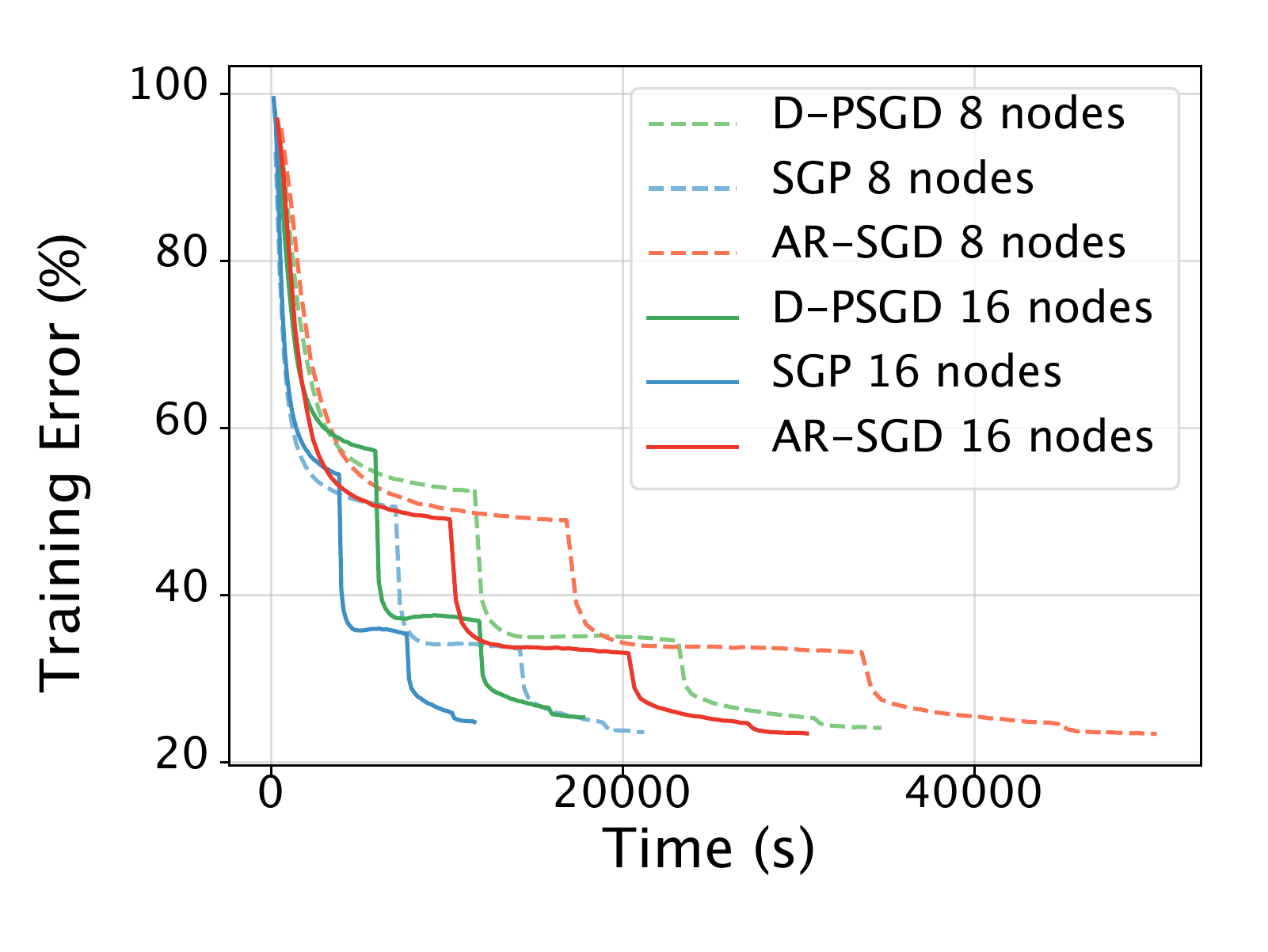}}
 	\subfloat[]{\includegraphics[width=0.25\textwidth]{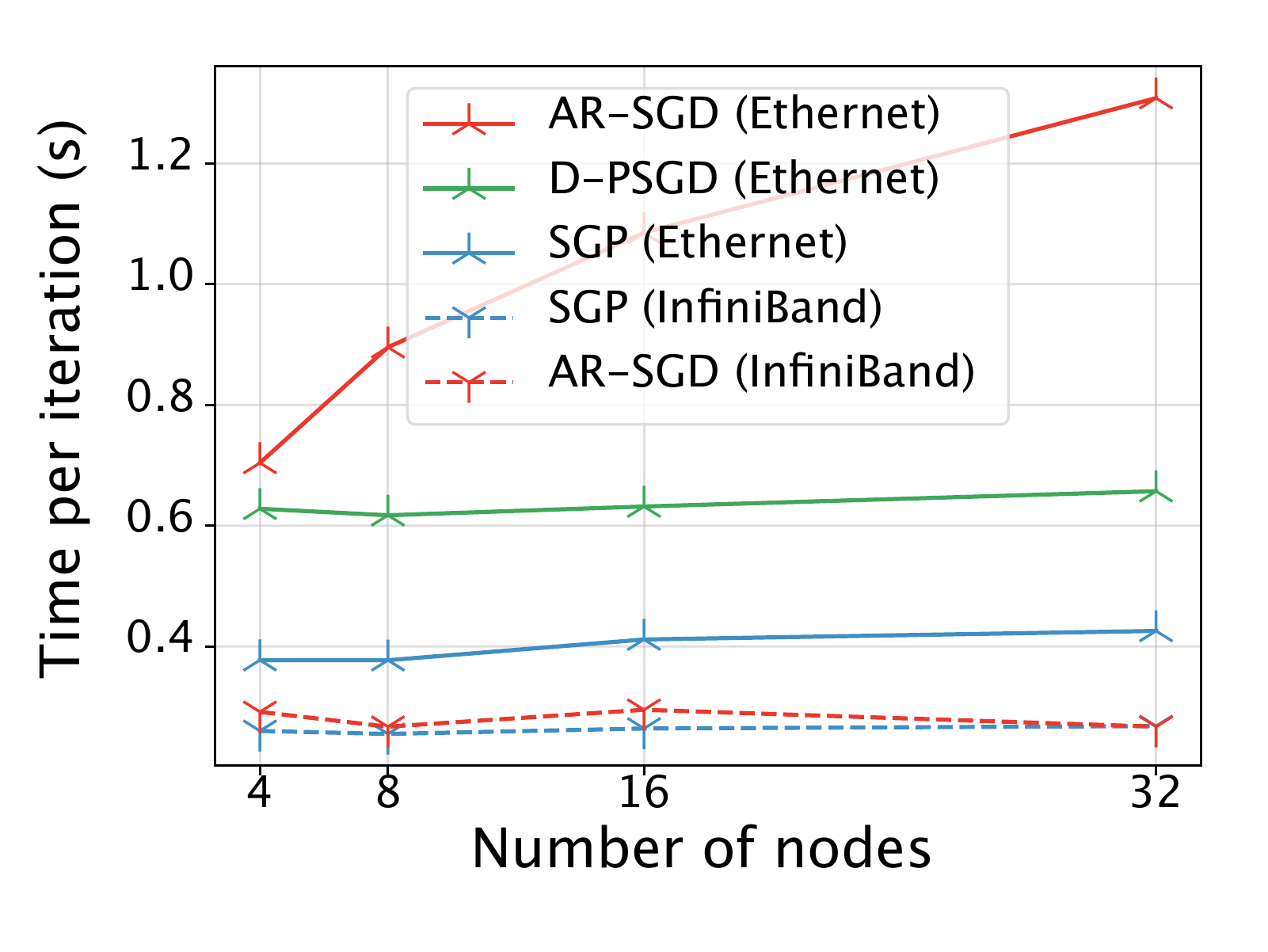}}
	\subfloat[]{\includegraphics[width=0.25\textwidth]{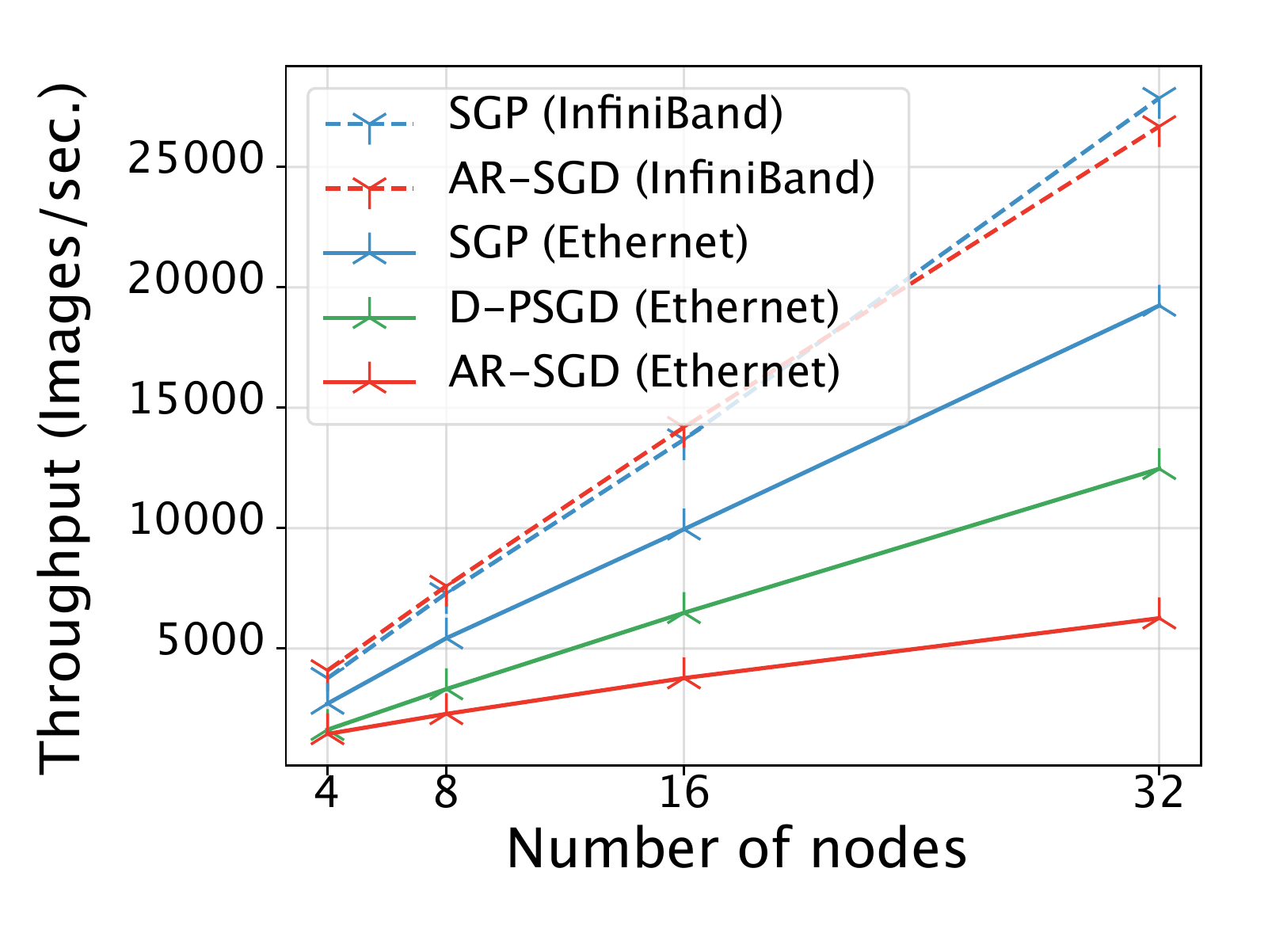}}
	\caption{Scaling and convergence results on 4--32 nodes interconnected via 10 Gbps Ethernet and 100Gbps-InfiniBand for \AllReduce-SGD (AR-SGD), SGP and D-PSGD. (a)/(b): Iteration-wise convergences and time-wise convergence over 10 Gbps Ethernet. (c)/(d): Time-wise scaling efficiency over Ethernet and InfiniBand networks.}
	\label{fig:img_convergence}
\end{figure*}
\begin{table*}[t]
\centering
 \small
 \begin{tabular}{l c c c c} \toprule
  & \textbf{4 nodes (32 GPUs)} & \textbf{8 nodes (64 GPUs)} & \textbf{16 nodes (128 GPUs)} & \textbf{32 nodes (256 GPUs)} \\ \midrule
 AR-SGD &
    \begin{tabular}{l l} $76.2$\% & $22.0$ hrs. \end{tabular} &
    \begin{tabular}{r l} $76.4$\% & $14.0$ hrs. \end{tabular} &
    \begin{tabular}{r l} $76.3$\% & $8.5$ hrs. \end{tabular} &
    \begin{tabular}{r l} $76.2$\% & $5.1$ hrs. \end{tabular} \\
 D-PSGD & 
     \begin{tabular}{r l} $76.4$\% & $19.7$ hrs. \end{tabular} &
    \begin{tabular}{r l} $76.1$\% & $9.7$ hrs. \end{tabular} &
    \begin{tabular}{r l} $75.9$\% & $5.0$ hrs. \end{tabular} &
    \begin{tabular}{r l} $74.4$\% & $2.6$ hrs. \end{tabular} \\ \midrule
 SGP &
    \begin{tabular}{r l} $76.3$\% & $11.8$ hrs. \end{tabular} &
    \begin{tabular}{r l} $76.4$\% & $5.9$ hrs. \end{tabular} &
    \begin{tabular}{r l} $75.9$\% & $3.2$ hrs. \end{tabular} &
    \begin{tabular}{r l} $75.0$\% & $1.7$ hrs. \end{tabular} \\ \bottomrule
\end{tabular}
 \caption{Top-1 validation accuracy (\%) and training time (hours), when communicating over 10 Gbps Ethernet for \AllReduce-SGD (AR-SGD), SGP and D-PSGD. SGP and D-PSGD are using 1-peer communication topologies.}
 \label{tb:img_convergence_eth}
\end{table*}

\paragraph{Scaling and convergence.}
The first set of experiments studies the scaling and convergence properties of our baseline SGP implementation, where every node sends and receives one message at every iteration (1-peer) and we do not overlap communication and computation (\ie, $\tau = 0$).

Figure~\ref{fig:img_convergence}~(a) shows the iteration-wise training convergence of SGP, \AllReduce-SGD, and D-PSGD.
Note that when we increase the number of nodes by a factor of 2, we also decrease the total number of iterations by the same factor.
Figure~\ref{fig:img_convergence}~(b) shows the time-wise training convergence over 8 and 16 node Ethernet networks.
In all cases, SGP completes 90 epochs in less time than \AllReduce-SGD and D-PSGD.
Figures~\ref{fig:img_convergence}~(c) and~(d) show the scaling efficiency of the methods on both 10~Gbps Ethernet and 100~Gbps InfiniBand networks.
In the case of the InfiniBand network, all methods exhibit a near linear scaling (constant time per iteration), which is expected since communication is not a bottleneck in this setting.
On the other hand, over the 10 Gbps Ethernet network, as we increase the number of nodes, the average iteration time stays almost constant for SGP and D-PSGD, while the per-iteration time of \AllReduce-SGD significantly increases, resulting in an overall slower training time. Moreover, although D-PSGD and SGP both exhibit strong scaling, SGP is roughly 1.5$\times$ faster over 10~Gbps Ethernet, supporting the claim that it involves less communication overhead.

Table~\ref{tb:img_convergence_eth} shows the total training time and top-1 validation accuracy of the different runs over the 10~Gbps Ethernet network using the baseline 1-peer topologies. For any number of nodes used in our experiments, we observe that SGP consistently outperforms D-PSGD and \AllReduce-SGD in terms of total training time. In particular, for 32 node networks (256 GPUs), SGP training takes approximately $1.7$ hours, while D-PSGD and \AllReduce-SGD require roughly $2.6$ and $5.1$ hours respectively.

To get a sense of the difference in runtimes between Ethernet and InfiniBand, and also to illustrate the variability, Table~\ref{tb:img_seeds} shows the mean training time and top-1 validation accuracy along with the maximum absolute deviation from the mean, for 4- and 16-node experiments run on 100~Gbps InfiniBand networks. The max.~absolute deviations are calculated based on five runs of each algorithm, using five different seeds. Even with the high-bandwidth InfiniBand network, SGP exhibits less variation in training time across different runs, supporting the claim that SGP helps reduce the effects of stragglers or other sources of latency.

\begin{table}[ht]
\centering
\small
\begin{tabular}{l c c} \toprule
    & \textbf{4 nodes (32 GPUs)} & \textbf{16 nodes (128 GPUs)} \\ \midrule
    AR-SGD &
        \begin{tabular}{r} $76.3 \pm 0.2$\% \\ $8.8 \pm 0.4$ hrs. \end{tabular} &
        \begin{tabular}{r} $76.2 \pm 0.2$\% \\ $2.5 \pm 0.3$ hrs. \end{tabular} \\ \midrule
    SGP &
        \begin{tabular}{r} $76.3 \pm 0.2$\% \\ $8.2 \pm 0.1$ hrs. \end{tabular} &
        \begin{tabular}{r} $75.8 \pm 0.2$\% \\ $2.2 \pm 0.1$ hrs. \end{tabular} \\ \bottomrule
\end{tabular}
\caption{Top-1 validation accuracy and training time statistics (mean $\pm$ max.~abs.~deviation), across 5 different runs of each algorithm using 5 different seeds. All methods complete 90 epochs, communicating over 100~Gbps InfiniBand.}
\label{tb:img_seeds}
\end{table}

All methods achieve roughly the same validation accuracy for smaller 1-peer topologies (4 and 8 nodes), and the accuracy of D-PSGD and SGP degrades for larger topologies (16 and 32 nodes). We hypothesize that this is due to the error introduced by approximate distributed averaging---models at different nodes are only approximately identical, and for larger networks the divergence between models at different nodes is larger. We investigate this further next.

\begin{figure}[t]
	\centering
	\subfloat[Sparse Graph Topology]{\includegraphics[width=0.5\columnwidth]{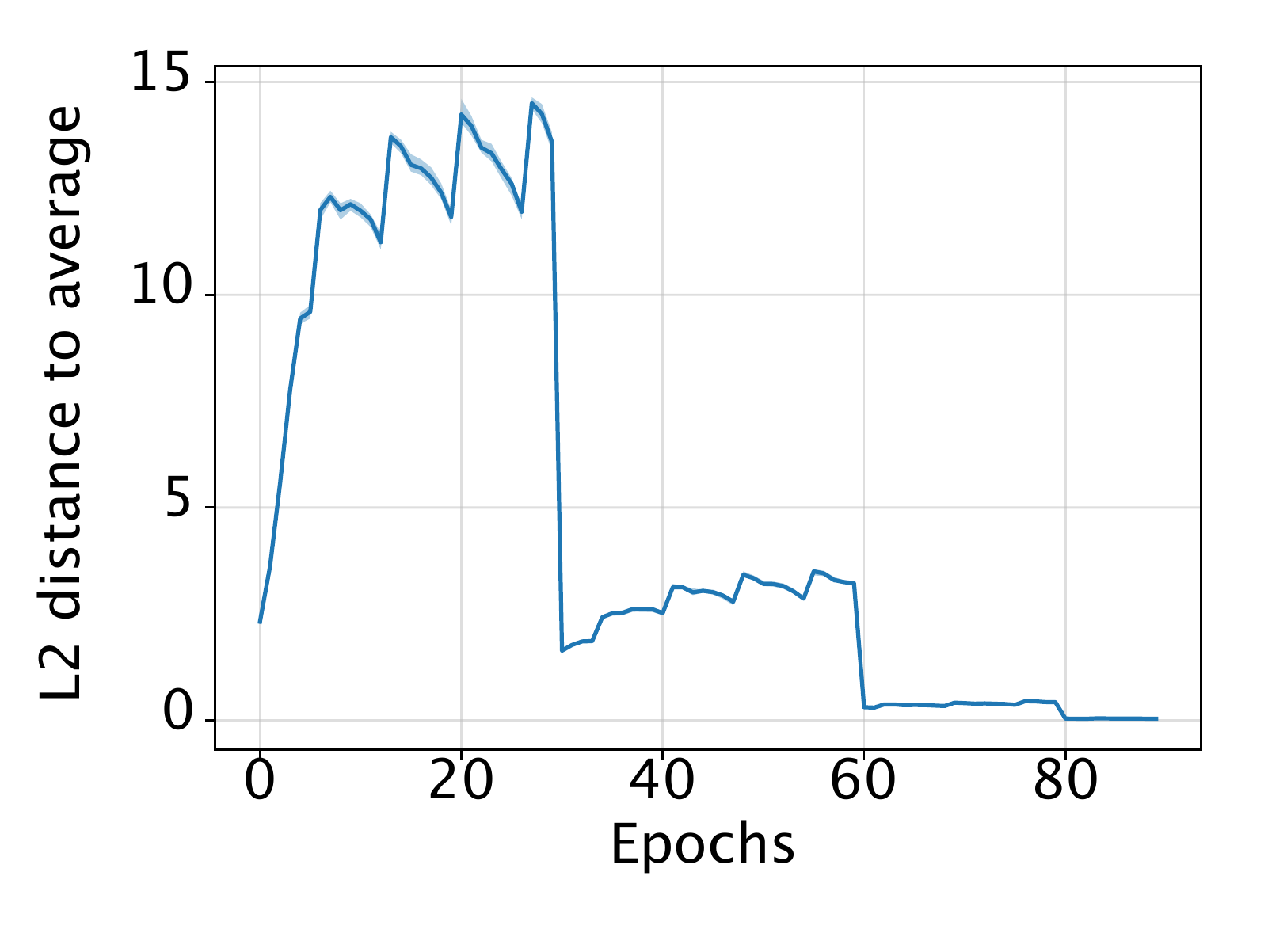}}
	\subfloat[Dense Graph Topology]{\includegraphics[width=0.5\columnwidth]{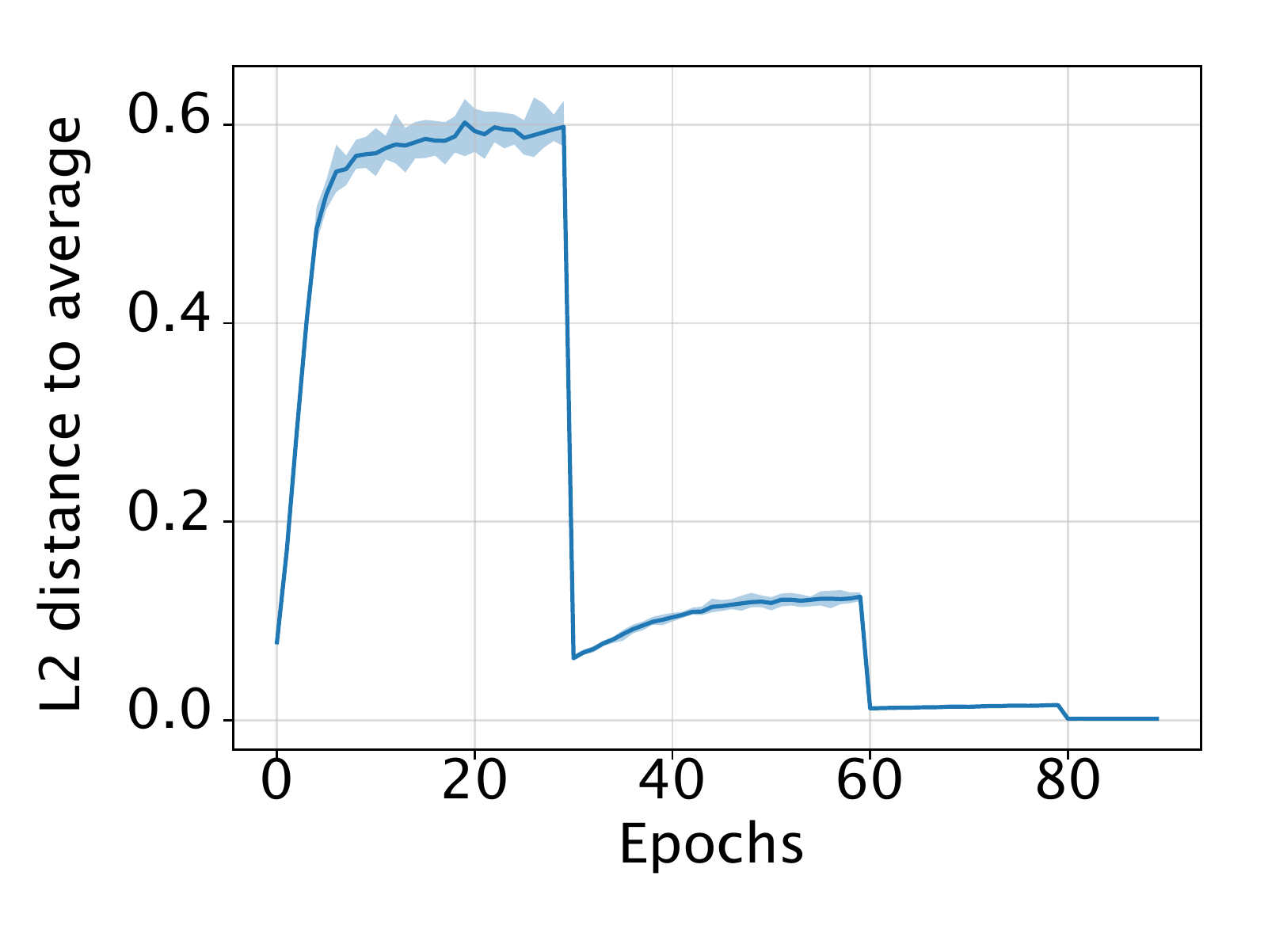}}
	\caption{Parameter Deviations for ResNet-50 trained on ImageNet (using SGP) over 16 node. Figures show the average Euclidean distance between the parameters at individual nodes and the node-wise average; recorded at the end of each training epoch, after the last gradient step, but before the last gossip step. Shaded regions show the max.~and min.~parameter deviations across all nodes.
	Nodes' parameters exhibit much greater deviation from the average when communicating over sparse topologies, than when communicating over dense topologies.}
	\label{fig:param_deviations}
\end{figure}
\paragraph{Parameter deviations.} We track the deviations between the model parameters at different nodes while training with SGP. Our convergence result (cf.~Lemma~\ref{MidosTheorem} in Appendix~\ref{sec:proofs}) states that each node converges exponentially fast to a neighborhood of the average, where the size of the neighborhood is proportional to the step-size and the connectivity of the graph topology.

Figure~\ref{fig:param_deviations} shows the average Euclidean distance between the parameters at individual nodes and the node-wise average, at the end of each epoch, for two different network configurations. The sparse graph corresponds to the time-varying 1-peer communication strategy used in the experiments above. The dense graph is fully-connected (all-to-all), and the deviations are computed just after line~4 in Alg.~\ref{SGPalg}, before communicating.

Indeed, we see that distance from the mean model is proportional to the magnitude of the step-size, and each node's parameters are approximately equidistant from the average.
At epochs 30, 60, and 80, the parameter deviations drop by an order of magnitude, in concert with the learning rate.
We also see the parameter deviations gradually increase in the first five epochs, following the learning rate warm-up.
The relation between the parameter deviations and the communication topology is also evident. Indeed, the dense topology exhibits significantly less parameter deviations than the time-varying sparsely connected topology (1-peer communication topology). Therefore, one can directly control the parameter deviations by adjusting the learning rate and/or graph topology, as predicted by Lemma~\ref{MidosTheorem}.

\paragraph{Communication and the speed-accuracy tradeoff.} Next we explore the effect of the  communication topology on the speed-accuracy tradeoff when training with 16 and 32 nodes (128 and 256 GPUs) over 10~Gbps Ethernet.
Table~\ref{tb:img_convergence_adp} shows the validation accuracy and wall-clock time for SGP using a 1-peer topology (1P-SGP), and SGP using a 2-peer topology (2P-SGP), \ie, each node sends and receives to two peers at each iteration. Using just this one additional neighbor improves the validation accuracy of SGP to 76.2\% in the 16 nodes case and to 75.7\% for 32 nodes, while retaining much of the speed advantages.
\begin{table}[t]
\centering
\small
\begin{tabular}{r c c} \toprule
    & \textbf{16 nodes (128 GPUs)} & \textbf{32 nodes (256 GPUs)} \\ \midrule
    AR-SGD &
        \begin{tabular}{l r} $76.3$\% & $8.5$ hrs. \end{tabular} &
        \begin{tabular}{l r} $76.2$\% & $5.2$ hrs. \end{tabular} \\
    2P-SGP &
        \begin{tabular}{l r} $76.2$\% & $5.1$ hrs. \end{tabular} &
        \begin{tabular}{l r} $75.7$\% & $2.5$ hrs. \end{tabular} \\
    1P-SGP &
        \begin{tabular}{l r} $75.9$\% & $3.2$ hrs. \end{tabular} &
        \begin{tabular}{l r} $75.0$\% & $1.7$ hrs. \end{tabular} \\ \midrule
    AR/1P-SGP &
        \begin{tabular}{l r} $76.2$\% & $4.8$ hrs. \end{tabular} &
        \begin{tabular}{lr} $75.4$\% & $2.8$ hrs. \end{tabular} \\
    2P/1P-SGP &
        \begin{tabular}{l r} $76.0$\% & $3.5$ hrs. \end{tabular} &
        \begin{tabular}{l r} $75.1$\% & $1.8$ hrs. \end{tabular} \\ \bottomrule
\end{tabular}
\caption{Top-1 validation accuracies (\%) and training time (hours) for 1P-SGP (1-peer topology); 2P-SGP (2-peer topology), AR-SGD (\AllReduce SGD), AR/1P-SGP (\AllReduce first 30 epochs, 1-peer topology last 60 epochs), and 2P/1P-SGP (2-peer topology first 30 epochs, 1-peer topology last 60 epochs), all communicating over 10~Gbps Ethernet.}
\label{tb:img_convergence_adp}
\end{table}
We also experiment with hybrid communication schemes that use more communication at the start of training, to mitigate the parameter deviations when they are largest (cf. Figure~\ref{fig:param_deviations}).
Table~\ref{tb:img_convergence_adp} compares 1P-SGP, 2P-SGP, and \AllReduce SGD (mathematically equivalent to running SGP with a fully-connected topology) with two hybrid methods: AR/1P-SGP, which uses \AllReduce for the first 30 epochs and 1-peer SGP for the last 60 epochs, and 2P/1P-SGP, which uses 2-peer SGP for the first 30-epochs and 1-peer SGP for the remainder.
We find that these hybrid communication schemes provide a balance between speed and accuracy, and that communicating more during the first few epochs of training can mitigate accuracy tradeoffs.

\paragraph{Overlap SGP.}
Table~\ref{tb:img_convergence_ovp} compares 1-OSGP using a 1-peer communication topology, with AD-PSGD, D-PSGD, and a biased implementation of 1-OSGP that directly incorporates delayed messages without accounting for the bias in the push-sum weight.
\begin{table}[t]
\centering
\small
\begin{tabular}{r c} \toprule
    & \begin{tabular}{c c c} Train Acc. & Val. Acc. & Train Time\end{tabular} \\ \midrule
    AR-SGD &
        \begin{tabular}{c c c} $76.9\%$ & $76.3$\% & $8.5$ hrs. \end{tabular} \\
    D-PSGD &
        \begin{tabular}{c c c} $75.6$\% & $75.9$\% & $4.9$ hrs. \end{tabular}  \\
    AD-PSGD &
        \begin{tabular}{c c c} $74.7$\% & $75.5$\% & $2.9$ hrs. \end{tabular}  \\
    SGP &
        \begin{tabular}{c c c} $75.6$\% & $75.9$\% & $3.2$ hrs. \end{tabular}  \\
    biased 1-OSGP &
        \begin{tabular}{c c c} $75.4$\% & $75.3$\% & $1.8$ hrs. \end{tabular}  \\ \midrule
    1-OSGP &
        \begin{tabular}{c c c} $77.1$\% & $75.7$\% & $1.8$ hrs. \end{tabular}  \\ \bottomrule
\end{tabular}
\caption{Comparing state-of-the-art synchronous and asynchronous gossip-based approaches to 1-OSGP, an implementation of synchronous SGP where communication is overlapped with 1 gradient step (all messages are always received with 1-iteration of staleness). 1-OSGP is also compared with a biased implementation of 1-OSGP that directly incorporates delayed messages without accounting for the bias in the push-sum weight. Experiments are run for 90 epochs over 16 nodes (128 GPUs) interconnected via 10~Gbps Ethernet.}
\label{tb:img_convergence_ovp}
\end{table}
Overlapping communication and computation greatly speeds up training and results in no accuracy degradation relative to non-overlap SGP. In contrast, the biased implementation of 1-OSGP (not using the \PushSum weight) has significantly worse accuracy, supporting our theoretical development that the bias tracked in the push-sum weight facilitates convergence. Moreover, synchronous 1-OSGP runs faster than the state-of-the-art asynchronous AD-PSGD, and achieves better training and validation accuracy.
\begin{table}[t]
\centering
\small
\begin{tabular}{r c c c} \toprule
    & Train Acc. & Val. Acc. & Train Time \\ \midrule
    AR-SGD & $76.9\%$ & $76.2$\% & $5.1$ hrs. (90 epochs) \\
    AD-PSGD & $80.3\%$ & $76.9$\% & $4.7$ hrs. (270 epochs) \\ \midrule
    SGP & $80.0$\% & $77.1$\% & $4.6$ hrs. (270 epochs) \\
    1-OSGP & $81.8$\% & $77.1$\% & $2.7$ hrs. (270 epochs) \\ \bottomrule
\end{tabular}
\caption{Comparing \AllReduce SGD (AR-SGD) and SGP under a fix runtime budget. Given a similar runtime, SGP outperforms SGD for both training and validation accuracy. Running 1-OSGP for the same number of epochs than SGP outperforms SGD while improving the overal training efficiency. Experiments are run over 1-peer graph topologies, using 32 nodes (256 GPUs) interconnected via 10~Gbps Ethernet.}
\label{tb:fix_runtime_exp}
\end{table}

\paragraph{Fixed runtime budget.} The results so far highlight that SGP completes 90 epochs of training faster than \AllReduce SGD, especially in communication-constrained settings, and this comes at the cost of reduced accuracy. However, if we compare the algorithms based on a runtime budget rather than an epoch budget, SGP achieves superior results. Specifically, since the baseline SGP is roughly 3$\times$ faster than \AllReduce SGD when run on 32 nodes over 10~Gbps Ethernet, we run SGP for 270 epochs (instead of 90 epochs) using a scaled learning rate schedule (similar warm-up, but decaying the learning rate by a factor of 10 at epochs 90, 180, and 240).
With this setup (32 nodes/256 GPUs and 10 Gbps Ethernet), SGP
surpasses the best \AllReduce SGD accuracy ($76.2\%$ after 90 epochs/5.1 hrs.), and achieves a top-1 validation accuracy of $77.1\%$ at then end of the 270 epochs ($4.6$ hrs.).
Similarly, overlap SGP 
achieves a top-1 validation accuracy of $77.1\%$ at then end of the 270 epochs ($2.7$ hrs.).
When run longer, AD-PSGD achieves better accuracy than \AllReduce SGD, but not as fast nor as high of an accuracy as achieved by 1-OSGP.

\subsection{Neural Machine Translation}
We train a transformer network~\cite{vaswani2017attention} on WMT16-En-De using our baseline implementation of SGP, and utilizing the same hyperparameters as~\cite{vaswani2017attention}. We train models using both $25K$ token batches~\cite{vaswani2017attention}, and $400K$ token batches~\cite{ott2018scaling}. 
All methods use Adam. Training is performed on 8 NVIDIA V100 GPUs, located on 8 different machines, interconnected via a 10 Gbps Ethernet link. This setup is similar to the AWS \textsc{p3.2xlarge} machines commonly used for distributed training~\cite{bernstein2018signsgd}.
\begin{figure}[t]
	\centering
	\subfloat[Iteration-wise convergence]{\includegraphics[width=.5\columnwidth]{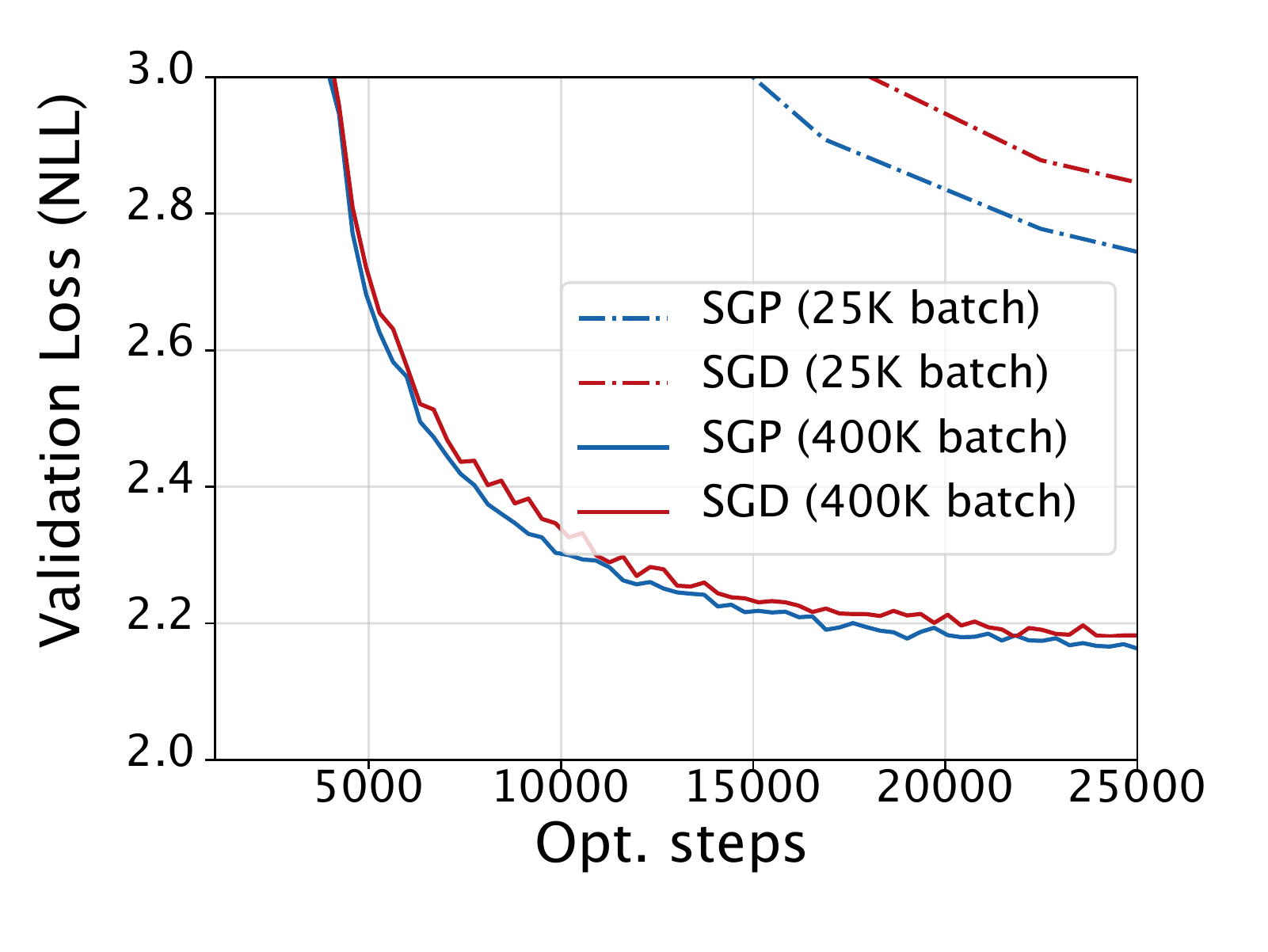}}
	\subfloat[Time-wise convergence]{\includegraphics[width=.5\columnwidth]{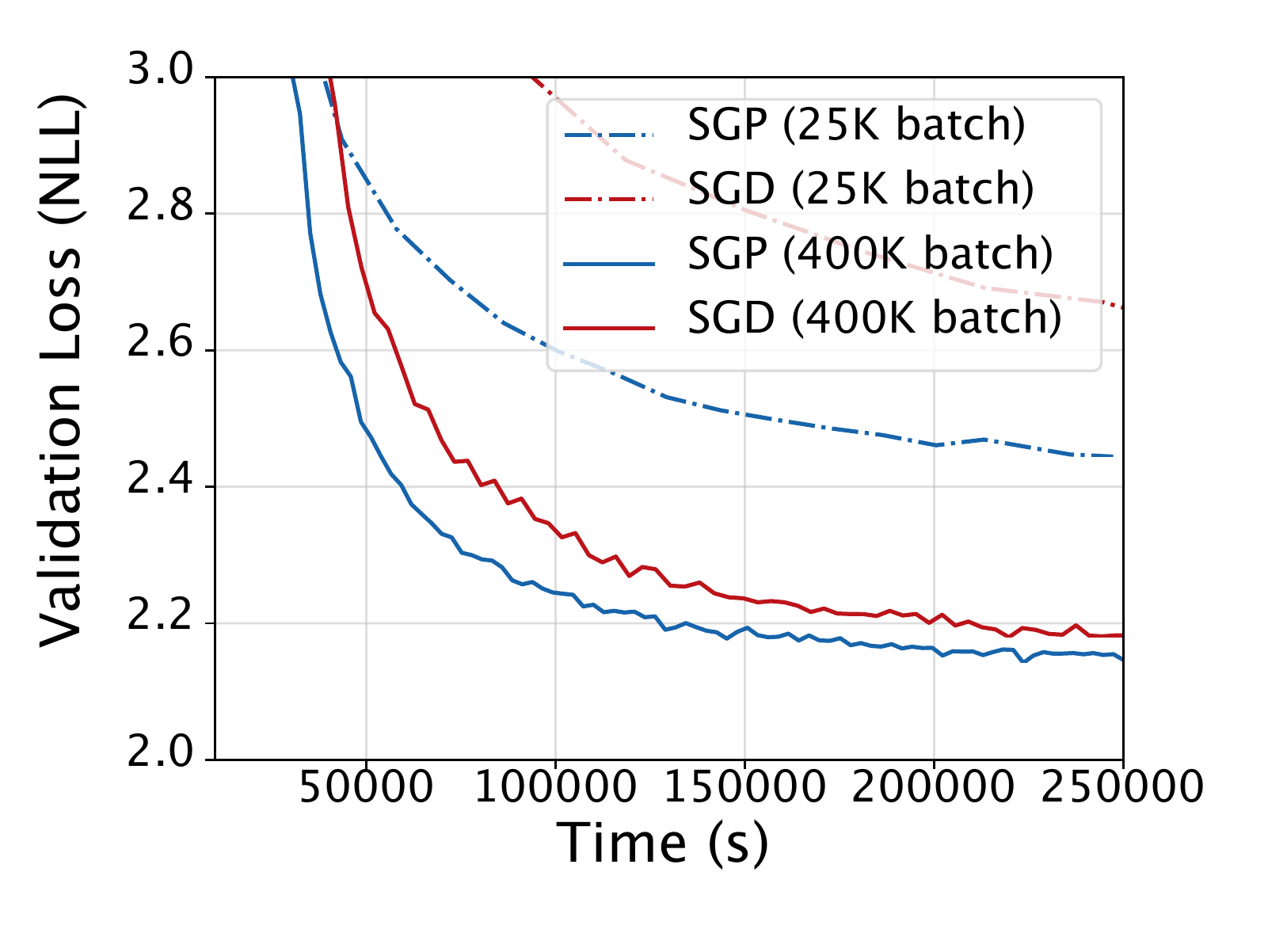}}
	\caption{Neural Machine Translations experiments run over 8$\times$V100 GPUs located on 8 different machines, interconnected via 10 Gbps Ethernet. Adam-SGP and \AllReduce Adam-SGD iteration- and time-wise convergence in both small- and large-batch settings. SGP makes slightly more progress per iteration in both small- and large-batch settings, and runs approximately 1.5$\times$ faster in the large-batch setting and 2$\times$ faster in the small batch setting.}
	\label{fig:nmt_convergence}
\end{figure}

Figures~\ref{fig:nmt_convergence}~(a) and~(b) show the iteration-wise and time-wise validation curves (respectively) of (Adam) SGP and \AllReduce (Adam) SGD using small- and large-batch training. We find that SGP makes slightly more progress per iteration in both small- and large-batch settings, and runs approximately 1.5$\times$ faster than \AllReduce SGD in the large-batch setting, and 2$\times$ faster in the small-batch setting.
We also evaluate the test-set BLEU scores for models trained using 400K tokens batch. We investigate the performance obtained after the same number of iterations ($\sim$14K), and after training for the same amount of time (3 days). We evaluate BLEU scores using a beam search of 4, and a length penalty of 0.6, following~\cite{vaswani2017attention}. SGP achieves a superior BLEU score to \AllReduce SGD, both after a fix number of iterations (26.4 for SGP vs 25.8 for AR-SGD), and for fix runtime budget (27.5 for SGP vs 26.9 for AR-SGD).


\section{Conclusion}
We propose SGP and OSGP for accelerating distributed training of DNNs.
We provide theoretical convergence guarantees in the smooth non-convex setting, matching known convergence rates for parallel SGD.
We also empirically study the methods over several computing infrastructures, and provide assessments on image classification (ImageNet, ResNet-50) and neural machine translation (Transformer, WMT16 EN-DE) tasks.
We find that SGP and OSGP can run significantly faster than parallel SGD in communication-bound settings, and can train better models in less time.


\subsubsection*{Acknowledgments}
We thank Shubho Sengupta and Teng Li for useful discussions and for maintaining the computing infrastructure used to conduct these experiments.

\bibliography{icml2019_conference.bib}
\bibliographystyle{icml2019}

%
\appendix
\
\counterwithin{figure}{section}

\onecolumn
\icmltitle{Supplementary Material \\ Stochastic Gradient Push for Distributed Deep Learning}



\section{Communication Topology}
\label{sec:implementation_details}

\textbf{Directed exponential graph.} For the SGP experiments we use a time-varying directed graph to represent the inter-node connectivity. Thinking of the nodes as being ordered sequentially, according to their rank, $0, \dots, n -1$,\footnote{We use indices $0,\dots,n-1$ rather than $1,\dots,n$ only in this section, to simplify the discussion.} each node periodically communicates with peers that are $2^0, 2^1, \ldots, 2^{\lfloor \log_2(n - 1) \rfloor}$ hops away. Fig.~\ref{fig:graph_topology} shows an example of a directed 8-node exponential graph. Node $0$'s $2^0$-hop neighbour is node $1$; node $0$'s $2^1$-hop neighbour is node $2$; and node $0$'s $2^2$-hop neighbour is node $4$.
\begin{figure*}[ht]
	\centering
	\subfloat[Directed Exponential Graph highlighting node $0$'s out-neighbours]{\includegraphics[width=0.35\textwidth]{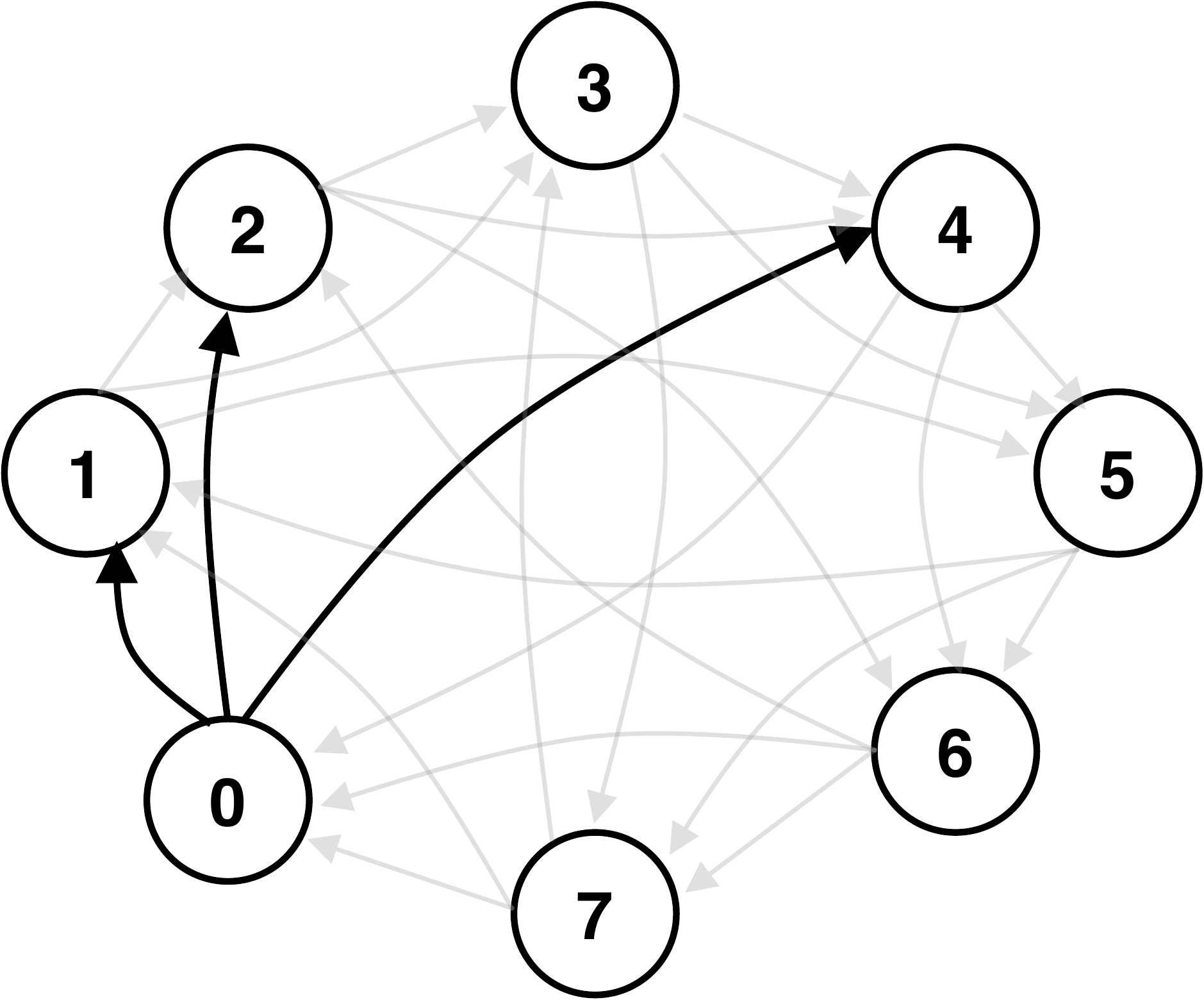}}
	\caption{Example of an 8-node exponential graph used in experiments}
	\label{fig:graph_topology}
\end{figure*}

In the one-peer-per-node experiments, each node cycles through these peers, transmitting, only, to a single peer from this list at each iteration. E.g., at iteration $k$, all nodes transmit messages to their $2^0$-hop neighbours, at iteration $k + 1$ all nodes transmit messages to their $2^1$-hop neighbours, an so on, eventually returning to the beginning of the list before cycling through the peers again. This procedure ensures that each node only sends and receives a single message at each iteration. By using full-duplex communication, sending and receiving can happen in parallel.

In the two-peer-per-node experiments, each node cycles through the same set of peers, transmitting to two peers from the list at each iteration. E.g., at iteration $k$, all nodes transmit messages to their $2^0$-hop and $2^1$-hop neighbours, at iteration $k + 1$ all nodes transmit messages to their $2^1$-hop and $2^2$ neighbours, an so on, eventually returning to the beginning of the list before cycling through the peers again.
Similarly, at each iteration, each node also receives, in a full-duplex manner, two messages from some peers that are unknown to the receiving node ahead of time. Thereby performing the send and receive operations in parallel.

\textbf{Definition of $\bm{P}^{(k)}$.} We choose the mixing matrices such that they are column stochastic (all columns sum to $1$), and conform to the graph topology described above.
Recall that each node $i$ can choose its mixing weights ($i^{th}$ column of $\bm{P}^{(k)}$), independently of the other nodes in the network.
To minimize the number of floating point operations in each iteration, we choose to use uniform mixing weights, meaning that nodes assign uniform message weights to all neighbours.
In the one-peer-per-node experiments, each node sends a message to one neighbor, and ``sends a message'' to itself at every iteration, and so each column of $\bm{P}^{(k)}$ has exactly two non-zero entries, both of which are equal to $1/2$.
The first set of non-zero entries corresponds to the diagonals.
At all time steps $k$, the diagonal entries satisfy $p_{i,i}^{(k)} = 1/2$ for all $i$.
The second set of non-zero entries correspond to the neighbor indices.
At time step $k$, each node sends to a neighbor that is $h_k \defeq 2^{k \text{ mod } \lfloor \log_2(n-1)\rfloor}$ hops away.
That is, at each time step $k$, each node $i$ sends a message to node $(i + h_k) \text{ mod } n$.
Thus, we get that
\[
p_{j,i}^{(k)} = \begin{cases} 1/2, & \text{ if } j = (i + h_k) \text{ mod } n\\ 0, & \text{ otherwise.} \end{cases}
\]
Note that, with this design, in fact each node sends to one peer and receives from one peer at every iteration, so the communication load is balanced across the network.

In the two-peer-per-node experiments, the definition is similar, but now there will be three non-zero entries in each column of $\bm{P}^{(k)}$, all of which will be equal to $1/3$; these are the diagonal, and the entries corresponding to the two neighbors to which the node sends at that iteration. In addition, each node will send two messages and receive two messages at every iteration, so the communication load is again balanced across the network.


\paragraph{Undirected exponential graph.} For the D-PSGD experiments we use a time-varying undirected bipartite exponential graph to represent the inter-node connectivity. Odd-numbered nodes send messages to peers that are $2^1 - 1, 2^2 - 1, \ldots, 2^{\lfloor \log_2(n-1) \rfloor} - 1$ (even-numbered nodes), and wait to a receive a message back in return. Each odd-numbered node cycles through the peers in the list in a similar fashion to the one-peer-per-node SGP experiments. Even-numbered nodes wait to receive a message from some peer (unknown to the receiving node ahead of time), and send a message back in return.

We adopt these graphs to be consistent with the experimental setup used in \citet{lian2017can} and \citet{Lian2018asynchronous}.

Note also that these graphs are all regular, in that all nodes have the same number of in-coming and out-going connections.

\textbf{Decentralized averaging errors.} To further motivate our choice of using the directed exponential graph with SGP, let us forget about optimization for a moment and focus on the problem of distributed averaging, described in Section~2, using the \PushSum algorithm. Recall that each node $i$ starts with a vector $\bm{y}_i^{(0)}$, and the goal of the agents is to compute the average $\ybar = \frac{1}{n} \sum_i \bm{y}_i^{(0)}$. Then, since $\bm{y}_i^{(k+1)} = \sum_{j=1}^n p_{i,j}^{(k)} \bm{y}_j^{(k)}$, after $k$ steps we have
\[
\bm{Y}^{(k)} = \bm{P}^{(k-1)} \bm{P}^{(k-2)} \cdots \bm{P}^{(1)} \bm{P}^{(0)} \bm{Y}^{(0)},
\]
where $\bm{Y}^{(k)}$ is a $n \times d$ matrix with $\bm{y}_i^{(k)}$ as its $i$th row.

Let $\bm{P}^{(k-1:0)} = \bm{P}^{(k-1)} \bm{P}^{(k-2)} \cdots \bm{P}^{(1)} \bm{P}^{(0)}$. The worst-case rate of convergence can be related to the second-largest singular value of $\bm{P}^{(k-1:0)}$ \cite{Nedic2018network}. In particular, after $k$ iterations we have
\[
\sum_{i} \| \bm{y}_i^{(k)} - \ybar \|_2^2 \le \lambda_2(\bm{P}^{(k-1:0)}) \sum_{i} \|\bm{y}_i^{(0)} - \ybar \|_2^2,
\]
where $\lambda_2(\bm{P}^{(k-1:0)})$ denotes the second largest singular value of $\bm{P}^{(k-1:0)}$.

For the scheme proposed above, cycling deterministically through neighbors in the directed exponential graph, one can verify that after $k = \lfloor \log_2(n-1) \rfloor$ iterations, we have $\lambda_2(\bm{P}^{(k-1:0)}) = 0$, so all nodes exactly have the average. Intuitively, this happens because the directed exponential graph has excellent mixing properties: from any starting node in the network, one can get to any other node in at most $\log_2(n)$ hops. For $n=32$ nodes, after $5$ iterations averaging has converged using this strategy. In comparison, if one were to cycle through edges of the complete graph (where every node is connected to every other node), then for $n=32$, after 5 consecutive iterations one would have still have $\lambda_2(\bm{P}^{(k-1:0)}) \approx 0.6$; i.e., nodes could be much further from the average (and hence, much less well-synchronized). 

Similarly, one could consider designing the matrices $\bm{P}^{(k)}$ in a stochastic manner, where each node randomly samples one neighbor to send to at every iteration. If each node samples a destination uniformly from its set of neighbors in the directed exponential graph, then $\mathbb{E} \lambda_2(\bm{P}^{(k-1:0)}) \approx 0.4$, and if each node randomly selected a destination uniformly among all other nodes in the network (i.e., randomly from neighbors in the complete graph), then $\mathbb{E} \lambda_2(\bm{P}^{(k-1:0)}) \approx 0.2$. Thus, random schemes are still not as effective at quickly averaging as deterministically cycling through neighbors in the directed exponential graph. Moreover, with randomized schemes, we are no longer guaranteed that each node receives the same number of messages at every iteration, so the communication load will not be balanced as in the deterministic scheme.

The above discussion focused only on approximate distributed averaging, which is a key step within decentralized optimization. When averaging occurs less quickly, this also impacts optimization. Specifically, since nodes are less well-synchronized (i.e., further from a consensus), each node will be evaluating its local mini-batch gradient at a different point in parameter space. Averaging these points (rather than updates based on mini-batch gradients evaluated at the same point) can be seen as injecting additional noise into the optimization process, and in our experience this can lead to worse performance in terms of train error.

\section{Overlap SGP}
\label{appendix:osgp}

Although SGP does not use network-wide collective communication primitives like \AllReduce, the implementation of Alg.~\ref{SGPalg} requires using blocking sends and receives; \ie, nodes do not proceed to until they have received messages from all neighbors at that iteration. In this section we present the pseudocode of Overlap-SGP (OSGP) in Alg.~\ref{alg:osgp} that overlaps gradient computation with communication to hide the communication cost.
\begin{algorithm}[!t]
    \small
	\caption{Overlap Stochastic Gradient Push (SGP) \label{alg:osgp}}
  	\begin{algorithmic}[1]
      	\REQUIRE{Initialize $\tau \geq 0$, count\_since\_last $=0$, $\gamma > 0$, $\itr{\bm{x}}{0}_i = \itr{\bm{z}}{0}_i \in \R^d$ and $\itr{w}{0}_i=1$ for all nodes $i \in \{1, 2, \ldots, n \}$}\\
    	\FOR{$k=0,1,2,\cdots, K$, at node $i$,}
        	\STATE {Sample new mini-batch $\itr{\xi}{k}_i \sim {\cD_i}$ from local distribution}
            \STATE {Compute mini-batch gradient at $\itr{\bm{z}}{k}_i$: $\nabla \bm{F}_i(\itr{\bm{z}}{k}_i; \itr{\xi}{k}_i)$}
            \STATE {$\itr{\bm{x}}{k + \frac{1}{2}}_i = \itr{\bm{x}}{k}_i -\gamma \nabla \bm{F}_i(\itr{\bm{z}}{k}_i; \itr{\xi}{k}_i)$}
            \IF{$k$ mod $\tau = 0$}
                \STATE {Non-blocking send $\big(p_{j,i}^{(k)} \bm{x}_i^{(k+\frac{1}{2})}, p_{j,i}^{(k)} w_i^{(k)}\big)$ to out-neighbors}
                \STATE {$\itr{\bm{x}}{k + 1}_i = p_{i,i}\itr{\bm{x}}{k + 1/2}_i$}
                \STATE {$\itr{w}{k + 1}_i = p_{i,i}\itr{w}{k}_i$}
            \ELSE
                \STATE {$\itr{\bm{x}}{k + 1}_i = \itr{\bm{x}}{k + 1/2}_i$}
                \STATE {$\itr{w}{k + 1}_i = \itr{w}{k}_i$}
            \ENDIF
            \IF{count\_since\_last $= \tau$}
                \STATE {Block until $\big(p_{i,j}^{(k-\tau)} \bm{x}_j^{(k-\tau + \frac{1}{2})}, p_{i,j}^{(k-\tau)} w_j^{(k-\tau)}\big)$ is received for all in-neighbors $j$}
                \STATE {count\_since\_last $\gets 0$}
            \ELSE
                \STATE {count\_since\_last $\gets$ count\_since\_last $+1$}
            \ENDIF
            \IF{Receive buffer non-empty}
                \FOR{$\big(p_{i,j}^{(k^\prime)} \bm{x}_j^{(k^\prime + \frac{1}{2})}, p_{i,j}^{(k^\prime)} w_j^{(k^\prime)}\big)$ in the receive buffer}
                  \STATE {$\itr{\bm{x}}{k + 1}_i \gets \itr{\bm{x}}{k + 1}_i + \itr{p}{k^\prime}_{i,j} \itr{\bm{x}}{k^\prime + \frac{1}{2}}_j$}
                  \STATE {$\itr{w}{k + 1}_i \gets \itr{p}{k^\prime}_{i,j} \itr{w}{k^\prime}_j$}
                \ENDFOR
            \ENDIF
            \STATE {$\itr{\bm{z}}{k + 1}_i = \itr{\bm{x}}{k + 1}_i / \itr{w}{k + 1}_i$}
 		\ENDFOR
	\end{algorithmic}
\end{algorithm}
\footnotetext{We define $(k$ mod $0) \defeq 0$.}
In line 25 in Algorithm~\ref{alg:osgp}, nodes compute the de-biased estimate of their model parameters.
In lines 19 to 24, nodes aggregate all messages received in that iteration.
Lines 13 to 18 ensure that the message delays are bounded, and that the nodes remain synchronized.
In particular, Algorithm~\ref{alg:osgp} is \emph{synchronous} because of lines 13 to 18.
If a node hasn't received a message from its in-neighbours in $\tau$ iterations, it will block and wait to received said messages.
Note that if $\tau=0$, vanilla SGP, then nodes block and wait to receive all incoming messages in each iteration.
In lines 5 to 6, nodes send messages to their neighbours every $\tau$ iterations.
Once again, note that if $\tau=0$, vanilla SGP, then nodes send messages to their neighbours every iteration.
In lines 2 to 4 the nodes take a stochastic gradient step.
If $\tau=1$ ($1$-overlap SGP), nodes transmit messages to their neighbours in every iteration, but don't wait to receive messages until the subsequent iteration.

We provide a lot of detail in Algorithm~\ref{alg:osgp} to make it easier to implement the method; however, in essence, $\tau$-overlap SGP is simply vanilla SGP with delayed communication. \ie, where nodes only send a message to their neighbours every $\tau$ iterations, and can receive messages at any time in-between communication intervals.

\section{Implementation Details}
\label{appendix:sgpimpl}
In all of our experiments, we minimize the number of floating-point operations performed in each iteration, $k$, by using the mixing weights
\[
	\itr{p}{k}_{j,i} = 1/ \abs{\itr{\Nout{i}}{k}}
\]
for all $i, j = 1, 2, \ldots, n$. In words, each node assigns mixing weights uniformly to all of its out-neighbors in each iteration. Recalling our convention that each node is an in- and out-neighbor of itself, it is easy to see that this choice of mixing-weight satisfies the column-stochasticity property. It may very well be that there is a different choice of mixing-weights that lead to better spectral properties of the gossip algorithm; however we leave this exploration for future work. We denote node $i$'s uniform mixing weights at iteration $k$ by $\itr{p}{k}_i$ --- dropping the other subscript, which identifies the receiving node.

To leverage the highly efficient NVLink interconnect within each server, we treat each machine as one node in all of our experiments. In our implementation of SGP, each node computes a local mini-batch in parallel using all 8 GPUs via a local \AllReduce, which is efficiently implemented via the NVIDIA Collective Communications Library. Then inter-node averaging is accomplished using \PushSum either over Ethernet or InfiniBand. In the InfiniBand experiments, we leverage GPUDirect to directly send/receive messages between GPUs on different nodes and avoid transferring the model back to host memory. In the Ethernet experiments this is not possible, so the model is transferred to host memory after the local \AllReduce, and then \PushSum messages are sent over Ethernet.

To maximize the utility of the resources available on each server, each node (occupying a single server exclusively) runs two threads, a gossip thread and a computation thread. The computation thread executes the main logic used to train the local model on the GPUs available to the node, while the communication thread is used for inter-node network I/O. In particular, the communication thread is used to gossip messages between nodes. When using Ethernet-based communication, the nodes communicate their parameter tensors over CPUs. When using InifiniBand-based communication, the nodes communicate their parameter tensors using GPUDirect RDMA, thereby avoiding superfluous device to pinned-memory transfers of the model parameters.

Each node initializes its model on one of its GPUs, and initializes its scalar push-sum weight to $1$. At the start of training, each node also allocates a \emph{send}- and a \emph{receive}- communication-buffer in pinned memory on the CPU (or equivalently on a GPU in the case of GPUDirect RDMA communication).

In each iteration, the communication thread waits for the send-buffer to be filled by the computation thread; transmits the message in the send-buffer to its out-neighbours; and then aggregates any newly-received messages into the receive-buffer.

In each iteration, the computation thread blocks to retrieve the aggregated messages (in the non-overlap case) in the receive-buffer; directly adds the received parameters to its own model parameters; and directly adds the received push-sum weights to its own push-sum weight. The computation thread then converts the model parameters to the \emph{de-biased} estimate by dividing by the push-sum weight; executes a forward-backward pass of the \emph{de-biased model} in order to compute a stochastic mini-batch gradient; converts the model parameters back to the \emph{biased estimate} by multiplying by the push-sum weight; and applies the newly-computed stochastic gradients to the biased model. The updated model parameters are then multiplied by the mixing weight, $\itr{p}{k}_i$, and asynchronously copied back into the send-buffer for use by the communication thread. The push-sum weight is also multiplied by the same mixing weight and concatenated into the send-buffer.

In short, gossip is performed on the biased model parameters (push-sum numerators); stochastic gradients are computed using the de-biased model parameters; stochastic gradients are applied back to the biased model parameters; and then the biased-model and the push-sum weight are multiplied by the same uniform mixing-weight and copied back into the send-buffer.
\begin{algorithm}[t]
	\caption{Stochastic Gradient Push with Momentum \label{alg:sgp_momentum}}
    \label{SGPalgmom}
    \small
 	\begin{algorithmic}[1]
     	\REQUIRE{Initialize $\gamma > 0$, $m \in (0, 1)$, $\itr{\bm{x}}{0}_i = \itr{\bm{z}}{0}_i \in \R^d$ and $\itr{w}{0}_i=1$ for all nodes $i \in [n]$}\\
    	\FOR{$k=0,1,2,\cdots, K$, at node $i$,}
        	\STATE {Sample new mini-batch $\itr{\xi}{k}_i \sim {\cD_i}$ from local distribution}
            \STATE {Compute mini-batch gradient at $\itr{\bm{z}}{k}_i$: $\nabla \bm{F}_i(\itr{\bm{z}}{k}_i; \itr{\xi}{k}_i)$}
            \STATE {$\itr{\bm{u}}{k + 1}_i = m \itr{\bm{u}}{k}_i + \nabla \bm{F}_i(\itr{\bm{z}}{k}_i; \itr{\xi}{k}_i)$}
            \STATE {$\itr{\bm{x}}{k + \frac{1}{2}}_i = \itr{\bm{x}}{k}_i -\gamma ( m \itr{\bm{u}}{k + 1}_i + \nabla \bm{F}_i(\itr{\bm{z}}{k}_i; \itr{\xi}{k}_i))$}
            \STATE {Send $\big(p_{j,i}^{(k)} \bm{x}_i^{(k+\frac{1}{2})}, p_{j,i}^{(k)} w_i^{(k)}\big)$ to out-neighbors;\par receive $\big(p_{i,j}^{(k)} \bm{x}_j^{(k + \frac{1}{2})}, p_{i,j}^{(k)} w_j^{(k)}\big)$ from in-neighbors}
            \STATE {$\itr{\bm{x}}{k + 1}_i = \sum_{j \in \itr{\Nin{i}}{k}} \itr{p}{k}_{i,j} \itr{\bm{x}}{k + \frac{1}{2}}_j$}
            \STATE {$\itr{w}{k + 1}_i = \sum_{j \in \itr{\Nin{i}}{k}} \itr{p}{k}_{i,j} \itr{w}{k}_j$}
            \STATE {$\itr{\bm{z}}{k + 1}_i = \itr{\bm{x}}{k + 1}_i / \itr{w}{k + 1}_i$}
		\ENDFOR
	\end{algorithmic}
\end{algorithm}

\subsection{Hyperparameters}
\label{appendix:hyperparams}

For the ImageNet experiments, we follow the experimental protocol described in~\citep{goyal2017accurate}.
When we ``apply the stochastic gradients'' to the biased model parameters, we actually carry out an SGD step with nesterov momentum (see Alg.~\ref{SGPalgmom}). For the $32,64$, and $128$ GPU experiments we use the same exact learning-rate, schedule, momentum, and weight decay as those suggested in~\citep{goyal2017accurate} for SGD. In particular, we use a reference learning-rate of $0.1$ with respect to a $256$ sample batch, and scale this linearly with the batch-size; we decay the learning-rate by a factor of $10$ at epochs $30,60,80$; we use a Nesterov momentum parameter of $0.9$, and we use weight decay $0.0001$.

For the machine translation experiment, we follow~\cite{vaswani2017attention} and combine Stochastic Gradient Push with the Adam preconditioner.
In particular, we make use of the \textsc{fairseq} code~\cite{gehring2017convs2s}, and train the transformer networks via SGP by replacing the built-in PyTorch parallel SGD model wrapper with our SGP model wrapper.


\section{Additional Experiments}

\subsection{Additional Training Curves}
\label{appendix:ethernet}
\begin{figure}[!ht]
	\centering
	\subfloat[Train]{\includegraphics[width=0.3\textwidth]{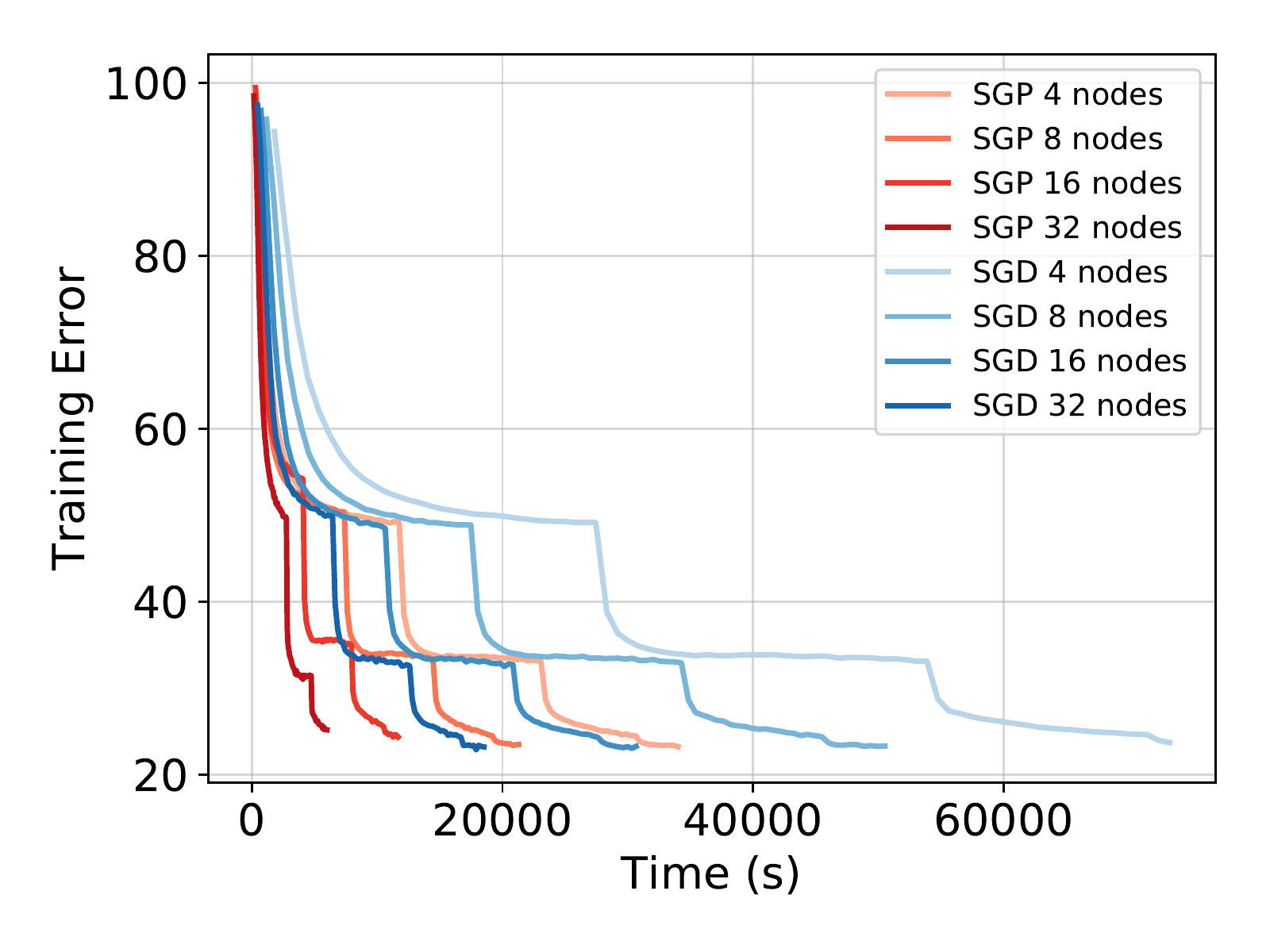}}
	\subfloat[Validation]{\includegraphics[width=0.3\textwidth]{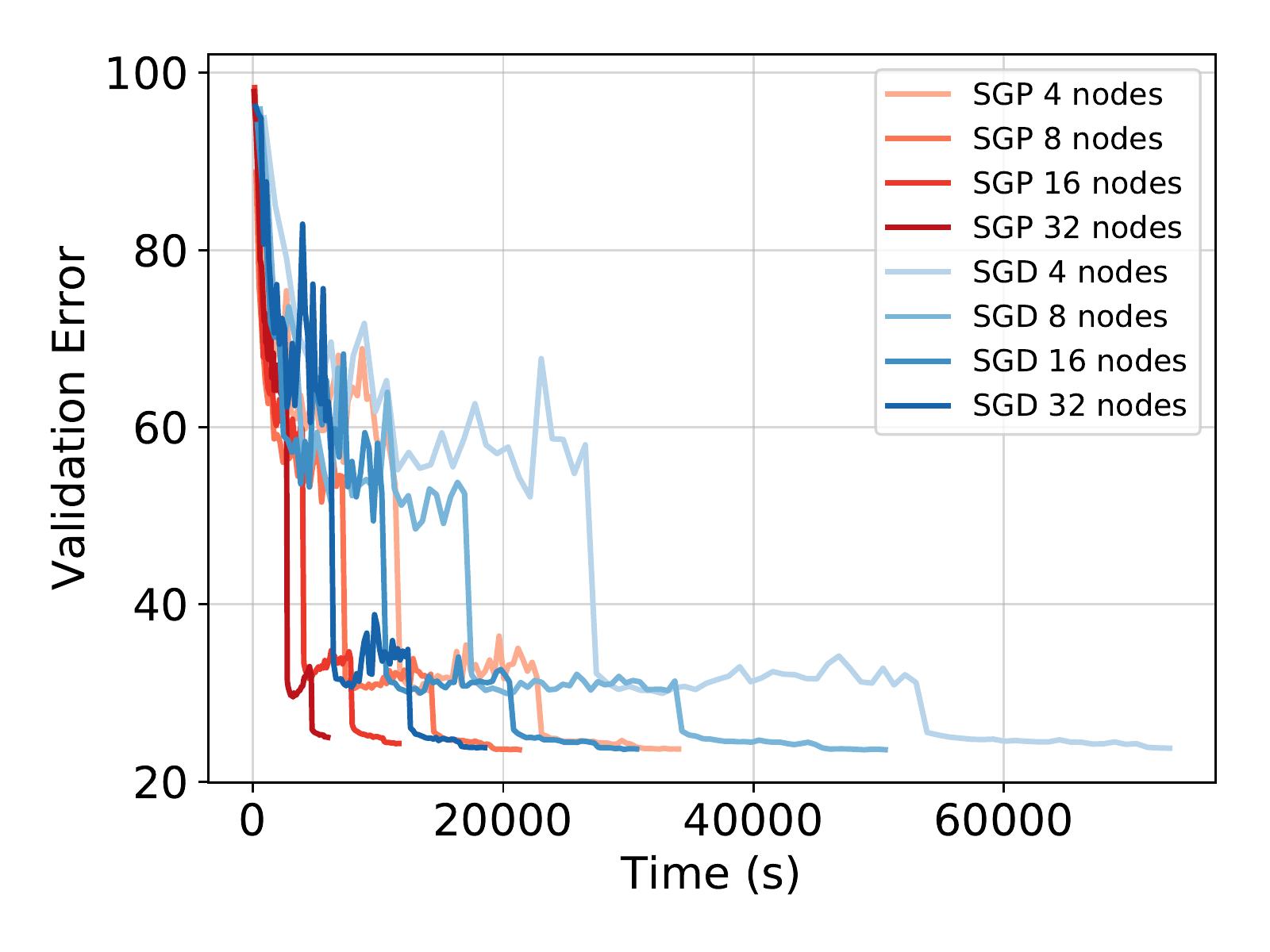}}
	\caption{Training on Ethernet 10Gbit/s}
	\label{fig:eth_train}
\end{figure}
\begin{figure}[!ht]
	\centering
	\subfloat[Train]{\includegraphics[width=0.3\textwidth]{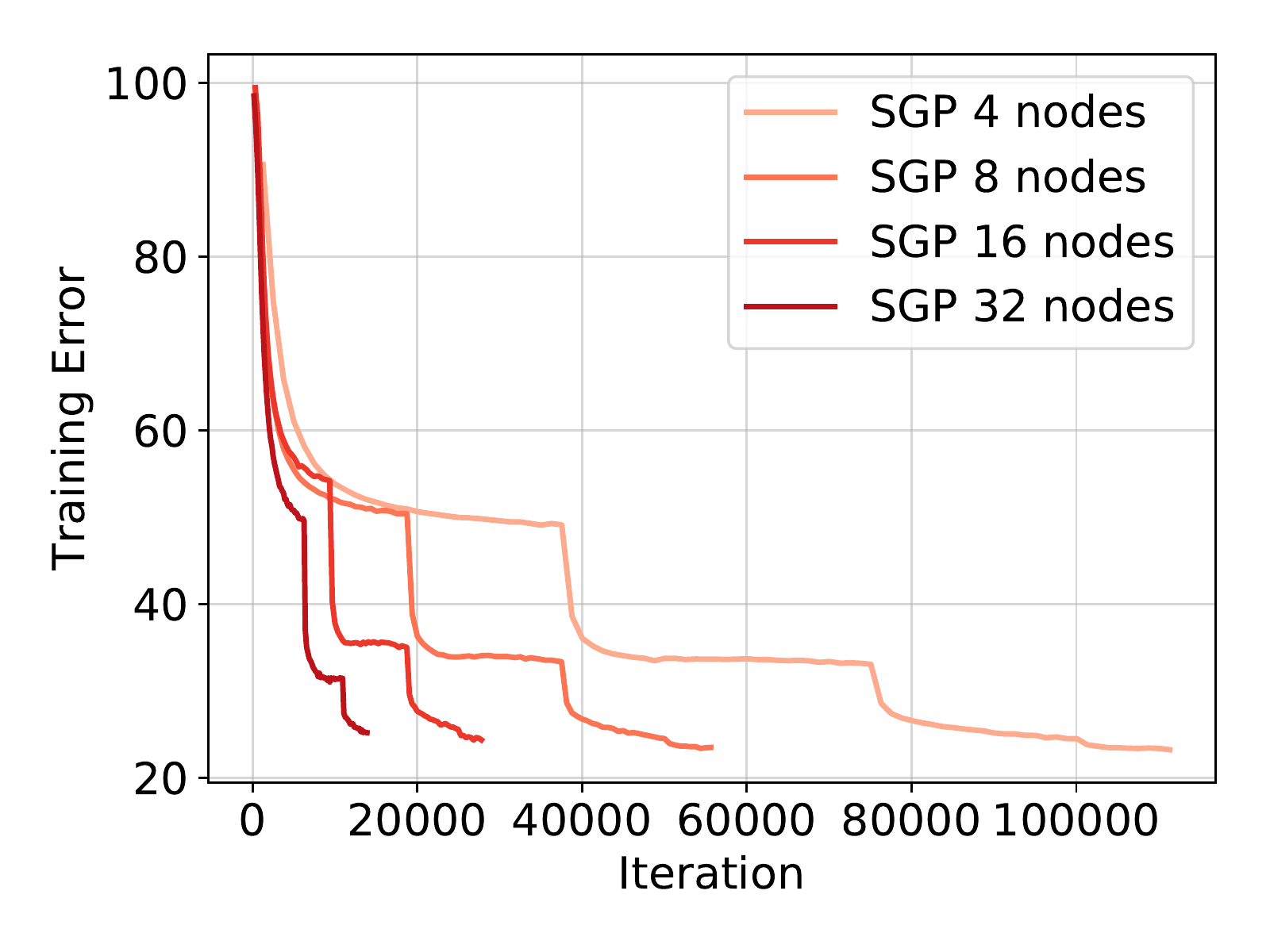}}
	\subfloat[Validation]{\includegraphics[width=0.3\textwidth]{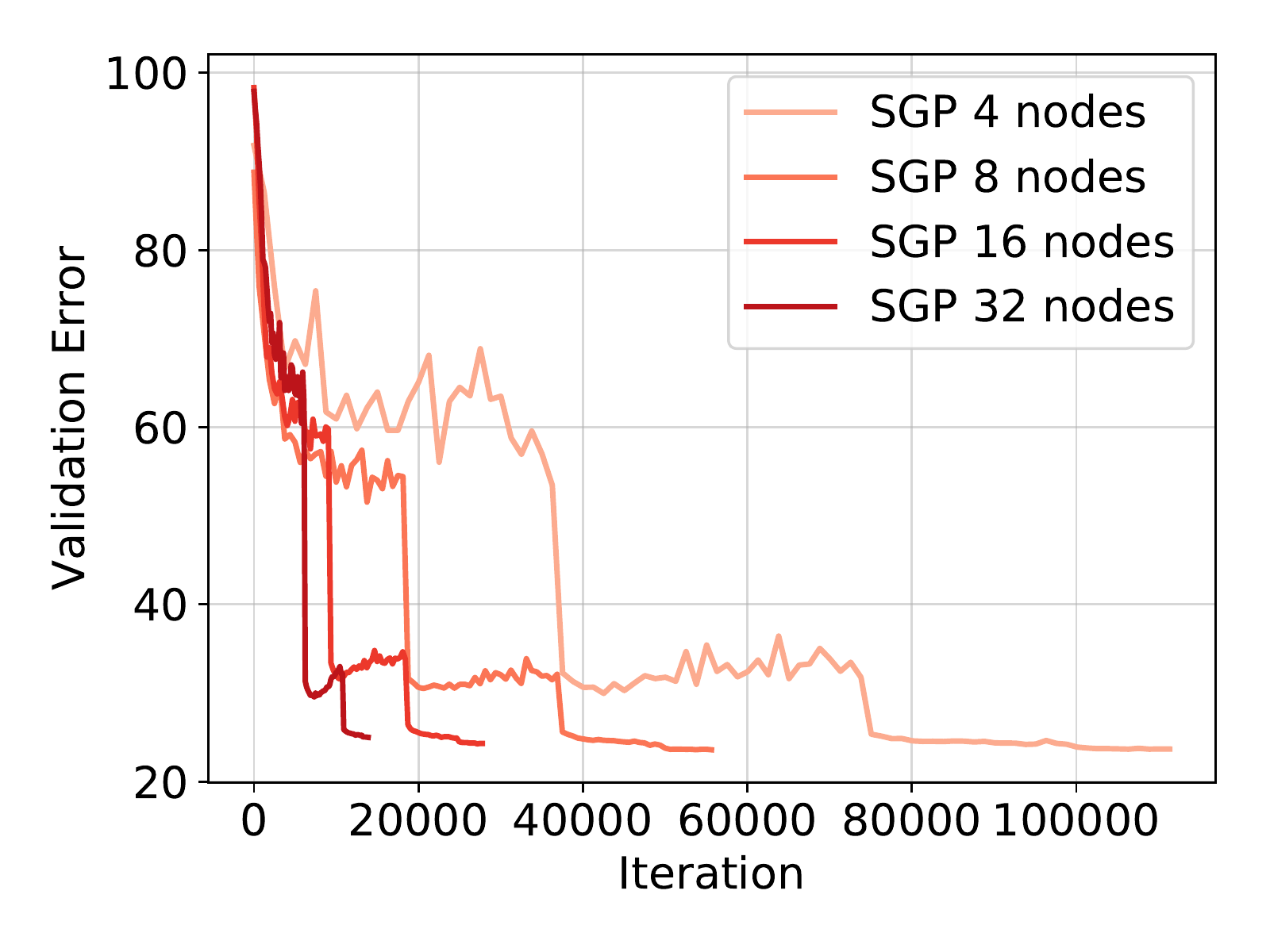}}
	\caption{Training/Validation accuracy per iteration for SGP (Ethernet 10Gbit/s). Each time we double the number of node in the network, we half the total number of iterations.}
	\label{fig:eth_iterations}
\end{figure}
The curves in Figure~\ref{fig:eth_train} show the time-wise train- and validation-accuracies for the different runs performed on Ethernet 10Gbit/s.
Figure~\ref{fig:eth_iterations} reports the iteration-wise training and validation accuracy of SGP when using 10Gbits/s Ethernet. 

\subsection{Discrepancy across different nodes}
\label{subsec:discrepancy}
\begin{figure}[!ht]
 \centering
 \subfloat[Discrepancy on 4 nodes]{\includegraphics[width=0.3\textwidth]{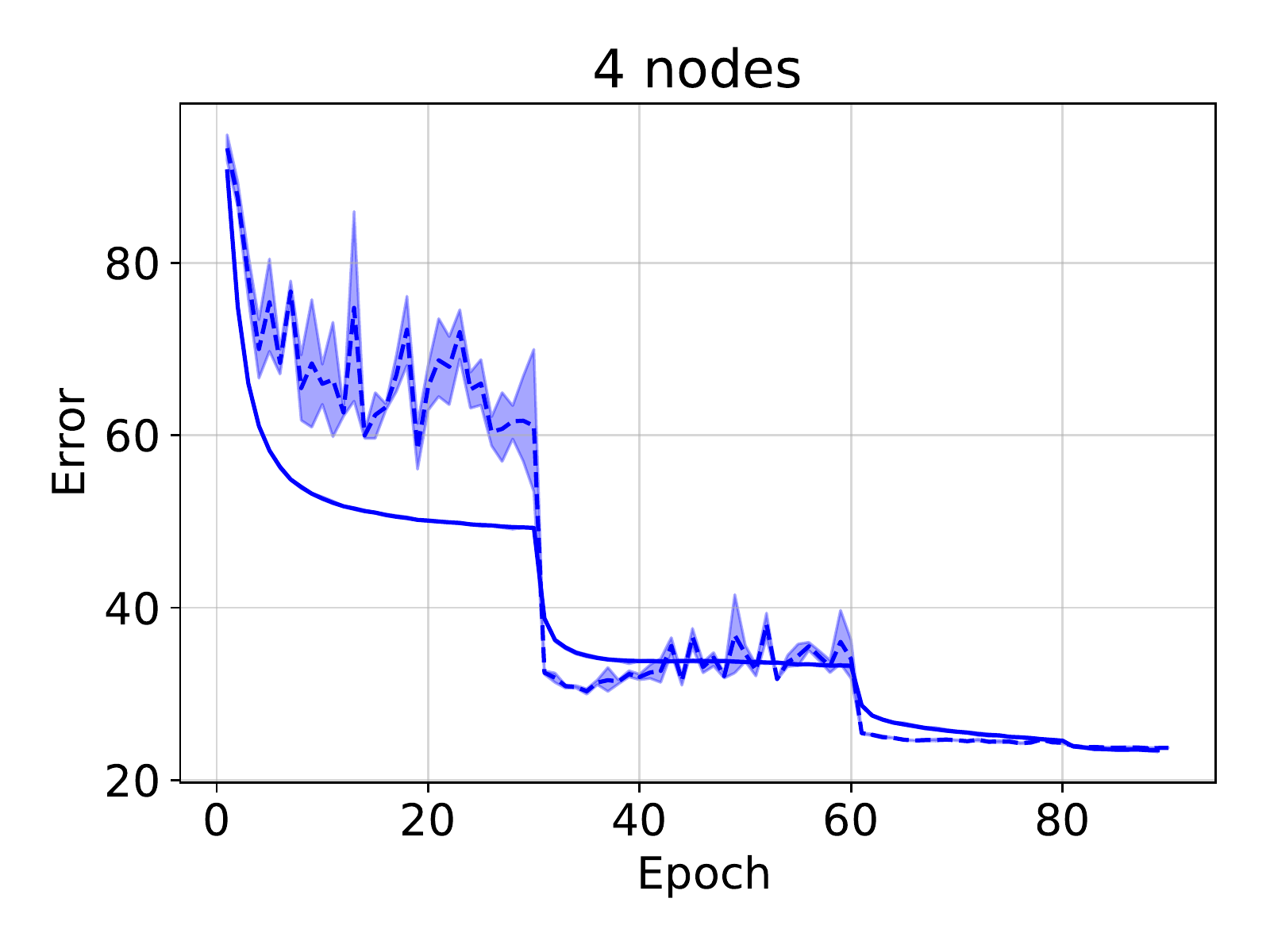}}
 \subfloat[Discrepancy on 32 nodes]{\includegraphics[width=0.3\textwidth]{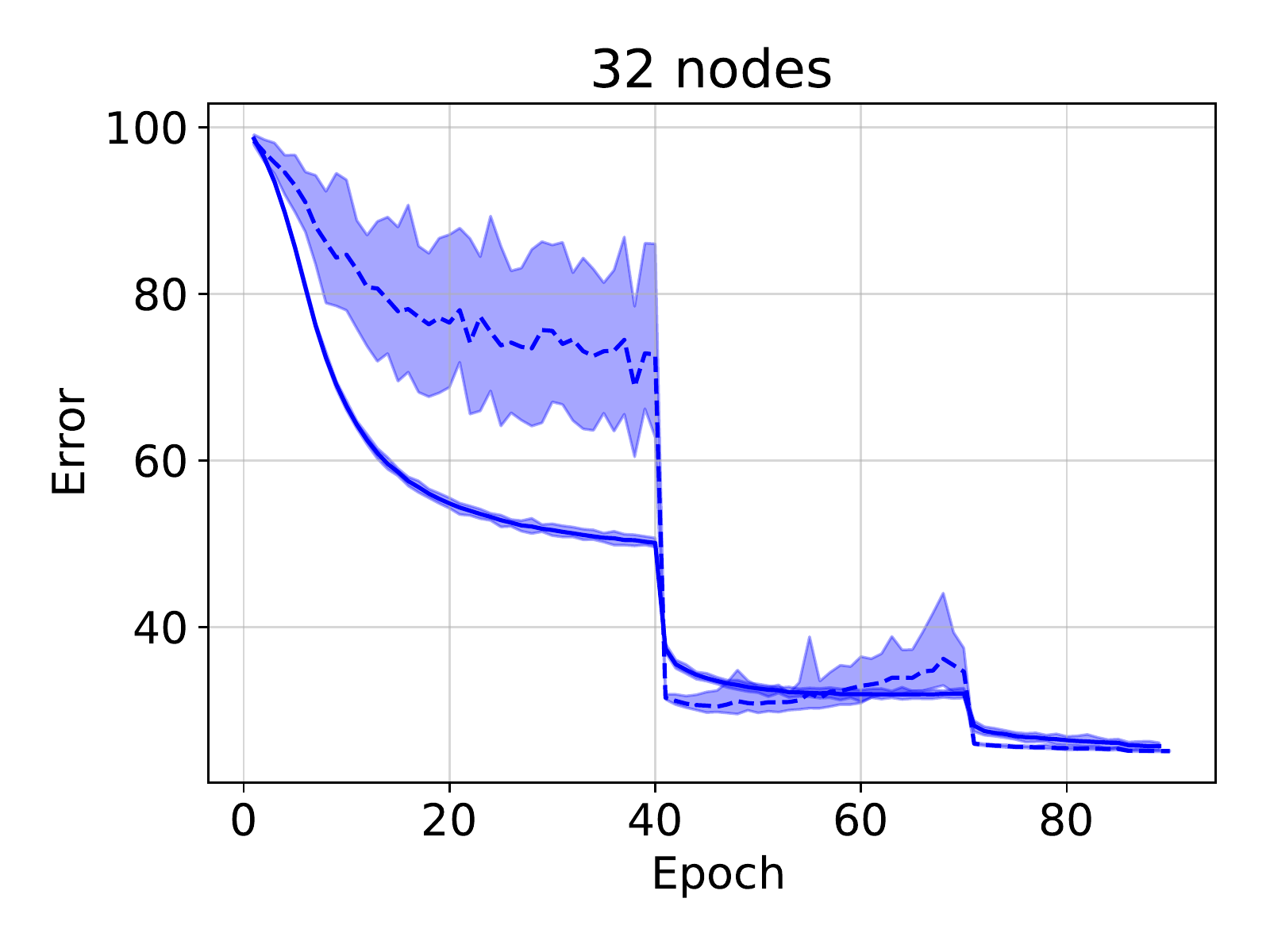}}
 \caption{Resnet50, trained with SGP, training and validation errors for 4 and 32 nodes experiments. The solid and dashed lines in each figure show the mean training and validation error, respectively, over all nodes. The shaded region shows the maximum and minimum error attained at different nodes in the same experiment. Although there is non-trivial variability across nodes early in training, all nodes eventually converge to similar validation errors, achieving consensus in the sense that they represent the same function.}
 \label{fig:eth_range}
\end{figure}
Here, we investigate the performance variability across nodes during training for SGP. In figure~\ref{fig:eth_range}, we report the minimum, maximum and mean error across the different nodes for training and validation. In an initial training phase, we observe that nodes have different validation errors; their local copies of the Resnet-50 model diverge. As we decrease the learning, the variability between the different nodes diminish and the nodes eventually converging to similar errors. This suggests that all models ultimately represent the same function, achieving consensus.

\subsection{SGP Scaling Analysis}
\begin{figure}[!ht]
	\centering
	\subfloat[ETH 10Gbit/s]{\includegraphics[width=0.3\textwidth]{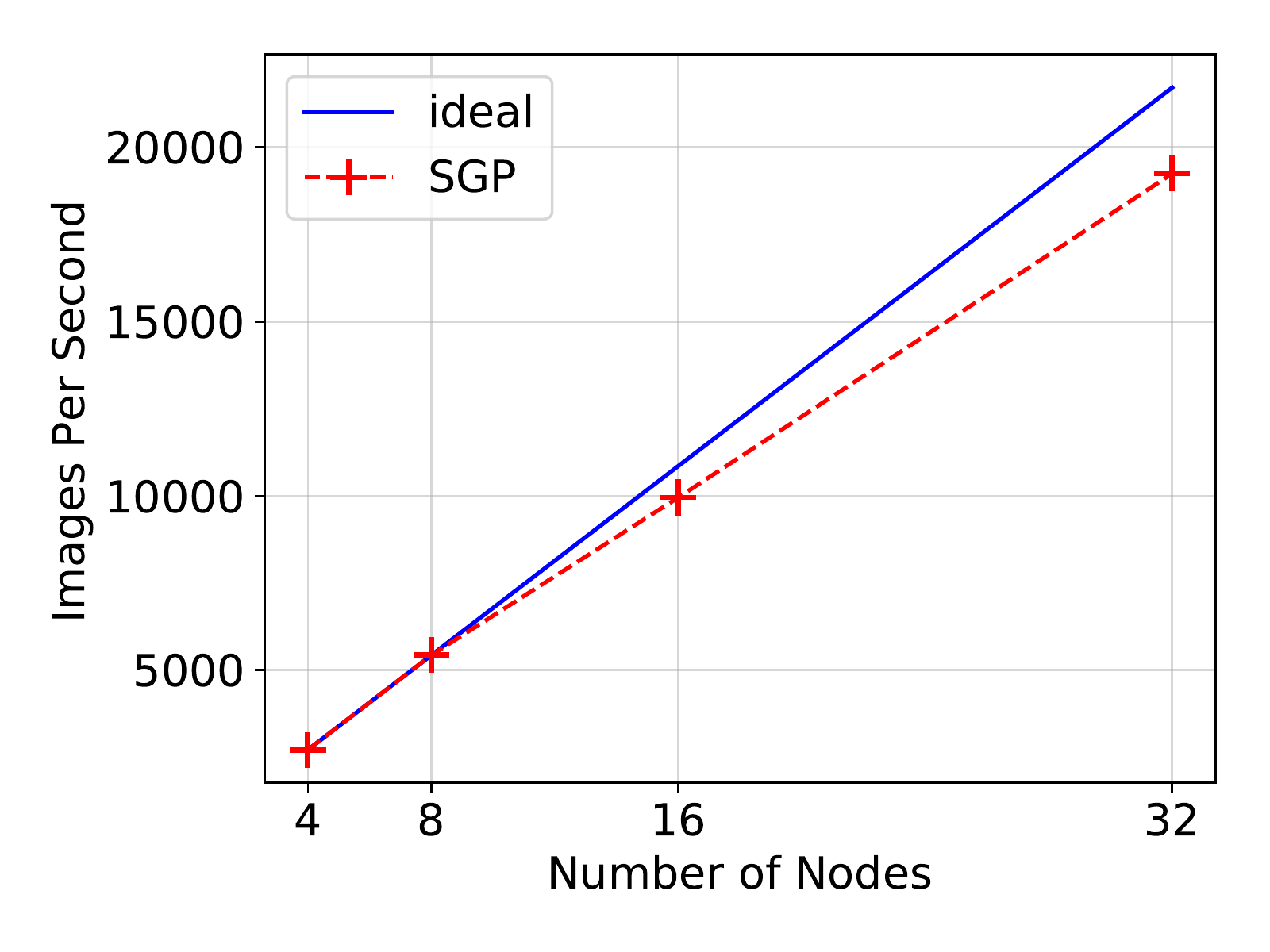}}
	\subfloat[InfiniBand 100Gbit/s]{\includegraphics[width=0.3\textwidth]{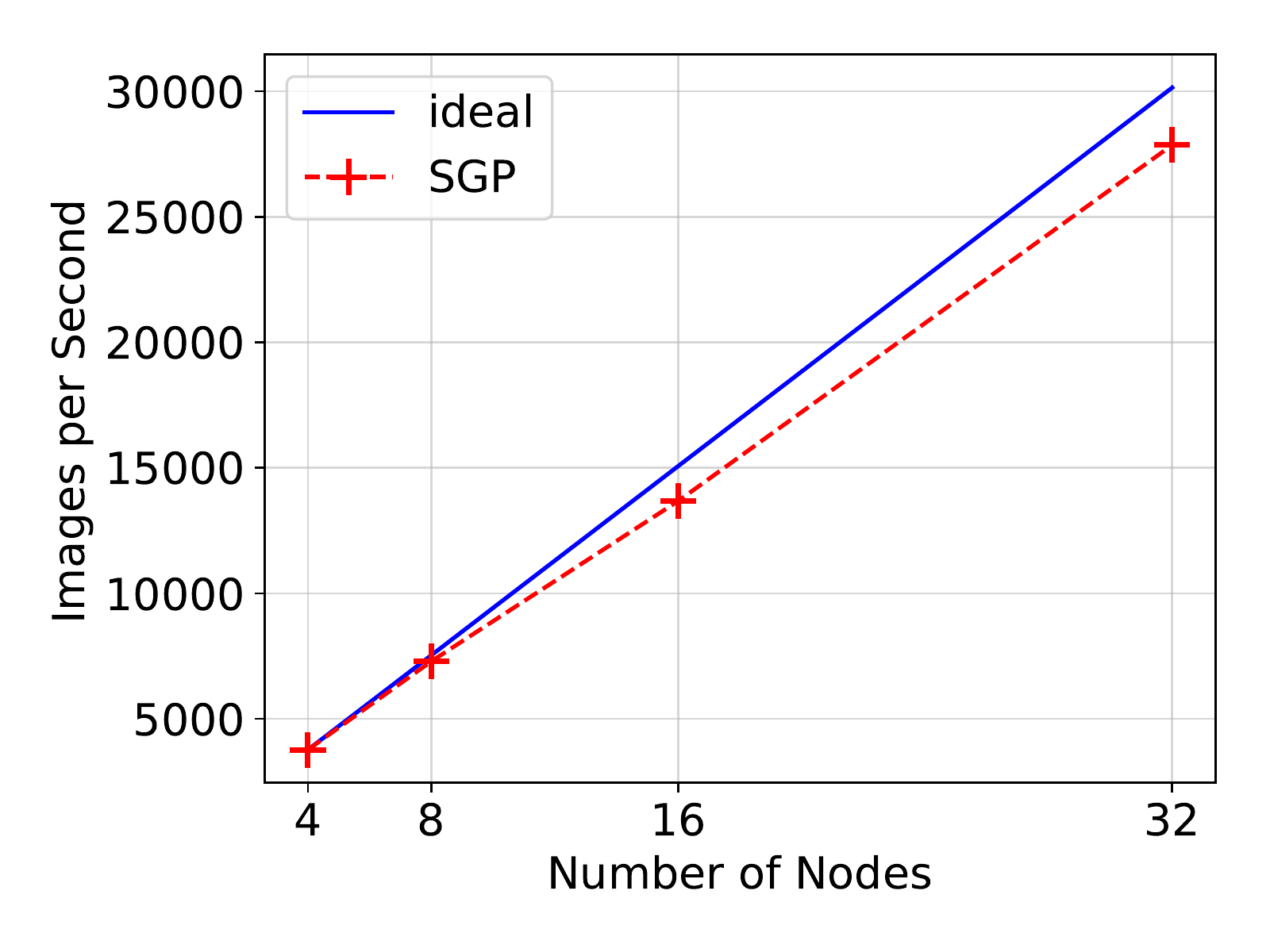}}
	\subfloat[Scaling of SGP and SGP]{\includegraphics[width=0.3\textwidth]{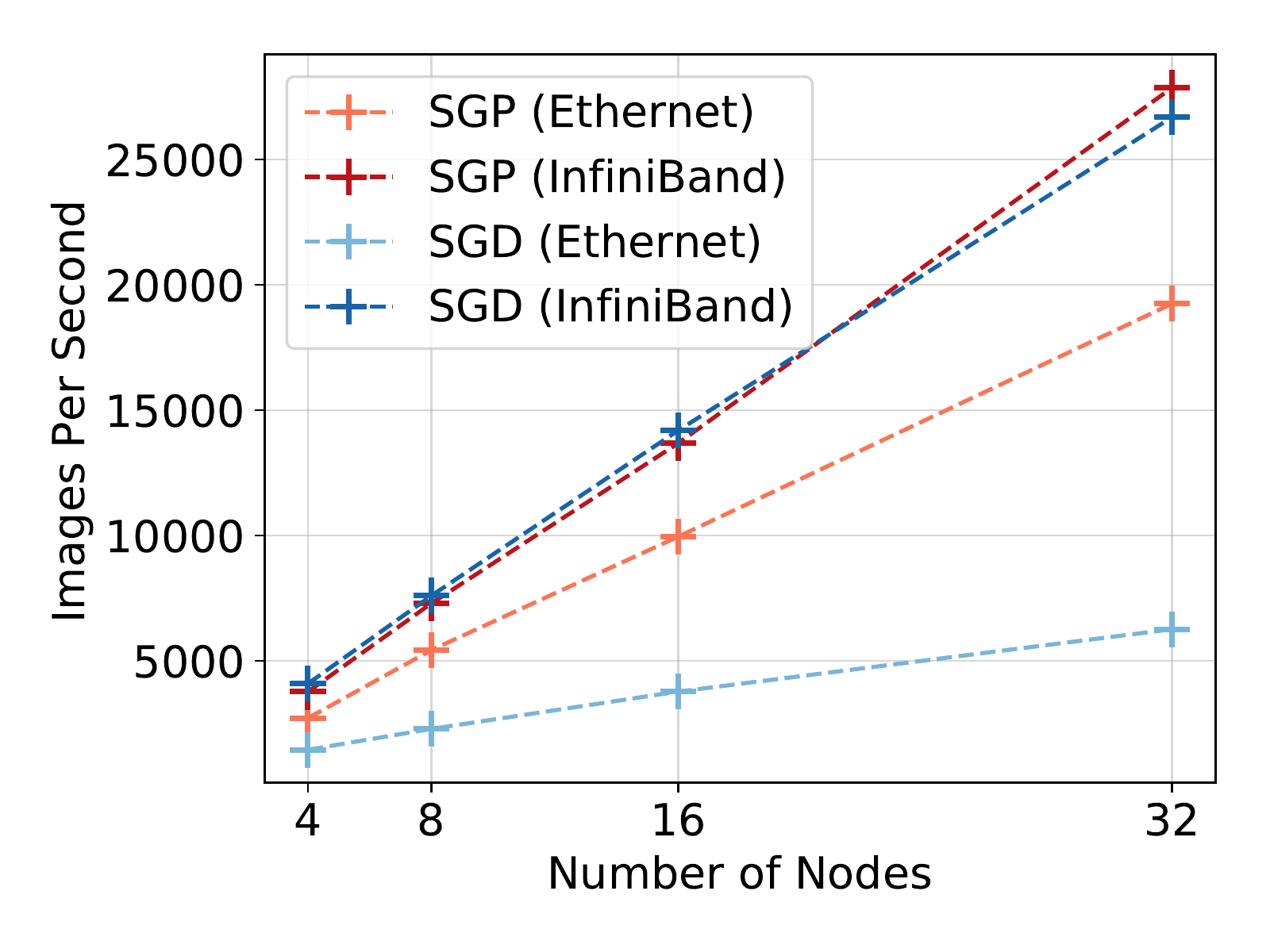}}
	\caption{SGP throughput on Ethernet (a) and InfiniBand (b). SGP  exhibits 88.6\% scaling efficiency on Ethernet 10Gbit/s and 92.4\% on InfiniBand. Comparison of SGD vs SGP throughput in Figure (c) shows that SGP exhibit better scaling and is more robust to high-latency interconnect.}
\label{appendix:scaling}
\end{figure}
Figure~\ref{appendix:scaling} highlights SGP input images throughput as we scale up the number of cluster node on both Ethernet 10Gbit/s and Infiniband 100Gbit/s. SGP  exhibits 88.6\% scaling efficiency on Ethernet 10Gbit/s and 92.4\% on InfiniBand and stay close to the ideal scaling in both cases. In addition Figure (c) shows that SGP exhibit better scaling as we increase the network size and is more robust to high-latency interconnect.

\section{Proofs of Theoretical Guarantees}
\label{sec:proofs}

Our convergence rate analysis is divided into three main parts. In the first one (subsection \ref{ImportantUpperBounds}) we present upper bounds for three important expressions that appear in our computations. In subsection \ref{ImportLemma} we focus on proving the important for our analysis Lemma~\ref{MainLemma} based on which we later build the proofs of our main Theorems. Finally in the third part (subsection \ref{ProofMain}) we provide the proofs for Theorems~\ref{thm:avg-convergence} and \ref{secondTheorem}.

\paragraph{Preliminary results.}
In our analysis, two preliminary results are extensively used. We state them here for future reference.
\begin{itemize}
\item Let $a,b \in \R$. Since $(a-b)^2 \geq 0$, it holds that 
\begin{equation}
\label{formula}
2ab \leq a^2+b^2.
\end{equation}
Thus, $\norm{\bm x} \norm{\bm y} \leq (\norm{\vx}^2 + \norm{\vy}^2) / 2$.
\item Let $r \in (0,1)$ then from the summation of geometric sequence and for any $K\leq \infty$ it holds that
\begin{equation}
\label{geometric}
\sum_{k=0}^K r^k  \leq \sum_{k=0}^\infty r^k=\frac{1}{1-r}.
\end{equation}
\end{itemize}

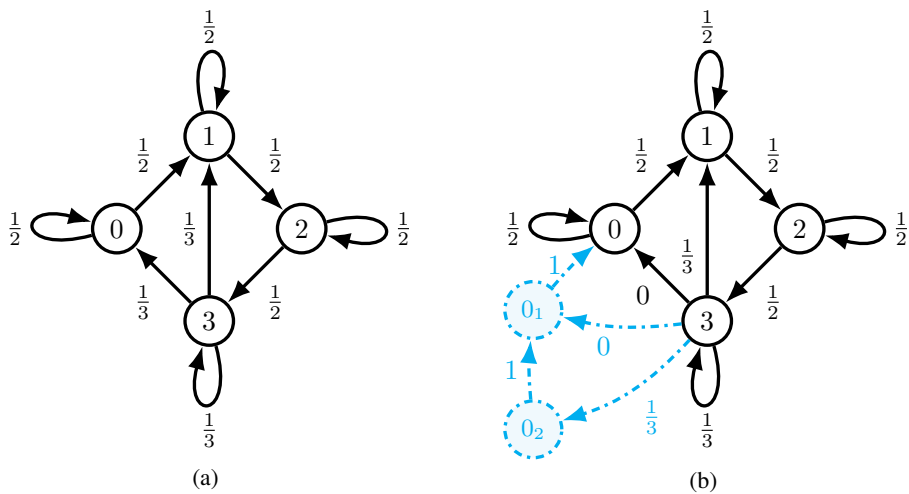
\begin{figure}[t]
\centering

\tikzstyle{edge} = [very thick, -{Latex}]
\tikzstyle{vedge} = [very thick, cyan, dashdotted, -{Latex}]
\tikzstyle{non-virtualnode} = [draw, circle, fill=white, very thick]
\tikzstyle{virtualnode} = [draw, circle, cyan, fill=cyan!5, very thick, dashdotted]

\subfloat[]{
\resizebox{0.34\textwidth}{!}{
\begin{tikzpicture}[baseline]
\node[non-virtualnode] (topcircle) {$1$};
\node[non-virtualnode] (leftcircle) [below left = 1cm of topcircle] {$0$};
\node[non-virtualnode] (rightcircle) [below right = 1cm of topcircle] {$2$};
\node[non-virtualnode] (bottomcircle) [below right = 1cm of leftcircle] {$3$};
\draw[edge] (leftcircle) -- node[above left]{$\frac{1}{2}$} (topcircle);
\draw[edge] (bottomcircle) -- node[below left]{$\frac{1}{3}$} (leftcircle);
\draw[edge] (rightcircle) -- node[below right]{$\frac{1}{2}$} (bottomcircle);
\draw[edge] (topcircle) -- node[above right]{$\frac{1}{2}$}(rightcircle);
\draw[edge] (bottomcircle) -- node[left]{$\frac{1}{3}$} (topcircle);
\draw[edge] (topcircle) to[loop above, looseness=15] node[above]{$\frac{1}{2}$} (topcircle);
\draw[edge] (leftcircle) to[loop left, looseness=15] node[left]{$\frac{1}{2}$} (leftcircle);
\draw[edge] (rightcircle) to[loop right, looseness=15] node[right]{$\frac{1}{2}$} (rightcircle);
\draw[edge] (bottomcircle) to[loop below, looseness=15] node[below]{$\frac{1}{3}$} (bottomcircle);
\end{tikzpicture}
}} \quad \quad
\subfloat[]{
\resizebox{0.34\textwidth}{!}{
\begin{tikzpicture}[baseline]
\node[non-virtualnode] (topcircle) {$1$};
\node[non-virtualnode] (leftcircle) [below left = 1cm of topcircle] {$0$};
\node[non-virtualnode] (rightcircle) [below right = 1cm of topcircle] {$2$};
\node[non-virtualnode] (bottomcircle) [below right = 1cm of leftcircle] {$3$};
\node[virtualnode] (v_leftcircle1) [below left = 0.75cm of leftcircle] {\footnotesize{$0_1$}};
\node[virtualnode] (v_leftcircle2) [below = 0.75cm of v_leftcircle1] {\footnotesize{$0_2$}};
\draw[edge] (leftcircle) -- node[above left]{$\frac{1}{2}$} (topcircle);
\draw[edge] (bottomcircle) -- node[below left]{$0$} (leftcircle);
\draw[edge] (rightcircle) -- node[below right]{$\frac{1}{2}$} (bottomcircle);
\draw[edge] (topcircle) -- node[above right]{$\frac{1}{2}$} (rightcircle);
\draw[edge] (bottomcircle) -- node[below left]{$\frac{1}{3}$} (topcircle);

\draw[edge] (topcircle) to[loop above, looseness=15] node[above]{$\frac{1}{2}$} (topcircle);
\draw[edge] (leftcircle) to[loop left, looseness=15] node[left]{$\frac{1}{2}$} (leftcircle);
\draw[edge] (rightcircle) to[loop right, looseness=15] node[right]{$\frac{1}{2}$} (rightcircle);
\draw[edge] (bottomcircle) to[loop below, looseness=15] node[below]{$\frac{1}{3}$} (bottomcircle);

\draw[vedge] (bottomcircle) to[bend left=10] node[below left]{$0$} (v_leftcircle1);
\draw[vedge] (bottomcircle) to[bend left=15] node[below right]{$\frac{1}{3}$} (v_leftcircle2);
\draw[vedge] (v_leftcircle1) to[bend left=5] node[left]{$1$} (leftcircle);
\draw[vedge] (v_leftcircle2) to[bend left=5] node[left]{$1$} (v_leftcircle1);
\end{tikzpicture}
}}
\caption{(a) Example of a $4$-node network, with mixing-weights drawn on edges. (b) Example of a $4$-node network, augmented with virtual nodes and edges, with mixing-weights draw on edges. The virtual nodes/edges are used to model the fact that messages from node $3$ to node $0$ can experience a delay of at most $2$ iterations. In this particular example, we model the fact that node $3$ sends a message to node $0$ with a delay of $2$ iterations. All virtual nodes always forward all of their messages to their out-neighbor.}
\label{fig:aug_graph}
\end{figure}

\paragraph{Modeling message delays.} To model message delays we follow the procedure used in~\citet{assran2018asynchronous} (which we will reiterate here).
In essence, we augment the communication topology (and the mixing matrices) with virtual nodes that store messages that were transmitted, but not yet received. Similar graph augmentations have been used in~\citet{charalambous2015distributed} and~\citet{hadjicostis2014average}.

We commence by presenting a brief example of the delay-model before formalizing the discussion.
Figure~\ref{fig:aug_graph}~(a) shows an example of a 4-node network at some arbitrary iteration $k$. Suppose each node communicates to each of its out-neighbors with uniform mixing weights.
These mixing weights are labeled on the corresponding edges in Figure~\ref{fig:aug_graph}~(a).
Then, the mixing matrix $\bm{P}^{(k)} \in \R^{4 \times 4}$ is given by
\[
\bm{P}^{(k)} =
\bordermatrix{
    & & & & \cr
    & 1/2 & 0 & 0 & 1/3 \cr
    & 1/2 & 1/2 & 0 & 1/3 \cr
    & 0 & 1/2 & 1/2 & 0 \cr
    & 0 & 0 & 1/2 & 1/3
}.
\]
Column indices correspond to sending nodes, and row indices correspond to receiving nodes.
Recall that sending nodes choose the mixing weights (columns of $\bm{P}^{(k)}$) used to pre-weight outgoing messages.
Note that the matrix $\bm{P}^{(k)}$ is column stochastic (all columns sum to $1$) --- the crucial requirement of our analysis.
Thus at time $k + 1$, we have the following parameter updates
\begin{align*}
    \itr{\bm{x}}{k + 1}_0 &= \frac{1}{2} \itr{\bm{x}}{k}_0 + \frac{1}{3} \itr{\bm{x}}{k}_3 \\
    \itr{\bm{x}}{k + 1}_1 &= \frac{1}{2} \itr{\bm{x}}{k}_0 + \frac{1}{2} \itr{\bm{x}}{k}_1 +  \frac{1}{3} \itr{\bm{x}}{k}_3 \\
    \itr{\bm{x}}{k + 1}_2 &= \frac{1}{2} \itr{\bm{x}}{k}_1 + \frac{1}{2} \itr{\bm{x}}{k}_2 \\
    \itr{\bm{x}}{k + 1}_3 &= \frac{1}{2} \itr{\bm{x}}{k}_2 + \frac{1}{3} \itr{\bm{x}}{k}_3.
\end{align*}
In particular, each node updates its variables with the most recent information from its in-neighbours.
Similar equations can be written for the push-sum weights $w$.

Now suppose that node $3$ sends messages to its neighbors, nodes $0$ and $1$, at iteration $k$, but the message to node $0$ doesn't arrive until iteration $k + 2$.
To model this delay, we augment the graph topology with virtual nodes $0_1$, $0_2$ (cf. Figure~\ref{fig:aug_graph}~(b)).
The virtual nodes are initialized with parameters $\itr{\bm{x}}{0} = \bm{0}$ and push-sum weight $\itr{w}{0} = 0$.
Given this model, node $3$ can send its pre-weighted message to virtual node $0_2$ (instead of node $0$) at iteration $k$, while the rest of communication proceeds business as usual. At the subsequent iteration, $k + 1$, node $0_2$ forwards this message to node $0_1$. Subsequently, at iteration, $k+2$, node $0_1$ forwards this message to node $0$, thereby modeling a $2$-iteration message delay.
The corresponding mixing matrix at iteration $k$ is given by
\[
\bm{P}^{(k)} =
\begin{blockarray}{c cccccc}
    & & & & & 0_1 & 0_2 \\
\begin{block}{c (cccccc)}
    \begin{block*}{c cccc|cc}
        & 1/2 & 0 & 0 & 1/2 & 1 & 0 \\
        & 1/2 & 1/2 & 0 & 0 & 0 & 0 \\
        & 0 & 1/2 & 1/3 & 0 & 0 & 0 \\
        & 0 & 0 & 1/3 & 1/2 & 0 & 0 \\ \cline{2-5}
    \end{block*}
    \begin{block*}{c cccccc}
    0_1 & 0 & 0 & 0 & 0 & 0 & 1 \\
    0_2 & 0 & 0 & 1/3 & 0 & 0 & 0 \\
    \end{block*}
\end{block}
\end{blockarray}.
\]
Note that we have added two extra rows and columns corresponding to the virtual nodes $0_1$ and $0_2$.
As intended, node $2$ sends a message to node $0_2$ (instead of node $0$) at iteration $k$. Node $0_2$ always forwards any and all information it receives to node $0_1$, and node $0_1$ always forwards any and all information it receives to node $0$.
Since all virtual nodes are initialized with parameters $\itr{\bm{x}}{k} = \bm{0}$ and push-sum weight $\itr{w}{0} = 0$, they do not have any impact on the final consensus value.
The sole purpose of the virtual nodes is to store messages that are in-transit (transmitted but not yet received).

If the message delays at every node are upper-bounded by $\tau$, then we can generalize this procedure, and add $\tau$ virtual nodes for every (non-virtual) node in the network.
Thus, the augmented graph has $n(\tau + 1)$ nodes in total.
The corresponding augmented mixing matrix, $\bm{P}^{(k)} \in \R^{n(\tau + 1) \times n(\tau + 1)}$, in block matrix form is written as
\[
\bm{P}^{(k)} = 
\begin{blockarray}{r ccccc}
    & & (0_1, 1_1 \ldots) & (0_2, 1_2, \ldots) & & (0_\tau, 1_\tau, \ldots) \\
\begin{block}{r (ccccc)}
    & \bm{\widetilde{P}}^{(k)}_0 & \bm{I} & \bm{0} & \hdots & \bm{0} \\
    & \bm{\widetilde{P}}^{(k)}_1 & \bm{0} & \bm{I} & & \vdots \\
    & \vdots & \vdots  & & \ddots & \bm{0} \\
    & \bm{\widetilde{P}}^{(k)}_{\tau - 1} & \bm{0} & \hdots & \bm{0} & \bm{I} \\
    & \bm{\widetilde{P}}^{(k)}_{\tau} & \bm{0} & \hdots & \bm{0} & \bm{0} \\
\end{block}
\end{blockarray}
\]
where each block is of size $n \times n$.
In particular, if node $i$ sends a message to node $j$ with weight $p^{(k)}_{j, i}$ at iteration $k$, and that message is received with delay $r$ (\ie, received at iteration $k + r$), then
\[
    [\bm{P}^{(k)}_r]_{j,i} = p^{(k)}_{j, i},
\]
otherwise
\[
    [\bm{P}^{(k)}_r]_{j,i} = 0.
\]
The off-diagonal of block identity matrices $\bm{I}$ denote the fact that the virtual nodes always forward all of their messages to the next node in the delay daisy-chain.
It is straightforward to verify that these augmented mixing matrices are still column stochastic at all iterations $k$.
We refer the curious reader to~\citet{assran2018asynchronous, charalambous2015distributed, hadjicostis2014average} for a deeper discussion of the augmented delay model.

\paragraph{Matrix Representation.}
In Algorithm~\ref{SGPalg}, SGP was presented from node $i$'s perspective (for all $i \in [n]$).
However, we can actually write the SGP update at each iteration from a global viewpoint.
To see this, first define the following matrices, for all $r = 1, 2, \ldots \tau$,
\[
    \bX^{(k)}_r = \left[\vx_{1_r}^{(k)},\vx_{2_r}^{(k)},\dots, \vx_{n_r}^{(k)} \right] \in \R^{d \times n}.
\]
The matrix $\bX^{(k)}_r$ denotes a concatenation of all the delay-$r$ nodes' parameters at iteration $k$. For the purpose of notational consistency, we let the matrix $\bX^{(k)}_0$ denote the concatenation of all the non-virtual nodes' parameters.
We generalize this notation to other variables as well.
In block-matrix form, we can define the augmented parameter matrix
\[
    \bX^{(k)} = [\bX^{(k)}_0, \bX^{(k)}_1, \ldots, \bX^{(k)}_\tau] \in \R^{d \times n(\tau + 1)},
\]
which denotes a concatenation of \emph{all} (virtual and non-virtual) nodes' parameters at iteration $k$.
Recall that the we initialize all virtual nodes with parameters $\itr{\bm{x}}{k} = \bm{0}$ and push-sum weight $\itr{w}{0} = 0$.
Additionally, since the virtual nodes are only used to model delays, and do not compute any gradient updates, we use the convention that $\itr{\bm{z}}{k} = 0$, $\itr{\xi}{k} = 0$, and $\nabla F(\itr{\bm{z}}{k}; \itr{\xi}{k}) = \bm{0}$ for all virtual nodes at all times $k$.
Therefore, we define the augmented de-biased parameter matrix and stochastic-seed matrix as follows
\[
    \bZ^{(k)} = [\bZ^{(k)}_0, \bm{0}, \ldots, \bm{0}] \in \R^{d \times n(\tau + 1)}; \quad \itr{{\xi}}{k} = [\bm{\xi}^{(k)}_0, \bm{0}, \ldots, \bm{0}] \in \R^{n(\tau + 1)}.
\]
Similarly, we define the augmented stochastic-gradient matrix as
\[
    \nabla \bm{F}(\bZ^{(k)}; \itr{\bm{\xi}}{k}) = [\nabla \bm{F}_0(\bZ^{(k)}_0; \itr{\bm{\xi}}{k}_0), \bm{0}, \ldots, \bm{0}] \in \R^{d \times n(\tau + 1)},
\]
where the block matrix $\nabla \bm{F}_0(\bZ^{(k)}_0; \itr{\bm{\xi}}{k}_0)$ denotes the concatenation of all non-virtual nodes' stochastic gradients at iteration $k$. Precisely
\[
    \nabla \bm{F}_0(\bZ^{(k)}_0, \bm{\xi}^{(k)}_0) = \left[\nabla F_1(\vz_{1}^{(k)} ; \xi_{1}^{(k)}),\nabla F_2(\vz_{2}^{(k)} ; \xi_{2}^{(k)}),\dots,\nabla F_n(\vz_{n}^{(k)} ; \xi_{n}^{(k)})\right] \in \R^{d \times n}.
\]
We also define the augmented expected gradient matrix (with respect to local node data distributions) as
\[
    \nabla \bm{F}(\bZ^{(k)}) = [\nabla \bm{F}_0(\bZ^{(k)}_0), \bm{0}, \ldots, \bm{0}] \in \R^{d \times n(\tau + 1)},
\]
where the block matrix $\nabla \bm{F}_0(\bZ^{(k)}_0)$ denotes the concatenation of all non-virtual nodes' expected stochastic gradients at iteration $k$. Precisely
\[
    \nabla \bm{F}_0(\bZ^{(k)}_0) = \left[\Exp_{\xi_1^{(k)} \sim \cD_1}[\nabla F_1(\vz_{1}^{(k)} ; \xi_{1}^{(k)})], \Exp_{\xi_2^{(k)} \sim \cD_2}[\nabla F_2(\vz_{2}^{(k)} ; \xi_{2}^{(k)})],\dots, \Exp_{\xi_n^{(k)} \sim \cD_n}[\nabla F_n(\vz_{n}^{(k)} ; \xi_{n}^{(k)})] \right] \in \R^{d \times n}.
\]
For notational convenience, we simply write $\nabla f_i(\vz_i^{(k)}) \defeq \Exp_{\xi_i^{(k)} \sim \cD_i}[\nabla F_i(\vz_{i}^{(k)} ; \xi_{i}^{(k)})]$.
Using the above matrices, the $6^{th}$ step of SGP in Algorithm~\ref{SGPalg} (lines 19 to 24 in OSGP Algorithm~\ref{alg:osgp}) can be expressed from a global perspective as follows
\begin{equation}
\label{UpdateRule}
    \bX^{(k+1)}=\left( \bX^{(k)} - \gamma \nabla \bm{F}(\bZ^{(k)}, \bm{\xi}^{(k)}) \right) [\bm{P}^{(k)}]^T,
\end{equation}
where $[\bm{P}^{(k)}]^T \in \R^{n(\tau + 1) \times n(\tau + 1)}$ is the transpose of the augmented mixing matrix.

Lastly, let $\nbar \defeq n (\tau + 1)$, and let $\itr{\xbar}{k} = (1/n) \bX^{(k)} \1_\mathrm{\nbar}$ denote the average of all nodes' parameters at iteration $k$. Note that this definition incorporates parameters that are in-transit.

\paragraph{Bound for the mixing matrices.} Next we state a known result from the control literature studying gossip-based optimization which allows us to bound the distance between the de-biased parameters at each node and the node-wise average.

\newcommand{\bmP}{\bm{P}}
\newcommand{\bmA}{\bm{A}}
Recall that we have assumed that the sequence of communication topologies is $B$-strongly connected.
A directed graph is called \emph{strongly connected} if every pair of vertices is connected with a directed path (\ie, following the direction of edges).
A sequence of directed graphs is called $B$-strongly connected if the graph with edge set $\bigcup_{k=lB}^{(l+1)B - 1} E^{(k)}$ is strongly connected, for every $l \ge 0$.
Recall that we have also assumed that the upper bound on the message delays is $\tau$ iterations. In particular, we assume all messages reach their destination within $\tau$-iterations from transmission. \ie, a message in-transit does not get dropped when the communication topology changes.

If the maximum message delay is $\tau$, and all non-zero mixing weights are at least $\epsilon$ large, and the diameter of the graph with edge set $\bigcup_{k=lB}^{(l+1)B - 1} E^{(k)}$ has diameter at most $\Delta$, then the product
\[
    \bm{A}^{(k)} := \bm{P}^{(k + (\tau + 1) \Delta B - 1)} \cdots \bmP^{(k + 1)} \bmP^{(k)}
\]
has no non-zero entries in the first $n$-rows (corresponding to non-virtual agents). Moreover, every entry in the first $n$-rows of $\bmA^{(k)}$ is at least $\epsilon^{(\tau + 1) \Delta B}$.

If we further assume that all nodes have at most $D$ out-neighbors in any iteration, and that all nodes always assign mixing weights uniformly, then $\epsilon = D^{-1}$, and every entry in the first $n$-rows of $\bmA^{(k)}$ is at least $D^{- (\tau + 1) \Delta B}$.

\begin{lem}
\label{MidosTheorem}
Suppose that Assumption~3 (mixing connectivity) holds. Let $\lambda = 1 - n D^{-(\tau + 1) \Delta B}$ and let $q = \lambda^{1/((\tau + 1) \Delta B + 1)}$. Then there exists a constant
\[
C < \frac{2 \sqrt{d} D^{(\tau + 1) \Delta B}}{\lambda^{\frac{(\tau + 1) \Delta B + 2}{(\tau + 1) \Delta B + 1}}},
\]
where $d$ is the dimension of $\xbar^{(k)}$, $\itr{z_i}{k}$, and $\itr{x_i}{0}$, such that, for all $i=1,2,\dots,n$ (non-virtual nodes) and $k\geq0$, 
$$\norm{ \itr{\xbar}{k} - \itr{z_i}{k} }_2 \leq C q^k \norm{ \itr{x_i}{0} }_2  + \gamma C \sum^k_{s=0} q^{k-s} \norm{ \ssgrad{i}{z}{\xi}{s} }_2.$$
\end{lem}

This particular lemma follows after a small adaptation to Theorem~1 in \citet{assran2018asynchronous} and its proof is based on \citet{Wolfowitz1963products}. Similar bounds appear in a variety of other papers, including \citet{Nedic2016stochastic}.

\subsection{Important Upper Bounds}
\label{ImportantUpperBounds}
\begin{lem}[Bound of stochastic gradient]
\label{boundonnormGradient}
We have the following inequality under Assumptions 1 and 2:
$$\E \norm{ \nabla f_i(\itr{\vz}{k}_i) }^2 \leq 3 L^2 \E \norm{ \itr{\vz}{k}_i - \itr{\xbar}{k} }^2  + 3 \zeta^2 + 3 \E \norm{ \nabla f(\itr{\xbar}{k}) }^2 $$
\end{lem}
\begin{proof}
\begin{eqnarray*}
	\E \norm{ \nabla f_i(\itr{\vz}{k}_i) }^2 & \leq &\ 3 \E \norm{ \nabla f_i(\itr{\vz}{k}_i) - \nabla f_i(\itr{\xbar}{k}) }^2 + 3 \E \norm{ \nabla f_i(\itr{\xbar}{k}) - \nabla f(\itr{\xbar}{k}) }^2  + 3 \E \norm{ \nabla f(\itr{\xbar}{k}) }^2 \notag\\
&\overset{\text{L-smooth}}{\leq} & 3 L^2 \E \norm{ \itr{\vz}{k}_i - \itr{\xbar}{k} }^2  + 3 \E \norm{ \nabla f_i(\itr{\xbar}{k}) - \nabla f(\itr{\xbar}{k}) }^2 + 3 \E \norm{ \nabla f(\itr{\xbar}{k}) }^2 \notag\\
&\overset{\text{\text{Bounded Variance}}}{\leq} & 3 L^2 \E \norm{ \itr{\vz}{k}_i - \itr{\xbar}{k} }^2  + 3 \zeta^2 + 3 \E \norm{ \nabla f(\itr{\xbar}{k}) }^2
\end{eqnarray*}
\end{proof}

\begin{lem}
\label{BoundOfQ}
Let Assumptions 1-3 hold. Then,
\begin{eqnarray}
Q^{(k)}_i = \E \norm{ \itr{\xbar}{k} - \itr{z_i}{k} }^2 &\leq & \left( \gamma^2 \frac{4 C^2}{(1-q)^2} + \gamma \frac{q^k C^2}{1-q} \right) \sigma^2 + \left( \gamma^2 \frac{12 C^2}{(1 - q)^2} + \gamma \frac{q^k 3 C^2}{1-q} \right) \zeta^2  \notag\\
	& +& \left( \gamma^2 \frac{12 L^2 C^2}{1 - q} + \gamma q^k 3 L^2 C^2 \right) \sum^k_{j=0} q^{k-j} \itr{Q}{j}_i  \notag\\
	& + &\left( \gamma^2 \frac{12 C^2}{1 - q} + \gamma q^k 3 C^2 \right) \sum^k_{j=0} q^{k-j} \E \norm{ \nabla f(\itr{\xbar}{j}) }^2  \notag\\
	& + & \left(q^{2k} C^2 + \gamma q^k \frac{2 C^2}{1-q} \right)  \norm{ \itr{\vx_i}{0} }^2.
\end{eqnarray}
\end{lem}
\begin{proof}
\begin{eqnarray}
	Q^{(k)}_i &= & \E \norm{ \itr{\xbar}{k} - \itr{\vz_i}{k} }^2 \notag\\
&\overset{Lemma~\ref{MidosTheorem}}{\leq}& \E  \left( C q^k \norm{ \itr{\vx_i}{0} }  + \gamma C \sum^k_{s=0} q^{k-s} \norm{ \ssgrad{i}{\vz}{\xi}{s} } \right)^2 \notag\\
	&=& \E \left( C q^k \norm{ \itr{\vx_i}{0} }  + \gamma C \sum^k_{s=0} q^{k-s} \norm{ \ssgrad{i}{\vz}{\xi}{s} - \nabla  f_i(\itr{\vz}{s}_i) + \nabla f_i(\itr{\vz}{s}_i) } \right)^2 \notag\\
	&\leq& \E \left( \underbrace{C q^k \norm{ \itr{\vx_i}{0} } }_{a}  + \underbrace{\gamma C \sum^k_{s=0} q^{k-s} \norm{ \ssgrad{i}{\vz}{\xi}{s} - \nabla f_i(\itr{\vz}{s})}}_{b} + \underbrace{\gamma C \sum^k_{s=0} q^{k-s} \norm{\nabla  f_i(\itr{\vz}{s}_i) }}_{c} \right)^2 \notag\\
\end{eqnarray}
\allowdisplaybreaks
Thus, using the above expressions of $a$, $b$ and $c$ we have that $Q^{(k)}_i \leq \E(a^2 +b^2 + c^2 + 2ab + 2bc + 2ac)$. 
Let us now obtain bounds for all of these quantities:
\begin{align*}
	a^2 &= C^2 \norm{ \itr{\vx_i}{0} }^2 q^{2k} \\
	b^2 &= \gamma^2 C^2 \sum^k_{j=0} q^{2(k - j)} \norm{ \ssgrad{i}{\vz}{\xi}{j} - \nabla f_i(\itr{\vz}{j}_i) }^2 \notag\\
	&\quad +  \underbrace{2 \gamma^2 C^2 \sum^k_{j=0} \sum^k_{s=j + 1} q^{2k -j - s} \norm{ \ssgrad{i}{\vz}{\xi}{j} - \nabla f_i(\itr{\vz}{j}_i) } \norm{ \ssgrad{i}{\vz}{\xi}{s} - \nabla f_i(\itr{\vz}{s}_i) } }_{b_1} \\
	c^2 &= \gamma^2 C^2 \sum^k_{j=0} q^{2(k - j)} \norm{ \nabla  f_i(\itr{\vz}{j}_i) }^2 +  \underbrace{2 \gamma^2 C^2 \sum^k_{j=0} \sum^k_{s=j + 1} q^{2k -j - s} \norm{ \nabla  f_i(\itr{\vz}{j}_i) } \norm{ \nabla  f_i(\itr{\vz}{s}_i) }}_{c_1} \\
	2ab &= 2 \gamma C^2 q^k \norm{ \itr{\vx_i}{0} } \sum^k_{s=0} q^{k-s} \norm{ \ssgrad{i}{\vz}{\xi}{s} - \nabla  f_i(\itr{\vz}{s}_i) } \\
	2ac &= 2 \gamma C^2 q^k \norm{ \itr{\vx_i}{0} } \sum^k_{s=0} q^{k-s} \norm{ \nabla f_i(\itr{\vz}{s}_i) } \\
	2bc &= 2 \gamma^2 C^2 \sum^k_{j= 0}\sum^k_{s= 0} q^{2k-j-s} \norm{ \ssgrad{i}{\vz}{\xi}{j} - \nabla  f_i(\itr{\vz}{j}_i) } \norm{ \nabla f_i(\itr{\vz}{s}_i) }.
\end{align*}

\allowdisplaybreaks
The expression $b_1$ is bounded as follows:
\begin{align}
\label{boundb1}
	b_1 &= \gamma^2 C^2 \sum^k_{j=0} \sum^k_{s=j + 1} q^{2k -j - s} 2 \norm{ \ssgrad{i}{\vz}{\xi}{j} - \nabla f_i(\itr{\vz}{j}_i) } \norm{ \ssgrad{i}{\vz}{\xi}{s} - \nabla f_i(\itr{\vz}{s}_i) }\notag\\
	&\overset{(\ref{formula})}{\leq} \gamma^2 C^2 \sum^k_{j=0} \sum^k_{s=j+1} q^{2k - s - j}  \norm{ \ssgrad{i}{\vz}{\xi}{j} - \nabla f_i(\itr{\vz}{j}_i) }^2 \notag\\
	&\quad + \gamma^2 C^2 \sum^k_{j=0} \sum^k_{s=j+1} q^{2k - s - j}  \norm{  \ssgrad{i}{\vz}{\xi}{s} - \nabla  f_i(\itr{\vz}{s}_i) }^2 \notag\\
	&\leq  \gamma^2 C^2 \sum^k_{j=0} \sum^k_{s=0} q^{2k - s - j}  \norm{ \ssgrad{i}{\vz}{\xi}{j} - \nabla  f_i(\itr{\vz}{j}_i) }^2 \notag\\
	&\quad +  \gamma^2 C^2 \sum^k_{j=0} \sum^k_{s=0} q^{2k - s - j}  \norm{  \ssgrad{i}{\vz}{\xi}{s} - \nabla  f_i(\itr{\vz}{s}_i) }^2 \notag\\
	&= \gamma^2 C^2 \sum^k_{j=0} q^{k - j} \norm{ \ssgrad{i}{\vz}{\xi}{j} - \nabla  f_i(\itr{\vz}{j}_i) }^2 \sum^k_{s=0} q^{k - s} \notag\\
	&\quad + \gamma^2 C^2 \sum^k_{s=0} q^{k - s}  \norm{  \ssgrad{i}{\vz}{\xi}{s} - \nabla  f_i(\itr{\vz}{s}_i) }^2 \sum^k_{j=0} q^{k - j} \notag\\
	&\overset{(\ref{geometric})}{\leq} \frac{1}{1 - q} \gamma^2 C^2 \sum^k_{j=0} q^{k - j}  \norm{ \ssgrad{i}{\vz}{\xi}{j} - \nabla  f_i(\itr{\vz}{j}_i) }^2 \notag\\
	&\quad +  \frac{1}{1-q} \gamma^2 C^2 \sum^k_{s=0} q^{k -s} \norm{  \ssgrad{i}{\vz}{\xi}{s} - \nabla  f_i(\itr{\vz}{s}_i) }^2 \notag\\
	&=  \frac{2}{1 - q} \gamma^2 C^2 \sum^k_{j=0} q^{k - j}  \norm{ \ssgrad{i}{\vz}{\xi}{j} - \nabla  f_i(\itr{\vz}{j}_i) }^2. 
\end{align}
Thus,
\begin{eqnarray}
b^2&=&\gamma^2 C^2 \sum^k_{j=0} q^{2(k - j)} \norm{ \ssgrad{i}{\vz}{\xi}{j} - \nabla f_i(\itr{\vz}{j}_i) }^2 + b_1 \notag\\
& \leq &\frac{ \gamma^ 2 C^2 }{1- q} \sum^k_{j=0} q^{k-j} \norm{ \ssgrad{i}{\vz}{\xi}{j} - \nabla f_i(\itr{\vz}{j}_i) }^2 + b_1\notag\\
&\overset{(\ref{boundb1})}{\leq} & \frac{3 \gamma^2 C^2}{1 - q} \sum^k_{j=0} q^{k-j} \norm{ \ssgrad{i}{\vz}{\xi}{j} - \nabla f_i({\itr{\vz}{j}_i}) }^2
\end{eqnarray}
where in the first inequality above we use the fact that for  $q \in (0, 1)$, we have $q^k< \frac{1}{1-q}, \forall k>0$.

By identical construction we have
\begin{align*}
	c^2 \leq \frac{3 \gamma^2 C^2}{1-q} \sum^k_{j=0} q^{k-j} \norm{ \nabla f_i(\itr{\vz}{j}_i) }^2.
\end{align*}

Now let us bound the products $2ab$, $2ac$ and $2bc$.

\begin{eqnarray}
	2ab &=& \gamma C^2 q^k \sum^k_{s=0} q^{k-s} 2\norm{ \itr{\vx_i}{0} } \norm{ \ssgrad{i}{\vz}{\xi}{s} - \nabla  f_i(\itr{\vz}{s}_i) } \notag\\
	&\overset{(\ref{formula})}{\leq}&\gamma C^2 q^k \sum^k_{j=0} q^{k-j} \norm{ \ssgrad{i}{\vz}{\xi}{j} - \nabla f_i(\itr{\vz}{j}_i) }^2 + \gamma C^2 q^k \sum^k_{j=0} q^{k-j} \norm{ \itr{\vx_i}{0} }^2 \notag\\
	&\overset{(\ref{geometric})}{\leq}& \gamma C^2 q^k \sum^k_{j=0} q^{k-j} \norm{ \ssgrad{i}{\vz}{\xi}{j} - \nabla f_i(\itr{\vz}{j}_i) }^2 + \frac{\gamma C^2 \norm{ \itr{x_i}{0} }^2}{1 - q} q^{k}
\end{eqnarray}
By similar procedure,
\begin{eqnarray}
	2ac \leq& \gamma C^2 q^k \sum^k_{s=0} q^{k-s} \norm{ \nabla f_i(\itr{z}{s}_i) }^2+\frac{\gamma C^2 \norm{ \itr{x_i}{0} }^2}{1-q} q^k 
	\end{eqnarray}
	Finally,
	\begin{eqnarray}
		2bc &=& \gamma^2 C^2 \sum^k_{j= 0}\sum^k_{s= 0} q^{2k-j-s} 2\norm{ \ssgrad{i}{z}{\xi}{j} - \nabla  f_i(\itr{z}{j}_i) } \norm{ \nabla f_i(\itr{z}{s}_i) }\notag\\
		&\overset{(\ref{formula})}{\leq}& \gamma^2 C^2 \sum^k_{j=0} \sum^k_{s=0} q^{2k-j-s} \norm{ \ssgrad{i}{z}{\xi}{j} - \nabla  f_i(\itr{z}{j}_i) }^2 + \gamma^2 C^2 \sum^k_{j=0} \sum^k_{s=0} q^{2k-j-s} \norm{ \nabla  f_i(\itr{z}{s}_i) }^2, \notag\\
	&=& \gamma^2 C^2 \sum^k_{j=0} q^{k-j} \norm{ \ssgrad{i}{z}{\xi}{j} - \nabla  f_i(\itr{z}{j}_i) }^2 \sum^k_{s=0}q^{k-s}+ \gamma^2 C^2 \sum^k_{s=0} q^{k-s} \norm{ \nabla  f_i(\itr{z}{s}_i) }^2 \sum^k_{j=0} q^{k-j} , \notag\\
	&\overset{(\ref{geometric})}{\leq}& \frac{\gamma^2 C^2}{1-q} \sum^k_{j=0} q^{k-j} \norm{ \ssgrad{i}{z}{\xi}{j} - \nabla f_i(\itr{z}{j}_i) }^2 + \frac{\gamma^2 C^2}{1-q} \sum^k_{s=0} q^{k-s} \norm{ \nabla f_i(\itr{z}{s}_i) }^2
	\end{eqnarray}

By combining all of the above bounds together we obtain:
\begin{eqnarray}
	Q^{(k)}_i &\leq& \E(a^2 +b^2 + c^2 + 2ab + 2bc + 2ac)\notag\\
&\leq& \E \frac{4 \gamma^2 C^2}{1 -q} \sum^k_{j=0} q^{k-j} \norm{ \ssgrad{i}{z}{\xi}{j} - \nabla f_i(\itr{z}{j}_i) }^2 \notag\\
	& + &\E \frac{4 \gamma^2 C^2}{1 - q} \sum^k_{j=0} q^{k-j} \norm{ \nabla f_i(\itr{z}{j}_i) }^2 \notag\\
	& +& C^2 \norm{ \itr{x_i}{0} }^2 q^{2k} \notag\\
	& +& \frac{2 \gamma C^2 \norm{ \itr{x_i}{0} }^2}{1-q} q^k \notag\\
	& + &\E \gamma C^2 q^k \sum^k_{j=0} q^{k-j} \norm{ \nabla f_i(\itr{z}{j}_i) }^2 \notag\\
	& +  &\E \gamma C^2 q^k \sum^k_{j=0} q^{k-j} \norm{ \ssgrad{i}{z}{\xi}{j} - \nabla  f_i(\itr{z}{j}_i) }^2.
\end{eqnarray}

After grouping terms together and using the upper bound of Lemma~\ref{boundonnormGradient}, we obtain
\begin{eqnarray}
	Q^{(k)}_i &\leq & \left( \gamma^2 \frac{4 C^2}{(1-q)^2} + \gamma \frac{q^k C^2}{1-q} \right) \sigma^2  +  \left(q^{2k} C^2 + \gamma q^k \frac{2 C^2}{1-q} \right)  \norm{ \itr{x_i}{0} }^2.\notag\\
	&+& \left( \gamma^2 \frac{4 C^2}{1 - q} + \gamma q^k C^2 \right) \sum^k_{j=0} q^{k-j} \E \norm{ \nabla  f_i(\itr{z}{j}_i) }^2 \notag\\
&\overset{Lemma~\ref{boundonnormGradient}}{\leq} & \left( \gamma^2 \frac{4 C^2}{(1-q)^2} + \gamma \frac{q^k C^2}{1-q} \right) \sigma^2 	 + \left(q^{2k} C^2 + \gamma q^k \frac{2 C^2}{1-q} \right)  \norm{ \itr{x_i}{0} }^2 \notag\\
&+&\left( \gamma^2 \frac{12 C^2}{(1 - q)^2} + \frac{\gamma q^k 3 C^2}{1 - q} \right) \zeta^2 \notag\\
	& + &\left( \gamma^2 \frac{12 L^2 C^2}{1 - q} + \gamma q^k 3 L^2 C^2 \right) \sum^k_{j=0} q^{k-j} \itr{Q}{j}_i \notag\\
	& + &\left( \gamma^2 \frac{12 C^2}{1 - q} + \gamma q^k 3 C^2 \right) \sum^k_{j=0} q^{k-j} \E \norm{ \nabla f(\itr{\xbar}{j}) }^2
\end{eqnarray}
This completes the proof.
\end{proof}

Having found a bound for the quantity $Q^{(k)}_i $, let us now present a lemma for bounding the quantity $\sum_{k=0}^{K-1} M^{(k)}$ where $K>1$ is a constant and $M^{(k)}$ is the average $Q^{(k)}_i $across all (non-virtual) nodes $i \in [n]$. That is, $M^{(k)}=\frac{1}{n} \sum^{n}_{i=1} \itr{Q}{k}_i $.

\begin{lem}
\label{BoundMk}
Let Assumptions 1-3 hold and let us define $D_2=1 - \dfrac{\gamma^2 12 L^2 C^2}{(1-q)^2} - \dfrac{\gamma 3 L^2 C^2}{(1-q)^2}$ .
Then,
\begin{eqnarray}
\sum^{K-1}_{k=0} \itr{M}{k} & \leq & \left( \gamma^2 \frac{4C^2}{(1-q)^2 D_2} \right) \sigma^2 K  + \left( \gamma \frac{C^2}{(1-q)^2D_2} \right) \sigma^2 \notag\\
	&& + \left( \gamma^2 \frac{12 C^2}{(1 - q)^2D_2} \right) \zeta^2 K + \left( \frac{\gamma 3 C^2}{(1 - q)^2 D_2} \right) \zeta^2 \notag\\
	&& + \left( \frac{ C^2}{(1-q)^2D_2} + \gamma \frac{2C^2}{(1-q)^2D_2} \right) \frac{\sum^n_{i=1} \norm{ \itr{x_i}{0} }^2}{n} \notag\\
	&& + \left( \gamma^2 \frac{12 C^2}{(1 - q)^2 D_2} + \gamma \frac{3 C^2}{(1-q)^2 D_2} \right) \sum^{K-1}_{k=0} \E \norm{ \nabla f(\itr{\xbar}{k}) }^2  
\end{eqnarray}
\end{lem}

\begin{proof}
Using the bound for $\itr{Q}{k}_i$ let us first bound its average across all nodes $\itr{M}{k}$
\begin{eqnarray}
	\itr{M}{k} &=& \frac{1}{n} \sum^n_{i=1} \itr{Q}{k}_i \notag\\
	&\overset{Lemma~\ref{BoundOfQ}}{\leq}& \left( \gamma^2 \frac{4 C^2}{(1-q)^2} + \gamma \frac{q^k C^2}{1-q} \right) \sigma^2 + \left( \gamma^2 \frac{12 C^2}{(1 - q)^2} + \frac{\gamma q^k 3 C^2}{1 - q}  \right) \zeta^2 \notag\\
	& + &\left( \gamma^2 \frac{12 C^2}{1 - q} + \gamma q^k 3 C^2 \right) \sum^k_{j=0} q^{k-j} \E \norm{ \nabla f(\itr{\xbar}{j}) }^2 \notag\\
	& +& \left( \gamma^2 \frac{12 L^2 C^2}{1 - q} + \gamma q^k 3 L^2 C^2 \right) \sum^k_{j=0} q^{k-j} \itr{M}{j} \notag\\
	& + &\left(q^{2k} C^2 + \gamma q^k \frac{2 C^2}{1-q} \right)  \frac{\sum^n_{i=1} \norm{ \itr{x_i}{0} }^2}{n}.
\end{eqnarray}

At this point note that for any $\lambda \in (0,1)$, non-negative integer $K \in \mathbb{N}$, and non-negative sequence $\{ \itr{\beta}{j} \}_{j=0}^k$, it holds that
\begin{eqnarray}
\label{naosnao}
\sum^K_{k=0} \sum^k_{j=0} \lambda^{k-j} \itr{\beta}{j}  &=& \itr{\beta}{0} \left( \lambda^K + \lambda^{K-1} + \cdots + \lambda^0 \right) +\itr{\beta}{1} \left( \lambda^{K-1} + \lambda^{K - 2} + \cdots + \lambda^{0} \right) +\cdots+ \itr{\beta}{K} \left( \lambda^0 \right) \notag\\
&\leq & \frac{1}{1-\lambda} \sum^K_{j=0} \itr{\beta}{j}. 
\end{eqnarray}
Similarly,
\begin{eqnarray}
\label{ansjka}
\sum^K_{k=0} \lambda^k \sum^k_{j=0} \lambda^{k-j} \itr{\beta}{j} = \sum^K_{k=0} \sum^k_{j=0} \lambda^{2k-j} \itr{\beta}{j} \leq \sum^K_{k=0} \sum^k_{j=0} \lambda^{2(k-j)} \itr{\beta}{j} &\overset{(\ref{naosnao})}{\leq}& \frac{1}{1-\lambda^2} \sum^K_{j=0} \itr{\beta}{j}
\end{eqnarray}

Now by summing from $k=0$ to $K-1$ and using the bounds of (\ref{naosnao}) and (\ref{ansjka}) we obtain:
\begin{align*}
	\sum^{K-1}_{k=0} \itr{M}{k} \leq& \left( \gamma^2 \frac{4C^2}{(1-q)^2} \right) \sigma^2 K  + \left( \gamma \frac{C^2}{(1-q)^2} \right) \sigma^2 \\
	& + \left( \gamma^2 \frac{12 C^2}{(1 - q)^2} \right) \zeta^2 K + \left( \frac{\gamma 3 C^2}{1 - q} \right) \zeta^2 \\
	& + \left( \frac{ C^2}{1-q^2} + \gamma \frac{2C^2}{(1-q)^2} \right) \frac{\sum^n_{i=1} \norm{ \itr{x_i}{0} }^2}{n} \\
	& + \left( \gamma^2 \frac{12 C^2}{(1 - q)^2} + \gamma \frac{3 C^2}{1-q^2} \right) \sum^{K-1}_{k=0} \E \norm{ \nabla f(\itr{\xbar}{k}) }^2  \\
	& + \left( \gamma^2 \frac{12 L^2 C^2}{(1 - q)^2} + \gamma \frac{3 L^2 C^2}{1-q^2} \right) \sum^{K-1}_{k=0} \itr{M}{k}.
\end{align*}

By rearranging:
\begin{align*}
	\left( 1-\gamma^2 \frac{12 L^2 C^2}{(1 - q)^2} - \gamma \frac{3 L^2 C^2}{1-q^2} \right) \sum^{K-1}_{k=0} \itr{M}{k} \leq& \left( \gamma^2 \frac{4C^2}{(1-q)^2} \right) \sigma^2 K  + \left( \gamma \frac{C^2}{(1-q)^2} \right) \sigma^2 \\
	& + \left( \gamma^2 \frac{12 C^2}{(1 - q)^2} \right) \zeta^2 K + \left( \frac{\gamma 3 C^2}{(1 - q)^2} \right) \zeta^2 \\
	& + \left( \frac{ C^2}{1-q^2} + \gamma \frac{2C^2}{(1-q)^2} \right) \frac{\sum^n_{i=1} \norm{ \itr{x_i}{0} }^2}{n} \\
	& + \left( \gamma^2 \frac{12 C^2}{(1 - q)^2} + \gamma \frac{3 C^2}{1-q^2} \right) \sum^{K-1}_{k=0} \E \norm{ \nabla f(\itr{\xbar}{k}) }^2  \\
\end{align*}

Note that since $q\in (0,1)$ it holds that $\frac{1}{1-q^2}\leq \frac{1}{(1-q)^2}$.\footnote{This step is used to simplified the expressions involve the parameter $q$. One can still obtain similar results by keeping the expression $\frac{1}{1-q^2}$ in the definition of $D_2$.}
Thus,

\begin{align*}
	\left( 1-\gamma^2 \frac{12 L^2 C^2}{(1 - q)^2} - \gamma \frac{3 L^2 C^2}{(1-q)^2} \right) \sum^{K-1}_{k=0} \itr{M}{k} \leq& \left( \gamma^2 \frac{4C^2}{(1-q)^2} \right) \sigma^2 K  + \left( \gamma \frac{C^2}{(1-q)^2} \right) \sigma^2 \\
	& + \left( \gamma^2 \frac{12 C^2}{(1 - q)^2} \right) \zeta^2 K + \left( \frac{\gamma 3 C^2}{(1 - q)^2} \right) \zeta^2 \\
	& + \left( \frac{ C^2}{(1-q)^2} + \gamma \frac{2C^2}{(1-q)^2} \right) \frac{\sum^n_{i=1} \norm{ \itr{x_i}{0} }^2}{n} \\
	& + \left( \gamma^2 \frac{12 C^2}{(1 - q)^2} + \gamma \frac{3 C^2}{(1-q)^2} \right) \sum^{K-1}_{k=0} \E \norm{ \nabla f(\itr{\xbar}{k}) }^2  \\
\end{align*}

Dividing both sides with $D_2=1 - \dfrac{\gamma^2 12 L^2 C^2}{(1-q)^2} - \dfrac{\gamma 3 L^2 C^2}{(1-q)^2}$ completes the proof.
\end{proof}

\subsection{Towards the proof of the main Theorems}
\label{ImportLemma}
The goal of this section is the presentation of Lemma~\ref{MainLemma}. It is the main lemma of our convergence analysis and based on which we build the proofs of Theorems~\ref{thm:avg-convergence} and \ref{secondTheorem}.

Let us first state a preliminary lemma that simplifies some of the expressions that involve expectations with respect to the random variable ${\xi_i^{(t)}}$.

\begin{lem}
\label{FirstLemma}
Under the definition of our problem and the Assumptions 1-3 we have that:
\begin{itemize}
\item[(i)] 
$$\Exp_{\xi_i^{(k)}} \norm{ \frac{ \sum^n_{i=1} \ssgrad{i}{z}{\xi}{k} }{n} }^2 =\Exp_{\xi_i^{(k)}}\norm{ \frac{\sum^n_{i=1} \ssgrad{i}{z}{\xi}{k} - \nabla f_i(\itr{z}{k}_i) }{n} }^2 + \Exp_{\xi_i^{(k)}}  \norm{\frac{\sum^n_{i=1} \nabla f_i(\itr{z}{k}_i)}{n} }^2 $$
\item[(ii)] 
$$\Exp_{\xi_i^{(k)}}\norm{ \frac{\sum^n_{i=1} \left[\ssgrad{i}{z}{\xi}{k} - \nabla f_i(\itr{z}{k}_i) \right]}{n} }^2 \leq \frac{\sigma^2}{n} $$
\end{itemize}
\end{lem}

\begin{proof}
\begin{eqnarray}
	\Exp_{\xi_i^{(k)}} \norm{ \frac{ \sum^n_{i=1} \ssgrad{i}{z}{\xi}{k} }{n} }^2 &=&\ \Exp_{\xi_i^{(k)}} \norm{ \frac{\sum^n_{i=1} \ssgrad{i}{z}{\xi}{k} - \nabla  f_i(\itr{z}{k}_i) }{n} + \frac{\sum^n_{i=1} \nabla f_i(\itr{z}{k}_i )}{n} }^2 \notag\\
	&=&\  \Exp_{\xi_i^{(k)}} \norm{ \frac{\sum^n_{i=1} \ssgrad{i}{z}{\xi}{k} - \nabla f_i(\itr{z}{k}_i) }{n} }^2 \notag\\
	&&+ \Exp_{\xi_i^{(k)}} \norm{\frac{\sum^n_{i=1} \nabla  f_i(\itr{z}{k}_i)}{n} }^2 \notag\\
	&&+ 2 \left \langle \frac{\sum^n_{i=1} \E_{\xi^{(k)}_i} \ssgrad{i}{z}{\xi}{k} - \nabla  f_i(\itr{z}{k}_i)}{n} \ , \frac{\sum^n_{i=1} \nabla f_i(\itr{z}{k}_i)}{n} \right \rangle \notag\\
	&=&\  \Exp_{\xi_i^{(k)}} \norm{ \frac{\sum^n_{i=1} \ssgrad{i}{z}{\xi}{k} - \nabla  f_i(\itr{z}{k}_i) }{n} }^2 \notag\\
	&&+ \Exp_{\xi_i^{(k)}} \norm{\frac{\sum^n_{i=1} \nabla  f_i(\itr{z}{k}_i)}{n} }^2.
\end{eqnarray}
where in the last equality the inner product becomes zero from the fact that $\E_{\xi^{(k)}_i} \ssgrad{i}{z}{\xi}{k}=\nabla  f_i(\itr{z}{k}_i).$

\begin{align}
\Exp_{\xi_i^{(k)}} &\norm{ \frac{\sum_{i=1}^n \ssgrad{i}{z}{\xi}{k}-\sum_{i=1}^n \nabla f_i(\itr{z}{k}_i)}{n}}^2 \notag\\
&=\frac{1}{n^2} \Exp_{\xi_i^{(k)}} \norm{\sum_{i=1}^n \left[ \ssgrad{i}{z}{\xi}{k}-\nabla f_i(\itr{z}{k}_i)\right]}^2 \notag\\
&= \frac{1}{n^2} \sum_{i=1}^n \Exp_{\xi_i^{(k)}} \norm{ \ssgrad{i}{z}{\xi}{k}-\nabla f_i(\itr{z}{k}_i)}^2 \notag\\ 
&\quad+ \frac{2}{n^2}\sum_{i\neq j} \left\langle \Exp_{\xi_i^{(k)}}\ssgrad{i}{z}{\xi}{k}- \nabla f_i(\itr{z}{k}_i), \Exp_{\xi_j^{(k)}}\ssgrad{j}{z}{\xi}{k}-\nabla f_j(\itr{z}{k}_j) \right\rangle  \notag\\
&= \frac{1}{n^2} \sum_{i=1}^n \Exp_{\xi_i^{(k)}} \norm{ \ssgrad{i}{z}{\xi}{k}-\nabla f_i(\itr{z}{k}_i)}^2 \notag\\ 
&\overset{\text{\text{Bounded Variance}}}{\leq}  \frac{1}{n^2} \sum_{i=1}^n \sigma^2 = \frac{\sigma^2 }{n},
\end{align}
\end{proof}

Before presenting the proof of next lemma let us define the conditional expectation
\[
    \Exp[\cdot | \cF_k] \defeq \Exp_{\xi_i^{(k)} \sim \cD_i \forall i \in [n]}[\cdot] = \Exp_{\xi_i^{(k)} \forall i \in [n]}[\cdot].
\]
The expectation in this expression is \emph{only} with respect to the random choices $\xi_{i}^{(k)}$ for all nodes $i \in [n]$ at the $k^{th}$ iteration.
In addition, we should highlight that the choices of random variables $\xi_i^k \sim \cD_i $, $\xi_j^k \sim \cD_j$ at the step $t$ of the algorithm, are independent for any two nodes $i \neq j \in [n]$. 
This is also true in the case that the two nodes follow the same distribution $\cD=\cD_i =\cD_j$.

\begin{lem}
\label{MainLemma}
Let Assumptions 1-3 hold and let $$D_1=\frac{1}{2}- \frac{L^2}{2}  \left( \frac{12\gamma^2 C^2 +3 \gamma C^2}{(1-q)^2 D_2} \right) \quad \text{and} \quad D_2=1 - \dfrac{\gamma^2 12 L^2 C^2}{(1-q)^2} - \dfrac{\gamma 3 L^2 C^2}{(1-q)^2}.$$ Here $C>0$ and $q\in (0,1)$ are the two non-negative constants defined in Lemma~\ref{MidosTheorem}. Let $\{\bX_k\}_{k=0}^\infty$ be the random sequence produced by (\ref{UpdateRule}) (Matrix representation of Algorithm~\ref{SGPalg}). Then,
\begin{align*}
	  &\frac{1}{K} \left( D_1 \sum^{K-1}_{k=0} \E \norm{ \nabla f(\itr{\xbar}{k}) }^2  + \frac{1-L \gamma}{2} \sum_{k=0}^{K-1} \Exp \left\|  \frac{\nabla \bm{F}(\bZ^{(k)}) \1_{\nbar}  }{n} \right\|^2 \right)\\
	\leq& \frac{f( \itr{\xbar}{0} ) - f^*}{\gamma K}+ \frac{L \gamma \sigma^2}{2 n} + \frac{4  L^2\gamma^2 C^2 \sigma^2+12L^2\gamma^2 C^2 \zeta^2}{2 (1-q)^2 D_2} +\frac{\gamma L^2 C^2 \sigma^2+3  L^2\gamma C^2\zeta^2}{2K (1-q)^2 D_2}\\
	& + \left( \frac{ L^2C + 2 L^2\gamma C^2}{2 (1-q)^2 D_2K} \right) \frac{\sum^n_{i=1} \norm{ \itr{\vx_i}{0} }^2}{n}  .
\end{align*}
\end{lem}
\begin{proof}
\begin{eqnarray}\label{main_equality}
f\left(\itr{\xbar}{k+1}\right) =f\left(\frac{\bX^{(k+1)}\1_{\nbar} }{n}\right) &\overset{(\ref{UpdateRule})}{=}& f\left(\frac{\bX^{(k)} [\bP^{(k)}]^\top \1_{\nbar}  - \gamma \nabla \bm{F}(\bZ^{(k)}, \bm{\xi}^{(k)}) [\bP^{(k)}]^\top \1_{\nbar}  }{n}\right) \notag\\
&=& f\left(\frac{\bX^{(k)} \1_{\nbar}}{n}  - \frac{\gamma \nabla \bm{F}(\bZ^{(k)}, \bm{\xi}^k) \1_{\nbar}  }{n}\right)\notag\\
&\overset{L-smooth}{\leq}& f\left(\frac{\bX^{(k)}\1_{\nbar}  }{n}\right)  - \gamma \left\langle \nabla  f\left(\frac{\bX^{(k)}\1_{\nbar}  }{n}\right),  \frac{\nabla \bm{F}(\bZ^{(k)}, \bm{\xi}^{(k)}) \1_{\nbar}  }{n} \right\rangle \notag\\
&&+ \frac{L \gamma^2}{2} \norm{\frac{\nabla \bm{F}(\bZ^{(k)}, \bm{\xi}^{(k)}) \1_{\nbar} }{n}}^2 
\end{eqnarray}
Taking expectations of both sides conditioned on $\cF_k$:
\begin{eqnarray}
\Exp \left[ f\left(\frac{\bX^{(k+1)}\1_{\nbar}  }{n}\right) | \cF_k \right] &\leq & f\left(\frac{\bX^{(k)}\1_{\nbar}  }{n}\right)  - \gamma \left\langle \nabla  f\left(\frac{\bX^{(k)}\1_{\nbar}  }{n}\right),  \frac{\nabla \bm{F}(\bZ^{(k)}) \1_{\nbar}  }{n} \right\rangle \notag\\
&&+ \frac{L \gamma^2}{2} \Exp \left[\norm{ \frac{\nabla \bm{F}(\bZ^{(k)}, \bm{\xi}^{(k)}) \1_{\nbar} }{n}}^2 | \cF_k \right]\notag\\
&\overset{Lemma~\ref{FirstLemma}[i]}{=} & f\left(\frac{\bX^{(k)}\1_{\nbar}  }{n}\right)  - \gamma \left\langle \nabla  f\left(\frac{\bX^{(k)}\1_{\nbar}  }{n}\right),  \frac{\nabla \bm{F}(\bZ^{(k)}) \1_{\nbar}  }{n} \right\rangle  \notag\\
&& + \frac{L \gamma^2}{2}  \Exp \left[\norm{ \frac{\sum_{i=1}^n \ssgrad{i}{z}{\xi}{k}-\sum_{i=1}^n \nabla f_i(\itr{z}{k}_i)}{n}}^2 | \cF_k \right] \notag\\
&&+ \frac{L \gamma^2}{2}  \Exp \left[\norm{ \frac{\sum_{i=1}^n \nabla f_i(\itr{\vz}{k}_i)}{n}}^2 | \cF_k \right]\notag\\
&\overset{Lemma~\ref{FirstLemma}[ii]}{\leq} & f\left(\frac{\bX^{(k)}\1_{\nbar}  }{n}\right)  - \gamma \left\langle \nabla  f\left(\frac{\bX^{(k)}\1_{\nbar}  }{n}\right),  \frac{\nabla \bm{F}(\bZ^{(k)}) \1_{\nbar}  }{n} \right\rangle  \notag\\
&&+  \frac{L \gamma^2 \sigma}{2n} +  \frac{L \gamma^2}{2}  \Exp\left[\norm{ \frac{\sum_{i=1}^n \nabla f_i(\itr{\vz}{k}_i)}{n}}^2 | \cF_k \right]\notag\\
&= & f\left(\frac{\bX^{(k)}\1_{\nbar}  }{n}\right)  - \frac{\gamma}{2} \left\| \nabla  f\left(\frac{\bX^{(k)}\1_{\nbar}  }{n}\right) \right\|^2 - \frac{\gamma}{2}\norm{\frac{\nabla \bm{F}(\bZ^{(k)}) \1_{\nbar}  }{n}}^2 ,  \notag\\
&&+\frac{\gamma}{2}  \norm{ \nabla  f\left(\frac{\bX^{(k)}\1_{\nbar}  }{n}\right) - \frac{\nabla \bm{F}(\bZ^{(k)}) \1_{\nbar}  }{n}}^2 + \frac{L \gamma^2 \sigma^2}{2 n} \notag\\
&&+ \frac{L \gamma^2}{2}  \Exp\left[\norm{\frac{\sum_{i=1}^n \nabla f_i(\itr{z}{k}_i)}{n}}^2 | \cF_k \right]
\end{eqnarray}
where in the last step above we simply expand the inner product.

Taking expectations with respect to $\cF_k$ and using the tower property, we get

\begin{eqnarray}
\Exp\left[ f\left(\frac{\bX^{(k+1)}\1_{\nbar}  }{n}\right)\right] &\leq & \Exp\left[f\left(\frac{\bX^{(k)}\1_{\nbar}  }{n}\right)\right]  - \frac{\gamma}{2} \Exp\left[\norm{\nabla  f\left(\frac{\bX^{(k)}\1_{\nbar}  }{n}\right)}^2\right] - \frac{\gamma}{2} \Exp\left[\left\|  \frac{\nabla \bm{F}(\bZ^{(k)}) \1_{\nbar}  }{n} \right\|^2\right] ,  \notag\\
&+ &\frac{\gamma}{2}  \Exp\left[\norm{\nabla  f\left(\frac{\bX^{(k)}\1_{\nbar}  }{n}\right) - \frac{\nabla \bm{F}(\bZ^{(k)}) \1_{\nbar}  }{n}}^2\right] + \frac{L \gamma^2 \sigma^2}{2 n} \notag\\
&+ &  \frac{L \gamma^2}{2}  \Exp \left[\norm{ \frac{\sum_{i=1}^n \nabla f_i(\itr{\vz}{k}_i)}{n}}^2 \right]\notag\\
&=& \Exp\left[f\left(\frac{\bX^{(k)}\1_{\nbar}  }{n}\right)\right]  - \frac{\gamma}{2} \Exp\left[\left\| \nabla  f\left(\frac{\bX^{(k)}\1_{\nbar}  }{n}\right) \right\|^2\right] - \frac{\gamma-L \gamma^2}{2} \Exp \left[\norm{ \frac{\nabla \bm{F}(\bZ^{(k)}) \1_{\nbar}  }{n}}^2 \right] ,  \notag\\
&+ &\frac{\gamma}{2}  \Exp\left[\norm{\nabla  f\left(\frac{\bX^{(k)}\1_{\nbar}  }{n}\right) - \frac{\nabla \bm{F}(\bZ^{(k)}) \1_{\nbar}  }{n}}^2\right] + \frac{L \gamma^2 \sigma^2}{2 n}
\end{eqnarray}

Let us now focus on find an upper bound for the quantity $\Exp\left[\left\|  \nabla  f\left(\frac{\bX^{(k)}\1_{\nbar}  }{n}\right) - \frac{\nabla \bm{F}(\bZ^{(k)}) \1_{\nbar}  }{n} \right\|^2\right] $.

\begin{eqnarray}
\Exp\left[\left\|  \nabla  f\left(\frac{\bX^{(k)}\1_{\nbar}  }{n}\right) - \frac{\nabla \bm{F}(\bZ^{(k)}) \1_{\nbar}  }{n} \right\|^2\right] &= & \Exp\left[\left\|  \nabla  f\left(\xbar \right) - \frac{\sum_{i=1}^n \nabla f_i(\itr{\vz}{k}_i) }{n} \right\|^2\right] \notag\\
&=& \Exp\left[\left\|  \frac{1}{n}\sum_i^n  \nabla f_i \left(\xbar \right) - \frac{\sum_{i=1}^n \nabla f_i(\itr{\vz}{k}_i)  }{n} \right\|^2\right] \notag\\
&=& \Exp\left[\left\|  \frac{\sum_i^n  \nabla f_i \left(\xbar \right) -\sum_{i=1}^n \nabla f_i(\itr{\vz}{k}_i) }{n} \right\|^2\right] \notag\\
&=&\Exp\left[\left\|\frac{1}{n} \sum_i^n  \left[\nabla f_i \left(\xbar \right) - \nabla f_i(\itr{\vz}{k}_i) \right] \right\|^2\right] \notag\\
&\overset{Jensen}{\leq} & \frac{1}{n} \sum_i^n  \Exp\left[\left\| \nabla f_i \left(\xbar \right) - \nabla f_i(\itr{\vz}{k}_i)\right\|^2\right] \notag\\
&\overset{L-smooth}{\leq} & \frac{L^2}{n} \sum_{i=1}^n  \Exp\left[\left\| \xbar - \itr{\vz}{k}_i \right\|^2\right] \notag\\
&= & \frac{L^2}{n} \sum_{i=1}^n  Q^{(k)}_i  
\end{eqnarray}

Thus we have that:
\begin{eqnarray}
\Exp\left[ f\left(\frac{\bX^{(k+1)}\1_{\nbar}  }{n}\right)\right] &\leq & \Exp\left[f\left(\frac{\bX^{(k)}\1_{\nbar}  }{n}\right)\right]  - \frac{\gamma}{2} \Exp\left[\left\| \nabla  f\left(\frac{\bX^{(k)}\1_{\nbar}  }{n}\right) \right\|^2\right] - \frac{\gamma-L \gamma^2}{2} \Exp\left[\left\|  \frac{\nabla \bm{F}(\bZ^{(k)}) \1_{\nbar}  }{n} \right\|^2\right] ,  \notag\\
&+ &\frac{\gamma L^2}{2 n} \sum_{i=1}^n  Q^{(k)}_i  + \frac{L \gamma^2 \sigma^2}{2 n} 
\end{eqnarray}

By rearranging:
\begin{eqnarray}
\frac{\gamma}{2} \Exp \left[\left\| \nabla  f\left(\frac{\bX^{(k)}\1_{\nbar}  }{n}\right) \right\|^2 \right] + \frac{\gamma-L \gamma^2}{2} \Exp \left[\left\|  \frac{\nabla \bm{F}(\bZ^{(k)}) \1_{\nbar}  }{n} \right\|^2 \right] &\leq & \Exp \left[f\left(\frac{\bX^{(k)}\1_{\nbar}  }{n}\right) \right] - \Exp \left[ f\left(\frac{\bX^{(k+1)}\1_{\nbar}  }{n}\right) \right]  \notag\\
&+ &\frac{L \gamma^2 \sigma^2}{2 n}+\frac{\gamma L^2}{2 n} \sum_{i=1}^n Q^{(k)}_i 
\end{eqnarray}

Let us now sum from $k=0 $ to $k=K-1$:

\begin{eqnarray}
\frac{\gamma}{2} \sum_{k=0}^{K-1}\Exp \left[\left\| \nabla  f\left(\frac{\bX^{(k)}\1_{\nbar}  }{n}\right) \right\|^2 \right] + \frac{\gamma-L \gamma^2}{2} \sum_{k=0}^{K-1} \Exp \left[\left\|  \frac{\nabla \bm{F}(\bZ^{(k)}) \1_{\nbar}  }{n} \right\|^2 \right] &\leq & \sum_{k=0}^{K-1} \left[ \Exp \left[f\left(\frac{\bX^{(k)}\1_{\nbar}  }{n}\right) \right] - \Exp \left[ f\left(\frac{\bX^{(k+1)}\1_{\nbar} }{n}\right) \right] \right] \notag\\
&+ & \sum_{k=0}^{K-1} \frac{L \gamma^2 \sigma^2}{2 n}+\frac{\gamma L^2}{2 n} \sum_{k=0}^{K-1} \sum_{i=1}^n  Q^{(k)}_i  \notag\\
&\leq & \Exp \left[f\left(\frac{\bX^{(0)}\1_{\nbar}  }{n}\right) \right] - \Exp \left[ f\left(\frac{\bX^{(k)}\1_{\nbar} }{n}\right) \right] \notag\\
&+ &\frac{L K \gamma^2 \sigma^2}{2 n} +\frac{\gamma L^2}{2} \sum_{k=0}^{K-1}\frac{1}{n}\sum_{i=1}^n Q^{(k)}_i \notag\\
&\leq & f( \itr{\xbar}{0} ) - f^* \notag\\
&+ &\frac{L K \gamma^2 \sigma^2}{2 n} +\frac{\gamma L^2}{2} \sum_{k=0}^{K-1}  \underbrace{\frac{1}{n}\sum_{i=1}^n Q^{(k)}_i}_{M_k}
\end{eqnarray}
For the last inequality above, recall that we let $f^*$ denote the global infimum of our problem.

Using the bound for the expression $\sum_{k=0}^{K-1}M_k$ from Lemma~\ref{BoundMk} we obtain:

\begin{align*}
	  &\frac{\gamma}{2} \sum_{k=0}^{K-1}\Exp \left[\left\| \nabla  f\left(\frac{\bX^{(k)}\1_{\nbar}  }{n}\right) \right\|^2 \right] + \frac{\gamma-L \gamma^2}{2} \sum_{k=0}^{K-1} \Exp \left[\left\|  \frac{\nabla \bm{F}(\bZ^{(k)}) \1_{\nbar}  }{n} \right\|^2 \right] \\
	\leq& f( \itr{\xbar}{0} ) - f^*+ \frac{L K \gamma^2 \sigma^2}{2 n} \\
	& + \frac{\gamma L^2}{2} \frac{4 \gamma^2 C^2 \sigma^2 K+\gamma C^2 \sigma^2}{(1-q)^2 D_2} + \frac{\gamma L^2}{2} \frac{12\gamma^2 C^2 \zeta^2 K+3 \gamma C^2\zeta^2}{(1-q)^2 D_2}\\
	& + \frac{\gamma L^2}{2}  \left( \frac{12\gamma^2 C^2 +3 \gamma C^2}{(1-q)^2 D_2} \right) \sum^K_{k=0} \E \norm{ \nabla f(\itr{\xbar}{k}) }^2  \\
	& + \frac{\gamma L^2}{2}  \left(\frac{ C^2 + 2\gamma C^2}{(1-q)^2 D_2} \right) \frac{\sum^n_{i=1} \norm{ \itr{\vx_i}{0} }^2}{n}  .
\end{align*}

By rearranging and dividing all terms by $\gamma K$ we obtain:
\begin{align*}
	  &\frac{1}{K} \left( \left[\frac{1}{2}- \frac{L^2}{2}  \left( \frac{12\gamma^2 C^2 +3 \gamma C^2}{(1-q)^2 D_2} \right) \right] \sum^{K-1}_{k=0} \E \norm{ \nabla f(\itr{\xbar}{k}) }^2  + \frac{1-L \gamma}{2} \sum_{k=0}^{K-1} \Exp \left\|  \frac{\nabla \bm{F}(\bZ^{(k)}) \1_{\nbar}  }{n} \right\|^2 \right)\\
	\leq& \frac{f( \itr{\xbar}{0} ) - f^*}{\gamma K}+ \frac{L \gamma \sigma^2}{2 n} + \frac{4  L^2\gamma^2 C^2 \sigma^2+12L^2\gamma^2 C^2 \zeta^2}{2 (1-q)^2 D_2} +\frac{\gamma L^2 C^2 \sigma^2+3  L^2\gamma C^2\zeta^2}{2K (1-q)^2 D_2}\\
	& + \left( \frac{ L^2C^2 + 2 L^2\gamma C^2}{2 (1-q)^2 D_2K} \right) \frac{\sum^n_{i=1} \norm{ \itr{\vx_i}{0} }^2}{n}  .
\end{align*}
By defining $D_1=\left[\frac{1}{2}- \frac{L^2}{2}  \left( \frac{12\gamma^2 C^2 +3 \gamma C^2}{(1-q)^2 D_2} \right) \right]$ the proof is complete.
\end{proof}

\subsection{Proofs of Main Theorems}
\label{ProofMain}
Having present all of the above Lemmas we are now ready to provide the proofs of main Theorems~\ref{thm:avg-convergence} and \ref{secondTheorem}.
\subsubsection{Proof of Theorem~\ref{thm:avg-convergence}}
Let $\gamma \leq \min\left\{\dfrac{(1-q)^2}{60L^2C^2},1 \right\}$. Then:
$$D_2=1 - \dfrac {\gamma^2 12 L^2 C^2}{(1-q)^2} - \dfrac{\gamma 3 L^2 C^2}{(1-q)^2}\overset{(\gamma^2<\gamma)}{\geq} 1 - \dfrac{\gamma 15 L^2 C^2}{(1-q)^2}\geq 1-\frac{1}{4}\geq\frac{1}{2}$$
and 
$$D_1=\frac{1}{2}- \frac{L^2}{2}  \left( \frac{12\gamma^2 C^2 +3 \gamma C^2}{(1-q)^2 D_2} \right)\overset{(\gamma^2<\gamma)}{\geq} \frac{1}{2}- \frac{15\gamma C^2L^2}{2(1-q)^2D_2}\geq \frac{1}{2}-\frac{1}{8 D_2}\geq \frac{1}{4}$$
By substituting the above bounds into the result of Lemma~\ref{MainLemma} and by removing the second term of left hand side we obtain:
\begin{eqnarray}
\frac{1}{4 }\frac{\sum^{K-1}_{k=0} \E \norm{ \nabla f(\itr{\xbar}{k}) }^2 }{K} &=&\frac{1}{K} \left( \frac{1}{4} \sum^{K-1}_{k=0} \E \norm{ \nabla f(\itr{\xbar}{k}) }^2  + \frac{1-L \gamma}{2} \sum_{k=0}^{K-1} \Exp \left\|  \frac{\nabla \bm{F}(\bZ_k) \1_{\nbar}  }{n} \right\|^2 \right)\notag\\
	&\leq& \frac{f( \itr{\xbar}{0} ) - f^*}{\gamma K}+ \frac{L \gamma \sigma^2}{2 n} + \frac{4  L^2\gamma^2 C^2 \sigma^2+12L^2\gamma^2 C^2 \zeta^2}{(1-q)^2 } +\frac{\gamma L^2 C^2 \sigma^2+3  L^2\gamma C^2\zeta^2}{K (1-q)^2}\notag\\
	& +& \left( \frac{ L^2C + 2 L^2\gamma C^2}{(1-q)^2 K} \right) \frac{\sum^n_{i=1} \norm{ \itr{\vx_i}{0} }^2}{n}
\end{eqnarray}
Let us now substitute in the above expression $\gamma=\sqrt{\frac{n}{K}}$. This can be done due to the lower bound (see \eqref{BoundK}) on the total number of iterations $K$ where guarantees that $\sqrt{\frac{n}{K}} \leq \min\left\{\dfrac{(1-q)^2}{60L^2C^2},1 \right\}$.

\begin{eqnarray}
\frac{1}{4 }\frac{\sum^{K-1}_{k=0} \E \norm{ \nabla f(\itr{\xbar}{k}) }^2 }{K} &\leq& \frac{f( \itr{\xbar}{0} ) - f^*}{\gamma K}+ \frac{L \gamma \sigma^2}{2 n} + \gamma^2 \frac{4  L^2 C^2 \sigma^2+12L^2 C^2 \zeta^2}{(1-q)^2 } +\gamma \frac{ L^2 C^2 \sigma^2+3  L^2 C^2\zeta^2}{K (1-q)^2}\notag\\
	& +& \frac{ L^2C}{(1-q)^2 K} \frac{\sum^n_{i=1} \norm{ \itr{\vx_i}{0} }^2}{n}+ \gamma \frac{2 L^2 C^2}{(1-q)^2 K} \frac{\sum^n_{i=1} \norm{ \itr{\vx_i}{0} }^2}{n}\notag\\
&\overset{\gamma=\sqrt{\frac{n}{K}}}{=}& \frac{f( \itr{\xbar}{0} ) - f^*}{\sqrt{n K}}+ \frac{L \sigma^2}{2 \sqrt{nK}} + \frac{n}{K} \frac{4  L^2 C^2 \sigma^2+12L^2 C^2 \zeta^2}{(1-q)^2 } +\sqrt{\frac{n}{K}} \frac{ L^2 C^2 \sigma^2+3  L^2 C^2\zeta^2}{K (1-q)^2}\notag\\
	& +& \frac{ L^2C^2}{(1-q)^2 K} \frac{\sum^n_{i=1} \norm{ \itr{\vx_i}{0} }^2}{n}+ \sqrt{\frac{n}{K}} \frac{2 L^2 C^2}{(1-q)^2 K} \frac{\sum^n_{i=1} \norm{ \itr{\vx_i}{0} }^2}{n}\notag\\
&=& \frac{f( \itr{\xbar}{0} ) - f^*+\frac{L}{2} \sigma^2}{\sqrt{n K}}+ \frac{L^2 C^2 }{K (1-q)^2} \left[(4\sigma^2+12\zeta^2) n+\frac{\sum^n_{i=1} \norm{ \itr{\vx_i}{0} }^2}{n}\right] \notag\\
&+&\frac{\sqrt{n}L^2 C^2}{\sqrt{K} (1-q)^2 K} \left[\sigma^2+3  L^2 C^2\zeta^2 + 2\frac{\sum^n_{i=1} \norm{ \itr{\vx_i}{0} }^2}{n}\right]
\end{eqnarray}
Using again the assumption on the lower bound (\ref{BoundK}) of the total number of iterations $K$, the last two terms of the above expression are bounded by the first term.
Thus,
\begin{eqnarray}
\label{bound2}
\frac{1}{4 }\frac{\sum^{K-1}_{k=0} \E \norm{ \nabla f(\itr{\xbar}{k}) }^2 }{K}  &\leq& 3\frac{f( \itr{\xbar}{0} ) - f^*+\frac{L}{2} \sigma^2}{\sqrt{n K}}
\end{eqnarray}

\subsubsection{Proof of Theorem~\ref{secondTheorem}}
\begin{proof}
From Lemma~\ref{BoundMk} we have that:
\begin{eqnarray}
\frac{1}{K} \sum^{K-1}_{k=0} \itr{M}{k} & \leq & \left( \gamma^2 \frac{4C^2}{(1-q)^2 D_2} \right) \sigma^2   + \left( \gamma \frac{C^2}{(1-q)^2D_2} \right) \frac{\sigma^2 }{K}\notag\\
	&& + \left( \gamma^2 \frac{12 C^2}{(1 - q)^2D_2} \right) \zeta^2  + \left( \frac{\gamma 3 C^2}{(1 - q)^2 D_2} \right) \frac{\zeta^2}{K} \notag\\
	&& + \left( \frac{ C^2}{(1-q)^2D_2 K} + \gamma \frac{2C^2}{(1-q)^2D_2 K} \right) \frac{\sum^n_{i=1} \norm{ \itr{\vx_i}{0} }^2}{n} \notag\\
	&& + \left( \gamma^2 \frac{12 C^2}{(1 - q)^2 D_2} + \gamma \frac{3 C^2}{(1-q)^2 D_2} \right) \frac{\sum^{K-1}_{k=0} \E \norm{ \nabla f(\itr{\xbar}{k}) }^2 }{K} 
\end{eqnarray}
Using the assumptions of Theorem~\ref{thm:avg-convergence} and stepsize $\gamma=\sqrt{\frac{n}{K}}$:
\begin{eqnarray}
\label{bound1}
\frac{1}{K} \sum^{K-1}_{k=0} \itr{M}{k} & \leq & \left( \frac{n}{K} \frac{4C^2}{(1-q)^2 D_2} \right) \sigma^2   + \left(\sqrt{\frac{n}{K}}  \frac{C^2}{(1-q)^2D_2} \right) \frac{\sigma^2 }{K}\notag\\
	&& + \left( \frac{n}{K}  \frac{12 C^2}{(1 - q)^2D_2} \right) \zeta^2  + \left( \frac{\sqrt{\frac{n}{K}} 3 C^2}{(1 - q)^2 D_2} \right) \frac{\zeta^2}{K} \notag\\
	&& + \left( \frac{ C^2}{(1-q)^2D_2 K} + \sqrt{\frac{n}{K}} \frac{2C^2}{(1-q)^2D_2 K} \right) \frac{\sum^n_{i=1} \norm{ \itr{\vx_i}{0} }^2}{n} \notag\\
	&& + \left( \frac{n}{K} \frac{12 C^2}{(1 - q)^2 D_2} + \sqrt{\frac{n}{K}} \frac{3 C^2}{(1-q)^2 D_2} \right) \frac{12\left[f( \itr{\xbar}{0} ) - f^*+\frac{L}{2} \sigma^2\right]}{\sqrt{n K}} \notag\\
	& = & \frac{1}{K}\left[  \frac{4n C^2  \sigma^2}{(1-q)^2 D_2} + \frac{12 n C^2  \zeta^2 }{(1 - q)^2D_2} + \frac{ C^2 \sum^n_{i=1} \norm{ \itr{\vx_i}{0} }^2}{n (1-q)^2D_2 } + \frac{3 \sqrt{n}  C^212\left[f( \itr{\xbar}{0} ) - f^*+\frac{L}{2} \sigma^2\right]}{\sqrt{n}(1-q)^2 D_2}\right]\notag\\
	&+&\frac{1}{K \sqrt{K}} \left[\frac{n \sigma^2 C^2}{(1-q)^2D_2}+\frac{\frac{n} 3 C^2 \zeta^2 }{(1 - q)^2 D_2} +\frac{2 C^2 \sum^n_{i=1} \norm{ \itr{\vx_i}{0} }^2}{(1-q)^2D_2 \sqrt{n}}+\frac{144 \sqrt{n} C^2 \left[f( \itr{\xbar}{0} ) - f^*+\frac{L}{2} \sigma^2\right] }{(1 - q)^2 D_2}\right]\notag\\
	&= &O \left(\frac{1}{K} +\frac{1}{K \sqrt{K}} \right)
\end{eqnarray}
where the Big O notation swallows all constants of our setting $\left (n,L, \sigma, \zeta, C , q, \sum^n_{i=1} \norm{ \itr{x_i}{0} }^2 \text{and} f( \itr{\xbar}{0} ) - f^* \right)$.

Now using the above upper bound \eqref{bound1} and result of Theorem~\ref{thm:avg-convergence} we obtain:
\begin{eqnarray}
\frac{1}{K}\sum_{k=0}^{K-1}\frac{1}{n} \sum^n_{i=1} \Exp \norm{\nabla f(\vz_i^k)}^2 & = & \frac{1}{K} \sum_{k=0}^{K-1}\frac{1}{n} \sum^n_{i=1} \Exp \norm{\nabla f(\vz_i^k)+\nabla f(\itr{\xbar}{k})-\nabla f(\itr{\xbar}{k})}^2 \notag\\
	&\leq &  \frac{1}{K} \sum_{k=0}^{K-1}\frac{1}{n} \sum^n_{i=1} 2 \Exp \norm{\nabla f(\vz_i^k) - \nabla f(\itr{\xbar}{k})}^2 + 2 \Exp \norm{\nabla f(\itr{\xbar}{k})}^2\notag\\
	&=&  \frac{1}{K} \sum_{k=0}^{K-1}\frac{1}{n} \sum^n_{i=1} 2 \Exp \norm{\nabla f(\vz_i^k) - \nabla f(\itr{\xbar}{k})}^2 + \frac{1}{K} \sum_{k=0}^{K-1}\frac{1}{n} \sum^n_{i=1} 2 \Exp \norm{\nabla f(\itr{\xbar}{k})}^2\notag\\
	&=&  2\frac{1}{K} \sum_{k=0}^{K-1}\frac{1}{n} \sum^n_{i=1} \Exp \norm{\nabla f(\vz_i^k) - \nabla f(\itr{\xbar}{k})}^2 + 2\frac{1}{K} \sum_{k=0}^{K-1} \Exp \norm{\nabla f(\itr{\xbar}{k})}^2\notag\\
	&\overset{L-smooth}{=}& 2 L^2 \frac{1}{K} \sum_{k=0}^{K-1}\frac{1}{n} \sum^n_{i=1} \Exp \norm{\vz_i^k - \itr{\xbar}{k}}^2 + 2\frac{1}{K} \sum_{k=0}^{K-1} \Exp \norm{\nabla f(\itr{\xbar}{k})}^2\notag\\
	&\overset{(\ref{bound1})+(\ref{bound2})}{\leq}&  O\left(\frac{1}{\sqrt{nK}} + \frac{1}{K}+\frac{1}{K^{3/2}}\right)
\end{eqnarray}
where again the Big O notation swallows all constants of our setting $\left (n,L, \sigma, \zeta, C , q, \sum^n_{i=1} \norm{ \itr{x_i}{0} }^2 \text{and} f( \itr{\xbar}{0} ) - f^* \right)$.
\end{proof}

\end{document}